\documentclass{article}

\usepackage{microtype}
\usepackage{graphicx}
\usepackage{subfigure}
\usepackage{booktabs} 
\usepackage{dsfont}
\usepackage{amsthm}
\usepackage{mdwmath}
\usepackage{mdwtab}
\usepackage{amsmath}
\usepackage{makecell}
\usepackage{amssymb}
\usepackage{amsfonts}
\usepackage{enumitem}
\usepackage{caption}
\usepackage{hyperref}
\usepackage{algorithm}
\usepackage{algorithmic}
\usepackage{color}
\usepackage{overpic}
\usepackage{natbib}
\setcitestyle{authoryear,round,citesep={;},aysep={,},yysep={;}}

\usepackage[accepted]{icml2024}

\theoremstyle{plain}
\newtheorem{Theorem}{Theorem}
\newtheorem{Definition}{Definition}

\newtheorem{Lemma}{Lemma}
\newtheorem{Remark}{Remark}

\newtheorem{Assumption}{Assumption}

\usepackage[textsize=tiny]{todonotes}

\icmltitlerunning{TernaryVote: Differentially Private, Communication Efficient, and Byzantine Resilient Distributed Optimization}

\begin{document}

\setlength{\abovedisplayskip}{0pt}
\setlength{\belowdisplayskip}{0pt}

\twocolumn[
\icmltitle{TernaryVote: Differentially Private, Communication Efficient, and Byzantine Resilient Distributed Optimization on Heterogeneous Data}

\begin{icmlauthorlist}
\icmlauthor{Richeng Jin}{zju}
\icmlauthor{Yujie Gu}{zju}
\icmlauthor{Kai Yue}{ncsu}
\icmlauthor{Xiaofan He}{Wu}
\icmlauthor{Zhaoyang Zhang}{zju}
\icmlauthor{Huaiyu Dai}{ncsu}
\end{icmlauthorlist}

\icmlaffiliation{zju}{Department of Information and Communication Engineering, Zhejiang University, China.}
\icmlaffiliation{Wu}{Electronic Information School, Wuhan University, China.}
\icmlaffiliation{ncsu}{Department of Electrical and Computer Engineering, North Carolina State University, Raleigh, NC, USA.}
\icmlcorrespondingauthor{Richeng Jin}{richengjin@zju.edu.cn}

%

\vskip 0.3in
]
\printAffiliationsAndNotice{} 

\begin{abstract}
Distributed training of deep neural networks faces three critical challenges: privacy preservation, communication efficiency, and robustness to fault and adversarial behaviors. Although significant research efforts have been devoted to addressing these challenges independently, their synthesis remains less explored. In this paper, we propose \textit{TernaryVote}, which combines a ternary compressor and the majority vote mechanism to realize differential privacy, gradient compression, and Byzantine resilience simultaneously. We theoretically quantify the privacy guarantee through the lens of the emerging $f$-differential privacy (DP) and the Byzantine resilience of the proposed algorithm. Particularly, in terms of privacy guarantees, compared to the existing sign-based approach \textit{StoSign}, the proposed method improves the dimension dependence on the gradient size and enjoys privacy amplification by mini-batch sampling while ensuring a comparable convergence rate. We also prove that \textit{TernaryVote} is robust when less than $50\%$ of workers are blind attackers, which matches that of {\scriptsize SIGN}SGD with majority vote. Extensive experimental results validate the effectiveness of the proposed algorithm.
\end{abstract}

\section{Introduction}\label{Introduction}
\noindent In the past decades, the ever-growing computational power distributed across the network and massive data generated daily have enabled the unprecedented success of distributed machine learning techniques \cite{dean2012large}. In the classic parameter server paradigm for distributed learning, the training process consists of multiple workers coordinated by the central server that updates a global model iteratively using the model updates from the workers. In distributed stochastic gradient descent (SGD), the server updates the model with the average of the stochastic gradients computed by the workers using their local datasets \cite{bertsekas2015parallel}.

While harnessing the computing power of the distributed workers, distributed SGD faces several critical challenges. Firstly, the local training data collected by the workers may contain sensitive information (e.g., medical data \cite{rieke2020future}), which hinders their willingness to participate in collaborative training. Despite that the federated learning (FL) paradigm offers a certain degree of privacy protection by practicing the principle of data minimization, there is no formal and rigorous quantification of privacy \cite{kairouz2019advances}. Secondly, in applications like FL, the workers are usually equipped with limited communication capability while the size of modern neural networks is unprecedentedly growing, which renders communication latency a major bottleneck.  Finally, the distributed training paradigm is vulnerable to fault and adversarial behaviors, and any faulty worker can ruin the convergence of distributed SGD by sending a sufficiently large gradient \cite{bernstein2018signsgd2}.

Significant research efforts have been devoted to addressing the aforementioned challenges. More specifically, various differentially private (DP) mechanisms \cite{abadi2016deep,agarwal2018cpsgd,chen2020breaking,kairouz2021distributed,agarwal2021skellam,chen2022poisson}, gradient compression schemes \cite{alistarh2017qsgd,haddadpour2020federated,stich2018sparsified,bernstein2018signsgd1,karimireddy2019error,safaryan2021stochastic,jin2024sign}, and Byzantine robust aggregators \cite{blanchard2017machine,yin2018byzantine,xie2019slsgd,karimireddy2021learning,farhadkhani2022byzantine,allouah2023fixing}, have been proposed to alleviate the privacy concern, the communication efficiency issue, and the vulnerability against Byzantine attacks, respectively. Although a few pioneering works, e.g., \cite{guerraoui2021differential,allouah2023privacy}, have studied the combination of DP mechanisms and Byzantine robust schemes, the requirement for communication efficiency is often ignored. Among these approaches, {\scriptsize SIGN}SGD with majority vote \cite{bernstein2018signsgd2} is of particular interest since it offers both Byzantine resilience and a significant reduction in communication overhead. However, it fails to converge in the presence of data heterogeneity \cite{chen2019distributed}. \cite{xiang2023distributed} shows the differential privacy guarantee of the stochastic-sign compressor that is proposed in \cite{jin2020stochastic} to address the non-convergence issue of {\scriptsize SIGN}SGD, which reveals the potential of providing differential privacy, communication efficiency, and Byzantine resilience in a unified framework. Nonetheless, the $\epsilon$-DP guarantee in \cite{xiang2023distributed} has a linear dependency on $d$ (i.e., the dimension of gradients), which renders the privacy protection less meaningful for modern neural networks with $d$ in the order of hundreds of millions. Recently, \cite{jin2023breaking} has observed that incorporating random sparsification into the stochastic-sign compressor leads to privacy amplification for distributed mean estimation. Inspired by this, we incorporate ternary compression into the majority vote mechanism and propose \textit{TernaryVote}. Similar to the sign-based approaches, the ternary-based majority vote mechanism is expected to provide a certain degree of Byzantine resilience. We make the aspiration rigorous and show that the non-convergence issue of the sign-based approach can be addressed while ensuring differential privacy and further improving communication efficiency.

\textbf{Our contributions}. Our main technical contributions are summarized as follows.
\begin{itemize}
    \item We analyze the differential privacy guarantee of the ternary compressor \cite{jin2023breaking} in the use case of SGD, based on which \textit{TernaryVote} is proposed. In sharp contrast to the results in \cite{xiang2023distributed}, the privacy guarantee of \textit{TernaryVote} has a dependency on $\sqrt{d}$ (instead of $d$), and \textit{TernaryVote} enjoys privacy amplification by mini-batch training data sampling.
    \item Under the bounded gradient assumption, we show that \textit{TernaryVote} converges in expectation with a rate $\mathcal{O}(1/\sqrt{T} + B/\sqrt{M})$ (which matches that of \textit{StoSign} \cite{xiang2023distributed}) in the low-privacy regime (i.e., a small $B$), in which $T$ is the number of communication rounds, $M$ is the number of workers, and $B$ is a tuning parameter for ternary compression. In the high-privacy regime (in $\mu_{T}$-GDP) with $B = \mathcal{O}(\sqrt{T}/\mu_{T})$, it converges in expectation with a rate $\mathcal{O}(1/\mu_{T} + \mu_{T}/\sqrt{T})$, which matches that of DP-SGD with the classic Gaussian mechanism \cite{fang2022improved}.
    \item We theoretically quantify the Byzantine resilience of \textit{TernaryVote}, which can tolerate up to $K = M-1$ blind attackers \cite{bernstein2018signsgd2} in which $M$ is the number of normal workers. 
    \item Experimental results on the MNIST, Fashion-MNIST, and CIFAR-10 datasets validate the effectiveness of the proposed method.
\end{itemize}

\section{Related Work}
\textbf{Differential Privacy Mechanism}: Since the seminar work \cite{abadi2016deep} introduces the classic Gaussian mechanism into deep learning, there has been a surging interest in developing various DP mechanisms. To cope with the communication efficiency issue, significant research efforts have been devoted to studying discrete mechanisms. \cite{dwork2006our} introduces the one-dimensional binomial noise, which is extended to the general $d$-dimensional case in \cite{agarwal2018cpsgd} with more comprehensive analysis in terms of $(\epsilon,\delta)$-DP. \cite{canonne2020discrete, kairouz2021distributed} investigate the DP guarantees of discrete Gaussian noise. \cite{agarwal2021skellam} and \cite{chen2022poisson} propose the Skellam mechanism and the Poisson binomial mechanism, respectively, with R\'enyi DP guarantees. \cite{chaudhuri2022privacy, guo2023privacy} achieve privacy-aware compression through numerical mechanism design, and \cite{zhu2023improving} studies the impact of random sparsification on DP-SGD \cite{abadi2016deep}. For distributed mean estimation, \cite{chen2020breaking} proposes the subsampled and quantized Kashin’s response (SQKR) mechanism that achieves order-optimal estimation error. \cite{chen2023privacy} studies privacy amplification by compression for central ($\epsilon,\delta$)-DP, while \cite{jin2023breaking} considers the privacy amplification of random sparsification through the lens of $f$-DP. However, none of these works takes the Byzantine resilience of the proposed mechanisms into consideration.

\textbf{Robustness}: In practice, the workers may fail to deliver correct information due to hardware or software failure, data corruption, transmission error, or malicious adversaries. To address this issue, various Byzantine resilient mechanisms have been proposed \cite{blanchard2017machine,xie2018generalized,mhamdi2018hidden,yin2018byzantine,karimireddy2021learning,farhadkhani2022byzantine,karimireddy2022byzantine,zhu2023byzantine,allouah2023fixing}. A few works investigate the robustness of differentially private mechanisms. Specifically, \cite{sun2019can,naseri2020local} empirically show that adding Gaussian noise to the model updates mitigates backdoor attacks and white-box membership inference attacks. \cite{nguyen2022flame} combines clustering, adaptive clipping, and the Gaussian mechanism to defend against poisoning attacks while realizing client-level DP. \cite{guerraoui2021differential} shows that most of the existing Byzantine resilient aggregation rules, including Krum \cite{blanchard2017machine}, Bulyan \cite{mhamdi2018hidden}, Trimmed mean \cite{yin2018byzantine}, and Median \cite{yin2018byzantine}, suffer from the curse of dimensionality when combined with the Gaussian mechanism. \cite{allouah2023privacy} studies the tradeoff between privacy, robustness, and accuracy in distributed learning and proposes the smallest maximum eigenvalue averaging method for Byzantine resilient aggregation. However, it results in a computational complexity of $\mathcal{O}(d^{3})$. In addition, none of the aforementioned works considers the communication efficiency.

\cite{zhang2023byzantine} incorporates sparsification and variance reduction into the classic Gaussian mechanism, which demonstrates robustness against Byzantine attackers empirically. \cite{zhu2022bridging} shows the DP guarantees of the sign-flipping mechanism (only for the scalar case) and extends the robust stochastic model aggregation (RSA) method \cite{li2019rsa} to its differentially private variant. \cite{xiang2023distributed} proves the differential privacy guarantee and the Byzantine resilience of the stochastic-sign compressor in \cite{jin2020stochastic}. However, the privacy guarantee in \cite{xiang2023distributed} has a linear dependency on $d$.

\section{Problem Setup and Preliminaries}
\subsection{Problem Setup}
\noindent We consider a classical parameter server paradigm for distributed learning that consists of $M$ honest workers (denoted by $\mathcal{M}$) and a central server. Each worker holds a local dataset $\mathcal{D}_{m}$, and the goal of the workers is to minimize the finite-sum objective of the form
\begin{equation}\label{objective}
\min_{\boldsymbol{w}\in \mathbb{R}^d}F(\boldsymbol{w})\overset{\mathrm{def}}{=} \frac{1}{M}\sum_{m=1}^{M}f_{m}(\boldsymbol{w}),
\end{equation}
where $f_{m}(\boldsymbol{w})$ is a local loss function defined by the local dataset of worker $m$ and the parameter vector $\boldsymbol{w} \in \mathcal{W}$. More specifically, $f_{m}(\boldsymbol{w})=\frac{1}{|\mathcal{D}_{m}|}\sum_{s\in \mathcal{D}_{m}}l(\boldsymbol{w};s)$
where $|\mathcal{D}_{m}|$ is the size of worker $m$'s local dataset $\mathcal{D}_{m}$ and $l: \mathcal{W}\times \mathcal{D} \rightarrow \mathbb{R}$ is the loss function that measures the loss of prediction on the data point $s \in \mathcal{D}$ made with $\boldsymbol{w}$.

\subsection{Privacy Measure}
Each honest worker $m$ aims to protect the privacy of their local dataset $\mathcal{D}_{m}$ from all the other participants, i.e., the server and the other workers. In this work, we adopt the well-known differential privacy as the privacy measure. We first introduce the most commonly used differential privacy measure $(\epsilon,\delta)$-DP \cite{dwork2006our}, followed by the emerging concept of $f$-DP \cite{dong2021gaussian}. Particularly, compared to $(\epsilon,\delta)$-DP, $f$-DP enjoys the hypothesis testing interpretation and a better composition property.

\textbf{$(\epsilon,\delta)$-Differential Privacy}: Formally, $(\epsilon,\delta)$-DP \cite{dwork2006our} is defined as follows.
\begin{Definition}[$(\epsilon,\delta)$-DP]
A randomized mechanism $\mathcal{Q}$ is $(\epsilon,\delta)$-differentially private {\color{black}if} for all neighboring datasets $S$ and $S'$ and all $O \subset \mathcal{O}$ in the range of $\mathcal{M}$, we have $P(\mathcal{Q}(S)\in O) \leq e^{\epsilon} P(\mathcal{Q}(S')\in O) + \delta$,
in which $S$ and $S'$ are neighboring datasets that differ in only one record, and $\epsilon,\delta \geq 0$ characterize the level of privacy.
\end{Definition}

\textbf{$f$-Differential Privacy}: For two neighboring datasets $S$ and $S'$, from the hypothesis testing perspective, we consider two hypotheses
\begin{equation}\label{hypothesistestingproblem}
\begin{split}
H_{0}: \text{the underlying dataset is}~S,\\
~H_{1}: \text{the underlying dataset is}~S'.
\end{split}
\end{equation}

Let $P$ and $Q$ denote the probability distribution of the outputs of the randomized mechanism $\mathcal{Q}(S)$ and $\mathcal{Q}(S')$, respectively. Consider a rejection rule $0 \leq \phi \leq 1$ (i.e., rejecting $H_{0}$ with a probability of $\phi$), there exists a tradeoff between the achievable type I and type II error rates defined as
\begin{equation}
    \alpha_{\phi} = \mathbb{E}_{P}[\phi], \beta_{\phi} = 1 - \mathbb{E}_{Q}[\phi],
\end{equation}
respectively. $f$-DP characterizes this tradeoff through the following tradeoff function.

\begin{Definition}[tradeoff function \cite{dong2021gaussian}]\label{tradeofffunction}
For any two probability distributions $P$ and $Q$ on the same space, the tradeoff function $T(P,Q): [0,1]\rightarrow [0,1]$ is defined as
\begin{equation}\label{betaphi}
    T(P,Q)(\alpha) = \inf\{\beta_{\phi}:\alpha_{\phi}\leq\alpha\},
\end{equation}
where the infimum is taken over all (measurable) $\phi$.
\end{Definition}
Formally, $f$-DP \cite{dong2021gaussian} is defined as follows.
\begin{Definition}[$f$-DP]
Let $f$ be a tradeoff function. A mechanism $\mathcal{Q}$ is $f$-differentially private if for all neighboring datasets $S$ and $S'$,
\begin{equation}
    T(\mathcal{Q}(S),\mathcal{Q}(S')) \geq f,
\end{equation}
which suggests that the attacker cannot achieve a type II error rate smaller than $f(\alpha)$ given that the type I error rate is no larger than $\alpha$.
\end{Definition}

$f$-DP can be converted to $(\epsilon,\delta)$-DP as follows.
\begin{Lemma}\label{Lemmaftoappro}\cite{dong2021gaussian}
A mechanism is $f(\alpha)$-differentially private if and only if it is $(\epsilon,\delta)$-differentially private with
\begin{equation}\label{ftodp}
f(\alpha) = \max\{0,1-\delta-e^{\epsilon}\alpha,e^{-\epsilon}(1-\delta-\alpha)\}.
\end{equation}
\end{Lemma}

Finally, we introduce a special case of $f$-DP with $f(\alpha) = \Phi(\Phi^{-1}(1-\alpha)-\mu)$, which is denoted as $\mu$-GDP. Specifically, $\mu$-GDP corresponds to the tradeoff function of two normal distributions with mean 0 and $\mu$, respectively, and a variance of 1. It enjoys the following composition property.
\begin{Lemma}\label{compositiontheorem}
The $T$-fold composition of $\mu_{i}$-GDP mechanisms is $\sqrt{u_{1}^2+u_{2}^2+\cdots+u_{T}^2}$-GDP.
\end{Lemma}


\subsection{Gradient Compression}\label{ternarycompressor}
In applications like FL, the workers are usually mobile devices that are equipped with limited communication capability, and gradient compression is commonly adopted to improve communication efficiency. In this work, we consider the following ternary stochastic compressor.

\begin{Definition}[\textbf{Ternary Stochastic Compressor}]\label{TernaryCompressor}
For any given $x\in[-c,c]$, the compressor $ternary$ outputs $ternary(x,A,B)$, which is given by
\begin{equation}
ternary(x,A,B) =
\begin{cases}
\hfill 1, \hfill \text{with probability $\frac{A+x}{2B}$},\\
\hfill 0, \hfill \text{with probability $1-\frac{A}{B}$},\\
\hfill -1, \hfill \text{with probability $\frac{A-x}{2B}$}.\\
\end{cases}
\end{equation}
\end{Definition}

The privacy guarantee of the above ternary compressor concerning $x$ is given by \citep{jin2023breaking} as follows.

\begin{Theorem}\citep{jin2023breaking}\label{privacyofternarytheorem}
The ternary compressor in Definition \ref{TernaryCompressor} is $f(\alpha)$-differentially private with
\begin{equation}\label{privacyofternary}
\begin{split}
&f(\alpha) =
\begin{cases}
\hfill 1 - \frac{A+c}{A-c}\alpha, \hfill ~~\text{for $\alpha \in [0,\frac{A-c}{2B}]$},\\
\hfill 1 - \frac{c}{B} - \alpha, \hfill ~~\text{for $\alpha \in [\frac{A-c}{2B}, 1-\frac{A+c}{2B}]$},\\
\hfill \frac{A-c}{A+c} - \frac{A-c}{A+c}\alpha, \hfill ~~\text{for $\alpha \in [1-\frac{A+c}{2B},1]$}.\\
\end{cases}
\end{split}
\end{equation}
\end{Theorem}

\subsection{Threat Model}
In addition to the $M$ normal workers, it is assumed that there exist $K$ Byzantine attackers, and its set is denoted as $\mathcal{K}$. The attackers may send arbitrary information to the central server, with their identities a priori unknown. 

\begin{algorithm}
\setlength{\abovedisplayskip}{0pt}
\setlength{\belowdisplayskip}{0pt}
\caption{Differentially Private,  Communication Efficient, and Byzantine Resilient Distributed SGD}
\label{DPSGDAlgorithm}
\begin{algorithmic}
\STATE \textbf{Initialization}: The initial model weights $\boldsymbol{w}_{0}$, the differentially private compression mechanism $\mathcal{C}(\cdot)$, the robust aggregator $Agg(\cdot)$, the batch size $b$, the clipping function $\textbf{Clip}(\cdot)$, the learning rate $\eta$, and the total number of communication rounds $T$.
\FOR{communication round $t=0,1,2,\cdots,T$}
\STATE Server randomly samples a subset of workers $\mathcal{N}_{t}$ and sends the model weights $\boldsymbol{w}^{(t)}$ to them.
\FOR{each worker $i \in \mathcal{N}_{t}$}
\IF{$i \in \mathcal{M}$}
\STATE Sample a mini-batch $\mathcal{B}_{i}^{(t)}$ training examples of size $b$ at random from $\mathcal{D}_{i}$. Then compute and clip the per-example gradients and average the mini-batch gradients:
\begin{equation}
\bar{\boldsymbol{g}}_{i}^{(t)} = \frac{1}{b}\sum_{s\in \mathcal{B}_{i}^{(t)}}\textbf{Clip}(\nabla l(\boldsymbol{w}^{(t)};s))
\end{equation}
\STATE Apply the differentially private compression mechanism and send $\boldsymbol{Z}_{i}^{(t)} = \mathcal{C}(\bar{\boldsymbol{g}}_{i}^{(t)})$ back to the central server.
\ELSIF{$i \in \mathcal{K}$}
\STATE Send arbitrary information $\boldsymbol{Z}_{i}^{(t)}$ to the server.
\ENDIF
\ENDFOR
\STATE The central server aggregates the received gradients $\hat{\boldsymbol{g}}^{(t)} = Agg(\{\boldsymbol{Z}_{i}^{(t)}\}_{i \in \mathcal{N}_{t}})$ and updates the model by
\begin{equation}
\boldsymbol{w}^{(t+1)} = \boldsymbol{w}^{(t)} - \eta\hat{\boldsymbol{g}}^{(t)}.
\end{equation}
\ENDFOR
\end{algorithmic}
\end{algorithm}

\subsection{The Overall Distributed Learning Process}
Algorithm \ref{DPSGDAlgorithm} summarizes the overall distributed training process. During each communication round $t$, the server randomly samples a subset of workers and broadcasts the global model weights to the selected workers. The workers sample a mini-batch of training examples, compute and clip the corresponding per-example gradient, and send the average mini-batch clipped gradients to the server after applying the differentially private compression mechanism. Upon receiving the gradients from all the workers, the server aggregates them and updates the global model weights accordingly. We remark that the vanilla DP-SGD \cite{abadi2016deep} is a special case of Algorithm \ref{DPSGDAlgorithm} with $\textbf{Clip}(\boldsymbol{g};C) = \boldsymbol{g}\cdot \min\{1, C/||\boldsymbol{g}||_{2}\}$, $\mathcal{C}(\boldsymbol{g}) = \boldsymbol{g} + \boldsymbol{\xi}$ where $\boldsymbol{\xi}$ is the Gaussian noise, and $Agg(\{\boldsymbol{Z}_{i}^{(t)}\}_{i \in \mathcal{N}_{t}}) = \frac{1}{|\mathcal{N}_{t}|}\sum_{i \in \mathcal{N}_{t}}\boldsymbol{Z}_{i}^{(t)}$.

\section{The Proposed Mechanisms}
\noindent In this section, we introduce the proposed mechanisms. During communication round $t$, the workers clip the per-example gradients with $l_{\infty}$ norm of $c$ and obtain $\boldsymbol{Z}_{i}^{(t)}$ by applying the ternary stochastic compressor to each coordinate of the average mini-batch gradient independently, i.e., $\boldsymbol{Z}_{i}^{(t)} = [ternary(\bar{\boldsymbol{g}}_{i,1}^{(t)},A,B),\cdots,ternary(\bar{\boldsymbol{g}}_{i,d}^{(t)},A,B)]$. For the server, we consider two candidate aggregators:
\begin{equation}\label{twoaggregators}
\begin{split}
\hat{\boldsymbol{g}}^{(t)} &= Agg(\{\boldsymbol{Z}_{i}^{(t)}\}_{i \in \mathcal{N}_{t}}) \\
&=
\begin{cases}
\hfill \frac{1}{|\mathcal{N}_{t}|}\sum_{i \in \mathcal{N}_{t}}\boldsymbol{Z}_{i}^{(t)}, \hfill ~~\text{Scheme I},\\
\hfill sign\left(\frac{1}{|\mathcal{N}_{t}|}\sum_{i \in \mathcal{N}_{t}}\boldsymbol{Z}_{i}^{(t)}\right), \hfill ~~\text{Scheme II}.\\
\end{cases}
\end{split}
\end{equation}

We note that the privacy guarantee in Theorem \ref{privacyofternarytheorem} assumes compressing the private data $x \in [-c,c]$ that is symmetric about 0, which means that it can be applied to the SGD scenario with gradient clipping. For mini-batch SGD, however, the results cannot be directly applied. More specifically, since the neighboring datasets differ in only one training example (denoted by $s'$), we have
\begin{equation}
\begin{split}
\bar{\boldsymbol{g}}_{i}^{(t)} &= \frac{1}{b}\underbrace{\textbf{Clip}(\nabla l(\boldsymbol{w}^{(t)};s'))}_{\boldsymbol{x}_{i}} + \frac{1}{b}\underbrace{\sum_{s\in \mathcal{B}_{i}^{(t)}/s'}\textbf{Clip}(\nabla l(\boldsymbol{w}^{(t)};s))}_{\boldsymbol{y}}.
\end{split}
\end{equation}

Therefore, $\bar{\boldsymbol{g}}_{i}^{(t)}$ (i.e., the data to be compressed) is symmetric about some $\boldsymbol{y}$ instead of $\boldsymbol{0}$ concerning the difference caused by the private training example $s'$ (i.e., $\bar{\boldsymbol{g}}_{i,j}^{(t)} \in [\frac{\boldsymbol{y}_{j}}{b} - \frac{c}{b}, \frac{\boldsymbol{y}_{j}}{b} + \frac{c}{b}]$ is no longer symmetric about $0$ when $\boldsymbol{y}_{j} \neq 0$), which renders the result in Theorem \ref{privacyofternarytheorem} not applicable directly. With such consideration, in the following, we extend the result in Theorem \ref{privacyofternarytheorem} to cover the mini-batch SGD scenario.


\begin{algorithm}
\caption{Ternary Compressor for Vectors}
\label{TernaryMechanismVector}
\begin{algorithmic}
\STATE \textbf{Input}: $c,b > 0$, $\bar{\boldsymbol{x}}_{i} = \frac{1}{b}(\boldsymbol{y} + \boldsymbol{x}_{i})$, in which $\boldsymbol{x}_{i,j} \in [-c,+c], \forall 1 \leq j \leq d$ and $\boldsymbol{y} \in [-(b-1)c,(b-1)c]^{d}$.
\STATE \textbf{Privatization}:

$\boldsymbol{Z}_{i} \triangleq [ternary(\bar{\boldsymbol{x}}_{i,1},A,B),...,ternary(\bar{\boldsymbol{x}}_{i,d},A,B)]$.
\end{algorithmic}
\end{algorithm}

\subsection{Privacy of the Ternary Compressor}
In this subsection, we first present the privacy guarantee of Algorithm \ref{TernaryMechanismVector} for the scalar case (i.e., $d=1$). Specifically, we extend the differential privacy guarantees of the ternary compressor in Theorem \ref{privacyofternarytheorem} to a more general case, in which the input to the ternary compressor is a linear combination of another (unknown) variable $y$ and the private data $x_{i}$ as shown in Algorithm \ref{TernaryMechanismVector}. For mini-batch SGD, $b$ in Algorithm \ref{TernaryMechanismVector} corresponds to the mini-batch size, while $x_{i}$ and $y$ correspond to the gradients of the training example of interest and the remaining training examples.
\begin{Theorem}\label{Ternary_Batch_scalar}
Assuming that $B > A + c$, the ternary compressor is $f(\alpha)$-DP for the scalar $x_{i}$ with
\begin{equation}\label{fternarymechanism}
\begin{split}
&f(\alpha) = \\
&\begin{cases}
\hfill 1 - \frac{Ab-(b-2)c}{(A-c)b}\alpha, \hfill \text{for $\alpha \in [0,\frac{A-c}{2B}]$},\\
\hfill 1-\frac{c}{Bb}- \alpha, \hfill \text{for $\alpha \in [\frac{A-c}{2B}, 1-\frac{Ab-(b-2)c}{2Bb}]$},\\
\hfill \frac{(A-c)b}{Ab-(b-2)c}(1 - \alpha), \hfill \text{for $\alpha \in [1-\frac{Ab-(b-2)c}{2Bb},1]$}.\\
\end{cases}
\end{split}
\end{equation}
\end{Theorem}
In the following, we extend the result to the vector case by utilizing the central limit theorem in \cite{dong2021gaussian}.
\begin{Theorem}\label{Ternary_Batch_vector}
Assuming that $B > A + c$, the ternary compressor is $f(\alpha)$-DP for the vector $\boldsymbol{x}_{i}$ with
\begin{equation}
\begin{split}
G_{\mu}(\alpha+\gamma)-\gamma \leq f(\alpha) \leq G_{\mu}(\alpha-\gamma)+\gamma,
\end{split}
\end{equation}
in which
\begin{equation}\label{computemu}
\mu = \frac{2\sqrt{d}c}{\sqrt{(A-c)Bb^2+Bbc-c^2}},
\end{equation}
\begin{equation}
\begin{split}
\gamma &= \frac{0.56\left[\frac{A-c}{2B}\left|1+\frac{c}{Bb}\right|^3+\frac{Ab-(b-2)c}{2Bb}\left|1-\frac{c}{Bb}\right|^3\right]}{(\frac{(A-c)b+c}{Bb}-\frac{c^2}{B^2b^2})^{3/2}d^{1/2}}\\
&+\frac{0.56\left[\left(1-\frac{(A-c)b+c}{Bb}\right)\left|\frac{c}{Bb}\right|^{3}\right]}{(\frac{(A-c)b+c}{Bb}-\frac{c^2}{B^2b^2})^{3/2}d^{1/2}}.
\end{split}
\end{equation}
\end{Theorem}

\begin{Remark}[\textbf{Privacy Improvement via Mini-batch Sampling}]
Similar to the classic Gaussian mechanism in which mini-batch SGD reduces the global sensitivity (and therefore improves the privacy) compared to SGD, Theorem \ref{Ternary_Batch_vector} implies that the privacy guarantee of the $ternary$ compressor also improves (i.e., $\mu$ decreases) as $b$ increases. When $b=1$, it recovers the result in \cite{jin2023breaking}. Besides, instead of distributed mean estimation, we focus on distributed learning with analyses on convergence and Byzantine resilience.
\end{Remark}

\begin{Remark}
\cite{xiang2023distributed} proves that the stochastic-sign compressor, which is a special case of the ternary compressor with $A=B$, is $(\epsilon,0)$-DP with $\epsilon = d\log(\frac{B+c}{B-c})$, which has a linear dependency on $d$. Besides showing the privacy amplification effect of mini-batch sampling, Theorem \ref{Ternary_Batch_vector} implies that the privacy guarantee $\mu = \mathcal{O}(\sqrt{d})$.
\end{Remark}

\subsection{Convergence Results in the Absence of Attackers}
To facilitate the convergence analysis for Algorithm \ref{DPSGDAlgorithm}, we make the following commonly adopted assumptions.
\begin{Assumption}\label{A1}(Lower bound).
 For all $\boldsymbol{x}$ and some constant $F^{*}$, we have objective value $F(\boldsymbol{x}) \geq F^{*}$.
\end{Assumption}
\begin{Assumption}\label{A2}(Smoothness).
$\forall \boldsymbol{y},\boldsymbol{x}$, we require for some non-negative constant $L$,
\begin{equation}
F(\boldsymbol{y}) \leq F(\boldsymbol{x}) + \langle\nabla F(\boldsymbol{x}), \boldsymbol{y}-\boldsymbol{x}\rangle + \frac{L}{2}||\boldsymbol{y}-\boldsymbol{x}||_2^{2},
\end{equation}
where $\langle\cdot,\cdot\rangle$ is the standard inner product.
\end{Assumption}
\begin{Assumption}\label{A3}(Variance bound).
For any worker $m$, the stochastic gradient oracle gives an independent unbiased estimate $\boldsymbol{g}_{m}$ that has coordinate bounded variance:
\begin{equation}
\mathbb{E}[\boldsymbol{g}_{m}] = \nabla f_{m}(\boldsymbol{w}),\mathbb{E}[(\boldsymbol{g}_{m,i}-\nabla f_{m}(\boldsymbol{w})_{i})^2] \leq \sigma^2_{i},
\end{equation}
for a vector of non-negative constants $\bar{\boldsymbol{\sigma}} = [\sigma_{1},\cdots,\sigma_{d}]$.
\end{Assumption}
\begin{Assumption}\label{A4}(Gradient bound).
For any worker $m$, the stochastic gradient satisfies $\boldsymbol{g}_{m,i} \leq c, \forall 1\leq i \leq d$.
\end{Assumption}

We note that in the implementation of differentially private SGD algorithms, clipping is usually applied to ensure bounded gradients \cite{abadi2016deep}. In this work, we follow the literature (e.g., \cite{xiang2023distributed}) and adopt the bounded gradient assumption, i.e., Assumption \ref{A4}, for convergence analysis. The impact of gradient clipping has also been studied in the literature, e.g., \cite{zhang2022understanding}, which is left for future work. In the following results, we consider  $\mathcal{C}(\cdot) = ternary(\cdot,A,B)$ in which $B \geq 2A \geq 2c$ and $\mathcal{N}_{t} = \mathcal{M}, \forall t$ in Algorithm \ref{DPSGDAlgorithm}. Moreover, we term Algorithm \ref{DPSGDAlgorithm} with the scheme I aggregator and the scheme II aggregator \textit{TernaryMean} and \textit{TernaryVote}, respectively.

\begin{Theorem}[\textbf{Convergence of \textit{TernaryMean}}]\label{convergerate_schemeI}
Suppose Assumptions \ref{A1}-\ref{A4} are satisfied, then by running Algorithm \ref{DPSGDAlgorithm} with \textit{TernaryMean} for $T$ iterations, we have
\begin{equation}
\begin{split}
\frac{1}{T}\sum_{t=1}^{T}&||\nabla F(\boldsymbol{w}^{(t)})||_{2}^{2} \leq \frac{F(\boldsymbol{w}^{(0)}) - F^{*}}{T\big(\frac{\eta}{B} - \frac{L\eta^2}{2B^{2}}\big)} \\
&+ \frac{L\eta^2}{2B^{2}\big(\frac{\eta}{B} - \frac{L\eta^2}{2B^{2}}\big)}\left[\frac{ABd}{M} + \frac{||\bar{\boldsymbol{\sigma}}||_{2}^{2}}{M}\right].
\end{split}
\end{equation}
\end{Theorem}
\begin{Remark}
When $\eta \leq \frac{B}{L}$, we have $\frac{\eta}{B} - \frac{L\eta^2}{2B^{2}} \geq \frac{\eta}{2B}$. Setting $\eta=\frac{\sqrt{M}}{\sqrt{TLd}}$ gives $\frac{1}{T}\sum_{t=1}^{T}||\nabla F(\boldsymbol{w}^{(t)})||_{2}^{2} \leq \mathcal{O}(\frac{B}{\sqrt{MT}}+\frac{A}{\sqrt{MT}}+\frac{||\bar{\boldsymbol{\sigma}}||_{2}^{2}}{B\sqrt{MT}}) = \mathcal{O}(\frac{1}{\sqrt{MT}})$, which matches that of distributed SGD \cite{jiang2018linear}.
\end{Remark}

\begin{Theorem}[\textbf{Convergence of \textit{TernaryVote}}]\label{convergerate_schemeII}
Suppose Assumptions \ref{A1}-\ref{A4} are satisfied, and the learning rate is set as $\eta=\frac{1}{\sqrt{TLd}}$. Then by running Algorithm \ref{DPSGDAlgorithm} with \textit{TernaryVote} for $T$ iterations, we have
\begin{equation}
\begin{split}
&\frac{1}{T}\sum_{t=1}^{T}||\nabla F(\boldsymbol{w}^{(t)})||_{1} \leq \frac{(F(\boldsymbol{w}^{(0)}) - F^{*})\sqrt{Ld}}{\sqrt{T}} \\
&+ \frac{\sqrt{Ld}}{2\sqrt{T}}+ \frac{4||\bar{\boldsymbol{\sigma}}||_{1}}{\sqrt{M}}+ \frac{2Bd}{\sqrt{M+1}}\bigg(1-\frac{1}{M+1}\bigg)^{\frac{M}{2}} \\
&\leq \mathcal{O}\left(\frac{1}{\sqrt{T}}+\frac{B}{\sqrt{M}}\right).
\end{split}
\end{equation}
\end{Theorem}

\begin{Remark}
We note that taking the majority vote during aggregation (assuming $sign(0) = 0$) enables downlink compression since the model updates are also ternary. The convergence rate derived in Theorem \ref{convergerate_schemeII} matches that of \textit{StoSign} in \cite{xiang2023distributed}. It is possible to improve the convergence rate to $\mathcal{O}(1/\sqrt{T}+B/M)$ for full-batch gradient descent (c.f. Appendix \ref{convergencefullbatch}).
\end{Remark}

It is worth mentioning that Theorem \ref{convergerate_schemeII} implies a convergence rate of $O(1/\sqrt{T}+B/\sqrt{T})$ when $M \geq T$ and $B$ is some finite constant, which echoes the results in \cite{jin2020stochastic, xiang2023distributed} that the convergence of the sign-based SGD methods approaches that of the vanilla SGD for a large enough $M$. However, when $M$ is small, setting a large $B$ seems to ruin its convergence, while Theorem \ref{Ternary_Batch_vector} implies that increasing $B$ yields better privacy preservation. We address this dilemma with the following theorem.

\begin{Theorem}[\textbf{Convergence of \textit{TernaryVote}}]\label{convergerate_schemeII2}
Suppose Assumptions \ref{A1}-\ref{A4} are satisfied, $|\nabla F(\boldsymbol{w}^{(t)})_{i}| < Q, \forall i, t$, $B \geq 2A = \mathcal{O}(\sqrt{T})$, and $\lim_{T\rightarrow \infty}M/\sqrt{T} = 0$. Then by running Algorithm \ref{DPSGDAlgorithm} with \textit{TernaryVote} and the learning rate $\eta=\frac{1}{\sqrt{TLd}}$ for $T$ iterations, we have
\begin{equation}
\begin{split}
&\frac{1}{T}\sum_{t=1}^{T}||\nabla F(\boldsymbol{w}^{(t)})||_{2}^{2} \leq \frac{1}{\mathcal{I}(A,B,M)} \times \\
&\bigg[\frac{(F(\boldsymbol{w}^{(0)}) - F^{*})\sqrt{Ld}}{\sqrt{T}} + \frac{\sqrt{Ld}}{2\sqrt{T}} \\
&+ \sum_{n=2}^{M}\bigg(1-\frac{A}{B}\bigg)^{M-n}\bigg[{M \choose n}\mathcal{O}\bigg(\frac{A^{n-2}}{B^{n}}\bigg)\bigg]Qd\bigg]\\
&\leq \mathcal{O}\bigg(\frac{B}{\sqrt{T}}\bigg) + \mathcal{O}\bigg(\frac{1}{B}\bigg),
\end{split}
\end{equation}
in which
\begin{equation}\nonumber
\mathcal{I}(A,B,M) = \sum_{n=1}^{M}(1-\frac{A}{B})^{M-n}\bigg[\frac{{n-1 \choose \lfloor\frac{n-1}{2}\rfloor}MA^{n-1}{M-1 \choose n-1}}{2^{n-1}B^{n}}\bigg].
\end{equation}
\end{Theorem}


\begin{Remark}
We note that Theorem \ref{Ternary_Batch_vector} characterizes $\mu$-GDP guarantees of the ternary compressor for one iteration. The composition of $\mu$-GDP mechanisms in Lemma \ref{compositiontheorem} gives an overall privacy guarantee $\mu_{T} = \mathcal{O}(\sqrt{T}/\sqrt{AB})$. By setting a fixed sparsity ratio $\frac{A}{B}$ and $A = \mathcal{O}(\sqrt{T}/\mu_{T})$, we obtain $\frac{1}{T}\sum_{t=1}^{T}||\nabla F(\boldsymbol{w}^{(t)})||_{2}^{2} \leq \mathcal{O}(1/\mu_{T}) + \mathcal{O}(\mu_{T}/\sqrt{T})$, which matches that of the classic DP-SGD with the Gaussian mechanism. Note that DP-SGD has a convergence rate of $\mathcal{O}(\log(1/\delta)/\epsilon)$ for $(\epsilon,\delta)$-DP \cite{fang2022improved}, which is equivalent to $\mathcal{O}(1/\mu)$ for $\mu$-GDP.
\end{Remark}

\textbf{The Impact of Worker Sampling:} Theorems \ref{convergerate_schemeI}-\ref{convergerate_schemeII2} assume that all the workers are sampled for training during each communication round. However, the proofs can be readily extended to incorporate worker sampling. For example, suppose that each worker is sampled independently with a probability $p_{s}$ \cite{yang2021achieving}. Combining worker sampling with the ternary stochastic compressor yields
\begin{equation}
\begin{split}
&\mathcal{C}(x,A,B,p_{s}) =
\begin{cases}
\hfill 1, \hfill \text{with probability $\frac{A+x}{2B}p_{s}$},\\
\hfill 0, \hfill \text{with probability $1-\frac{A}{B}p_{s}$},\\
\hfill -1, \hfill \text{with probability $\frac{A-x}{2B}p_{s}$},\\
\end{cases}
\end{split}
\end{equation}
which implies that incorporating the uniform worker sampling strategy is equivalent to increasing the parameter $B$ by a factor of $1/p_{s}$.

\section{Byzantine Resilience}
\noindent In this section, we investigate the Byzantine resilience of \textit{TernaryVote}, i.e., Algorithm \ref{DPSGDAlgorithm} with the Scheme II aggregator. Since each normal worker $m \in \mathcal{M}$ only shares a ternary vector $\boldsymbol{Z}_{m} \in \{-1,0,1\}^{d}$, the Byzantine attackers will be easily identified if it shares anything other than a ternary vector. Therefore, we assume that each Byzantine attacker $k \in \mathcal{K}$ first obtains a gradient estimate $\boldsymbol{g}_{k}^{(t)}$, and then shares $\boldsymbol{Z}_{k}^{(t)} = [ternary(\boldsymbol{g}_{k,1}^{(t)},A,B),\cdots,ternary(\boldsymbol{g}_{k,d}^{(t)},A,B)]$ with the server, in which $\boldsymbol{g}_{k}^{(t)} \in [-c,c]^{d}$ can be arbitrary.

\begin{Theorem}\label{ByzantineTheorem1}
Suppose Assumptions \ref{A1}-\ref{A4} are satisfied, $|\nabla F(\boldsymbol{w}^{(t)})_{i}| < Q, \forall i, t$, and the learning rate is set as $\eta=\frac{1}{\sqrt{TLd}}$, then by running Algorithm \ref{DPSGDAlgorithm} with \textit{TernaryVote} and $\mathcal{N}_{t} = \mathcal{M} \cup \mathcal{K}$ for $T$ iterations, we have
\begin{equation}\label{ByzantineConvergenceEquation}
\begin{split}
&\frac{1}{T}\sum_{t=1}^{T}||\nabla F(\boldsymbol{w}^{(t)})||_{1} \leq \frac{(F(\boldsymbol{w}^{(0)}) - F^{*})\sqrt{Ld}}{\sqrt{T}} \\
&+ \frac{\sqrt{Ld}}{2\sqrt{T}}+ \frac{4K(Q+A)d}{M+K}+ \frac{4\sqrt{M}||\bar{\boldsymbol{\sigma}}||_{1}}{M+K}  \\
&+ \frac{2Bd}{\sqrt{M+K+1}}\bigg(1-\frac{1}{M+K+1}\bigg)^{\frac{M+K}{2}}\\
&\leq \mathcal{O}\left(\frac{1}{\sqrt{T}}+\frac{B}{\sqrt{M+K}} + \frac{\sqrt{M}+K}{M+K}\right).
\end{split}
\end{equation}
\end{Theorem}

\begin{Remark}
Compared to Theorem \ref{convergerate_schemeII}, there is an additional term $\mathcal{O}(\frac{\sqrt{M}+K}{M+K})$ in Theorem \ref{SPByzantine1}. If $K = \mathcal{O}(\sqrt{M})$, then the last two terms in (\ref{ByzantineConvergenceEquation}) scale in $M$ with order $\mathcal{O}(\frac{1}{\sqrt{M}})$, which means that when $M = \mathcal{O}(T)$, the proposed \textit{TernaryVote} algorithm can tolerate $\mathcal{O}(\sqrt{M})$ Byzantine attackers.
\end{Remark}

\begin{table*}[t!]
\vspace{-0.1in}
\caption{Test Accuracy on MNIST with $A/B = 0.01$ (200 communication rounds)}
\label{table_mnist}
\begin{center}
\begin{sc}
\begin{tabular}{ccccc}
\toprule
$\mu$ & 0.1 & 0.5 & 1 & 2 \\
\midrule
\makecell{Gaussian Noise \&\\  Random Sparsification} & $24.68\pm 1.63\%$ & $66.57\pm 1.35\%$ & $76.40\pm 1.97\%$ & $84.19\pm 0.24\%$\\
\textit{TernaryMean} & $54.62\pm 2.68\%$ & $79.60\pm 0.45\%$ & $84.07\pm 0.15\%$ & $85.33\pm 0.47\%$\\
\textit{TernaryVote} & $55.25\pm 3.26\%$ & $80.93\pm 0.89\%$ & $84.18\pm 0.94\%$ & $85.53\pm 0.64\%$\\
\bottomrule
\end{tabular}
\end{sc}
\end{center}
\vspace{-0.1in}
\end{table*}

\begin{table*}[t!]
\vspace{-0.1in}
\caption{Test Accuracy on Fashion-MNIST with $A/B = 0.01$ (200 communication rounds)}
\label{table_fashionmnist}
\begin{center}
\begin{sc}
\begin{tabular}{ccccc}
\toprule
$\mu$ & 0.1 & 0.5 & 1 & 2 \\
\midrule
\makecell{Gaussian Noise \&\\  Random Sparsification} & $32.79\pm 4.57\%$ & $66.66\pm 1.39\%$ & $71.55\pm 0.49\%$ & $74.88\pm 0.81\%$\\
\textit{TernaryMean} & $60.77\pm 0.84\%$ & $74.09\pm 0.65\%$ & $75.89\pm 0.78\%$ & $76.31\pm 0.62\%$\\
\textit{TernaryVote} & $61.42\pm 1.58\%$ & $74.89\pm 0.35\%$ & $76.63\pm 0.21\%$ & $76.98\pm 0.78\%$\\
\bottomrule
\end{tabular}
\end{sc}
\end{center}
\vspace{-0.1in}
\end{table*}

\begin{Theorem}\label{ByzantineTheorem2}
Suppose Assumptions \ref{A1}-\ref{A4} are satisfied, $|\nabla F(\boldsymbol{w}^{(t)})_{i}| < Q, \forall i, t$, $B \geq 2A = \mathcal{O}(\sqrt{T})$, and $\lim_{T\rightarrow \infty}(M+K)/\sqrt{T} = 0$. Then by running Algorithm \ref{DPSGDAlgorithm} with \textit{TernaryVote}, $\mathcal{N}_{t} = \mathcal{M}\cup\mathcal{K}$, and the learning rate $\eta=\frac{1}{\sqrt{TLd}}$ for $T$ iterations, we have
\begin{equation}
\begin{split}
&\frac{1}{T}\sum_{t=1}^{T}\sum_{i=1}^{d}\left(M\left|\nabla F(\boldsymbol{w}^{(t)})_{i}\right| - \left|\sum_{k\in\mathcal{K}}\boldsymbol{g}_{k,i}^{(t)}\right|\right)|\nabla F(\boldsymbol{w}^{(t)})_{i}| \\
&\leq \mathcal{O}\bigg(\frac{B}{\sqrt{T}}\bigg) + \mathcal{O}\bigg(\frac{1}{B}\bigg).
\end{split}
\end{equation}
\end{Theorem}

\begin{table*}[h!]
\vspace{-0.1in}
\caption{Test Accuracy on CIFAR-10 with $A/B = 0.01$ (500 communication rounds)}
\label{table_cifar10_0_01}
\begin{center}
\begin{sc}
\begin{tabular}{ccccc}
\toprule
$\mu$ & 0.1 & 0.5 & 1 & 2 \\
\midrule
\makecell{Gaussian Noise \&\\  Random Sparsification} & $13.03\pm 1.24\%$ & $26.80\pm 0.85\%$ & $30.84\pm 0.37\%$ & $35.17\pm 0.42\%$\\
\textit{TernaryMean} & $21.71\pm 1.19\%$  & $32.29\pm 1.32\%$ & $37.35\pm 1.13\%$ & $43.21\pm 0.44\%$\\
\textit{TernaryVote} & $23.15\pm 1.21\%$ & $31.04\pm 1.53\%$ & $36.08\pm 1.52\%$ & $37.66\pm 0.73\%$\\
\bottomrule
\end{tabular}
\end{sc}
\end{center}
\vspace{-0.1in}
\end{table*}


\begin{Remark}[\textbf{Robustness against Blind Attackers}]
Theorem \ref{ByzantineTheorem2} implies that the convergence of the proposed \textit{TernaryVote} algorithm is guaranteed as long as $M|\nabla F(\boldsymbol{w}^{(t)})_{i}| - |\sum_{k\in\mathcal{K}}\boldsymbol{g}_{k,i}^{(t)}| > 0$. For instance, if the attackers have access to the true gradients $\nabla F(\boldsymbol{w}^{(t)})$ and adopt $\mathcal{C}(\nabla F(\boldsymbol{w}^{(t)})) = ternary(-\nabla F(\boldsymbol{w}^{(t)}),A,B)$ (i.e., blind attackers as in \cite{bernstein2018signsgd2}), we have $|\sum_{k\in\mathcal{K}}\boldsymbol{g}_{k,i}^{(t)}| = K|\nabla F(\boldsymbol{w}^{(t)})_{i}|$. Then, the proposed \textit{TernaryVote} algorithm can tolerate $K = M-1$ Byzantine attackers, which is the same as {\scriptsize SIGN}SGD with majority vote with homogeneous data distribution across workers \cite{bernstein2018signsgd2}. Note that {\scriptsize SIGN}SGD fails to converge in the presence of data heterogeneity \cite{jin2020stochastic}), while we do not make assumptions on the data distribution.
\end{Remark}

\begin{table*}[h!]
\vspace{-0.1in}
\caption{Test Accuracy on CIFAR-10 with $A/B = 0.05$ (500 communication rounds)}
\label{table_cifar10_0_05}
\begin{center}
\begin{sc}
\begin{tabular}{ccccc}
\toprule
$\mu$ & 0.1 & 0.5 & 1 & 2 \\
\midrule
\makecell{Gaussian Noise \&\\  Random Sparsification} & $20.27\pm 1.46\%$ & $31.52\pm 0.95\%$ & $33.66\pm 0.59\%$ & $40.31\pm 0.60\%$\\
\textit{TernaryMean} & $24.87\pm 0.65\%$ & $32.62\pm 1.04\%$ & $38.25\pm 1.15\%$ & $41.95\pm 0.96\%$\\
\textit{TernaryVote} & $22.98\pm 1.68\%$ & $31.76\pm 0.73\%$ & $37.44\pm 1.47\%$ & $41.43\pm 0.98\%$\\
\bottomrule
\end{tabular}
\end{sc}
\end{center}
\vspace{-0.1in}
\end{table*}

\begin{figure*}[t!]
\centering
\begin{minipage}{0.3\linewidth}
  \centering
  \includegraphics[width=1\textwidth]{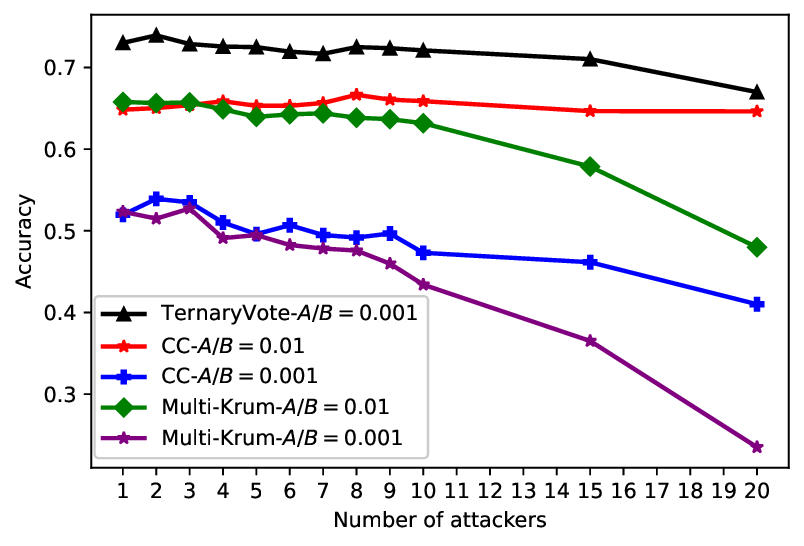}
\end{minipage}
\begin{minipage}{0.3\linewidth}
  \centering
  \includegraphics[width=1\textwidth]{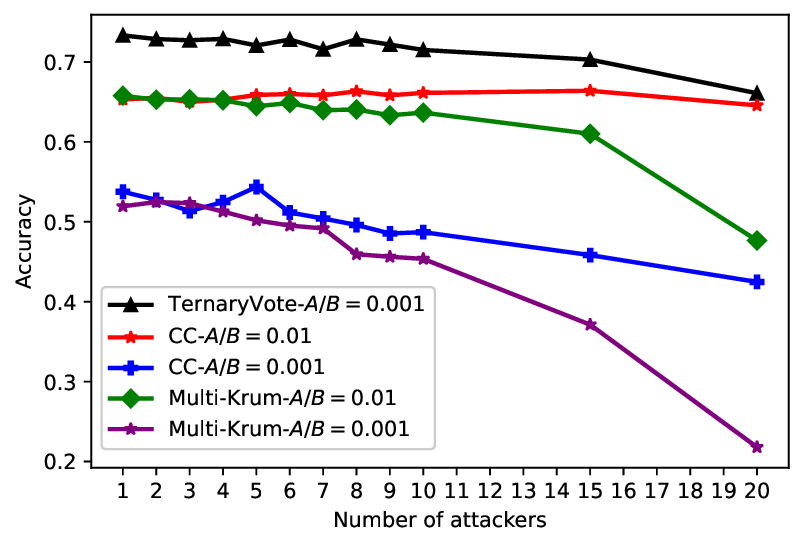}
\end{minipage}
\begin{minipage}{0.3\linewidth}
  \centering
  \includegraphics[width=1\textwidth]{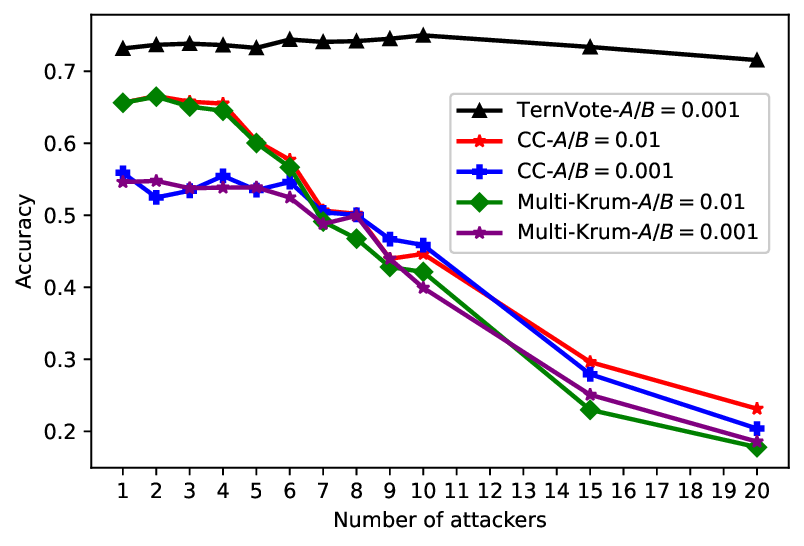}
\end{minipage}
\vspace{-0.1in}
\caption{Left, middle, and right figures compare \textit{TernaryVote} with Centered-Clipping (CC) and Multi-Krum under the flip sign (FS), the fall of empire (FoE), and the little is enough (LIE) attacks, respectively, on Fashion-MNIST given the same setting as that in Section \ref{secaccwithoutatt}.}
\label{ByzantinePerformance}
\end{figure*}

\section{Experimental Results}
In this section, we examine the performance of the proposed methods with a three-layer fully connected neural network on MNIST and Fashion-MNIST, and a CNN with four convolutional layers on CIFAR-10. In the absence of attackers, we compare the proposed algorithm with the combination of the Gaussian mechanism \cite{abadi2016deep} and random sparsification \cite{zhu2023improving} to ensure the same sparsity as the ternary compressor. For the selection of $A$ and $B$, we fix the sparsity ratio $A/B$ and the privacy guarantee $\mu$, and then find the corresponding $A$ and $B$ such that (\ref{computemu}) is satisfied. We use a batch size of $128$ in our experiments and clip the per-example gradient by $L_{2}$ norm with a threshold of $C=2$ \cite{abadi2016deep} for the baseline algorithm and by magnitude with a threshold of $c=0.0003$ for the proposed algorithm. In the presence of attackers, we further incorporate Multi-Krum \cite{blanchard2017machine} and the recently proposed centered clipping \cite{karimireddy2021learning} into the baseline algorithm. We note that in the high-sparsity regime, the median-based and trimmed mean-based methods may fail since the results will be 0 with a high probability. We run all the algorithms for $5$ repeats and present the results for high-sparsity and high-privacy scenarios, which are of more practical interest. More results for lower sparsity scenarios and implementation details can be found in Appendix \ref{suppadditionalresults} and Appendix \ref{DetailsImplementation}, respectively.
\subsection{Accuracy in the Absence of Attackers}\label{secaccwithoutatt}
Tables \ref{table_mnist}-\ref{table_cifar10_0_05} compare the test accuracy of the proposed methods with the baseline algorithm that combines the Gaussian mechanism and random sparsification. For MNIST and Fashion-MNIST, we consider a scenario of $M=100$ normal workers with the training data on each worker drawn independently with class labels following a Dirichlet distribution $Dir(\alpha)$ with $\alpha=0.1$, and 50 workers are sampled uniformly at random for training during each round. For CIFAR-10, we consider a scenario of $M=300$ normal workers with $\alpha=3$, and 90 workers are sampled uniformly at random for training during each communication round. It can be observed that \textit{TernaryMean} outperforms the baseline algorithm for all the examined scenarios, while \textit{TernaryVote} achieves a comparable performance to \textit{TernaryMean}. We note that despite the same sparsity, both \textit{TernaryMean} and \textit{TernaryVote} require only 1 bit to represent the value of each nonzero coordinate, while the baseline algorithm uses 32 bits (assuming that 32 bits are used to represent a float number). In this sense, the proposed methods achieve higher test accuracy while reducing the communication overhead from the workers to the parameter server. In addition, by taking the majority vote on the server side, \textit{TernaryVote} further reduces the communication overhead from the server to the workers compared to \textit{TernaryMean}.

\subsection{Accuracy in the Presence of Attackers}
In this subsection, we consider three types of attackers with the same data distribution as normal workers. The flip sign (FS) attackers flip the signs of gradients \cite{bernstein2018signsgd2} before applying random sparsification or the ternary compressor, the little is enough (LIE) attackers follow the method in \cite{baruch2019little} to generate the perturbed gradients before applying compression, while the fall of empire (FoE) attackers \cite{xie2020fall} flip the signs of the average gradients of the normal workers before applying compression. We assume that the attackers do not add noise (for the baseline algorithm) or set $A = c$ (for \textit{TernaryVote}) since they do not have any privacy concerns. We consider two different Byzantine robust aggregators for the baseline algorithm: the multi-krum aggregator \cite{blanchard2017machine}, and the centered-clipping mechanism \cite{karimireddy2021learning}. Fig. \ref{ByzantinePerformance} compares \textit{TernaryVote} with the baselines under the above three types of attackers with $\mu=0.5$ on Fashion-MNIST. It can be observed that \textit{TernaryVote} with $A/B = 0.001$ outperforms the baseline algorithms with $A/B = 0.01$ and $A/B = 0.001$ in the presence of up to 20 attackers, which corroborates its effectiveness.

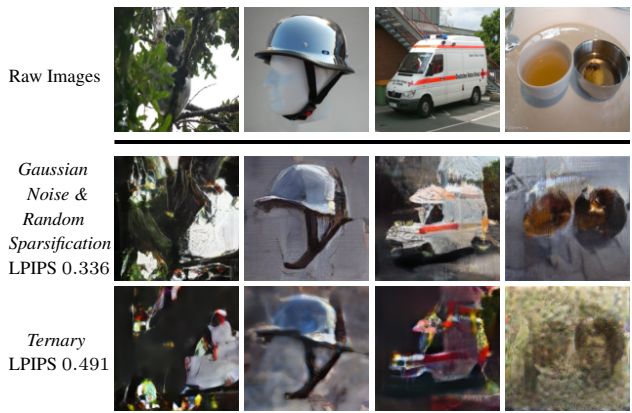
\begin{figure}[!tb]
    \begin{overpic}[width=\linewidth]{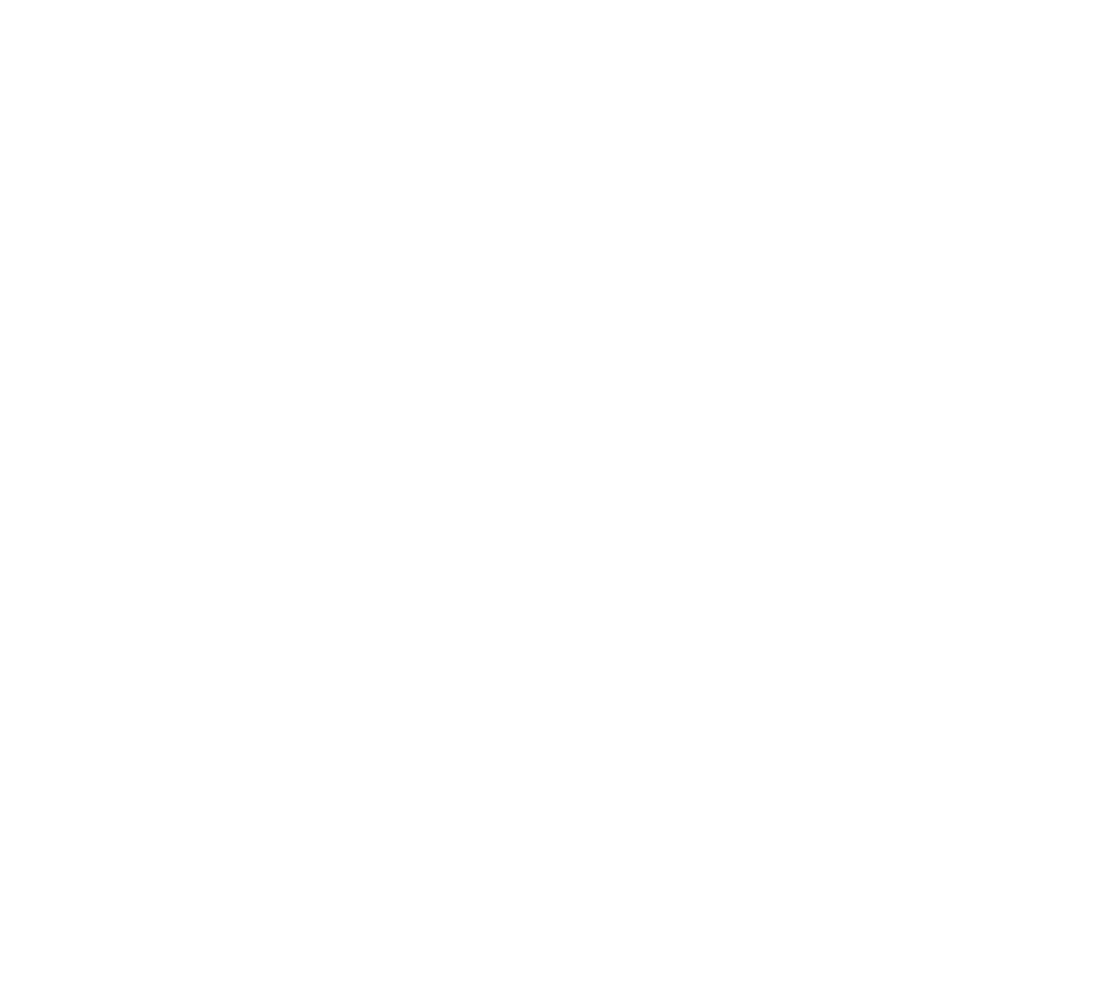}
    \put(17, 69){\includegraphics[width=0.2\linewidth]{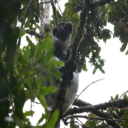}}
    \put(38, 69){\includegraphics[width=0.2\linewidth]{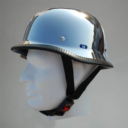}}
    \put(59, 69){\includegraphics[width=0.2\linewidth]{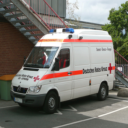}}
    \put(80, 69){\includegraphics[width=0.2\linewidth]{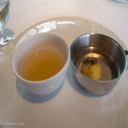}}

    \put(17, 67){\includegraphics[width=0.83\linewidth, height=1.3pt]{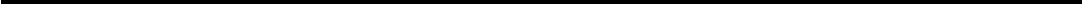}}

    \put(17, 45){\includegraphics[width=0.2\linewidth]{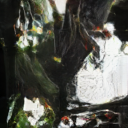}}
    \put(38, 45){\includegraphics[width=0.2\linewidth]{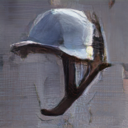}}
    \put(59, 45){\includegraphics[width=0.2\linewidth]{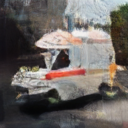}}
    \put(80, 45){\includegraphics[width=0.2\linewidth]{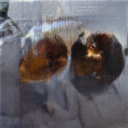}}

    \put(17, 24){\includegraphics[width=0.2\linewidth]{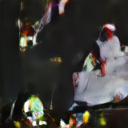}}
    \put(38, 24){\includegraphics[width=0.2\linewidth]{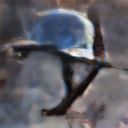}}
    \put(59, 24){\includegraphics[width=0.2\linewidth]{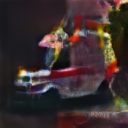}}
    \put(80, 24){\includegraphics[width=0.2\linewidth]{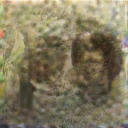}}

    \put(0, 77){\scriptsize Raw Images}
    \put(0, 62){\scriptsize \textit{~~Gaussian}}
    \put(0, 58){\scriptsize \textit{~~~~Noise \& }}
    \put(0, 54){\scriptsize \textit{~~~Random}}
    \put(0, 50){\scriptsize \textit{Sparsification}}
    \put(0, 46){\scriptsize LPIPS $0.336$}
    \put(0, 34.5){\scriptsize \textit{~~~~Ternary}}
    \put(0, 30.5){\scriptsize LPIPS $0.491$}
    \end{overpic}
    \vspace{-1in}
    \caption{Reconstructed images using the ROG attack for $\mu = 1$.}
    \label{rog}
    \vspace{-0.2in}
\end{figure}
\subsection{Protection Against Data Reconstruction Attacks}
Finally, we examine the privacy preservation capability of the proposed method against the reconstruction with obfuscated gradient (ROG) attack  \cite{yue2023gradient}, which is proposed for training data reconstruction based on compressed gradients. We consider training the LeNet \cite{zhu2019deep} using the validation dataset of ImageNet \cite{deng2009imagenet} with a mini-batch size of 32 and $A/B=0.05$. LPIPS is adopted to measure the data reconstruction quality, and a larger LPIPS value indicates better privacy protection against the attack \cite{yue2023gradient}. Overall, we obtain average (over the entire batch) LPIPS values of $0.365$ and $0.446$ for the baseline algorithm (i.e., the combination of the Gaussian mechanism and random sparsification) and the ternary compressor, respectively. Four exemplary images and their average LPIPS values are presented in Fig. \ref{rog}, which validate the effectiveness of the proposed method.

\section{Limitation}
We note that, in this paper, the privacy guarantee $\mu$ is computed for each communication round, and the results in Theorem \ref{Ternary_Batch_vector} asymptotically approximate that of the Gaussian mechanism. While it can be readily extended to accumulate privacy across communication rounds, the privacy amplification due to subsampling has not been accounted for, which will be an interesting and important future direction.

\section{Conclusion}
In this work, we propose a ternary compressor-based algorithm that is differentially private, communication efficient, and Byzantine resilient. Both the privacy guarantees and Byzantine resilience are theoretically quantified, and the convergence is established. It is expected that the proposed methods can find wide applications in areas such as federated learning where privacy, communication efficiency, and robustness are major bottlenecks. Further incorporating error feedback and momentum to reduce compression error and gradient variance remains interesting future works.


\bibliography{Ref-Richeng}

\begin{thebibliography}{59}
\providecommand{\natexlab}[1]{#1}
\providecommand{\url}[1]{\texttt{#1}}
\expandafter\ifx\csname urlstyle\endcsname\relax
  \providecommand{\doi}[1]{doi: #1}\else
  \providecommand{\doi}{doi: \begingroup \urlstyle{rm}\Url}\fi

\bibitem[Abadi et~al.(2016)Abadi, Chu, Goodfellow, McMahan, Mironov, Talwar,
  and Zhang]{abadi2016deep}
Abadi, M., Chu, A., Goodfellow, I., McMahan, H.~B., Mironov, I., Talwar, K.,
  and Zhang, L.
\newblock Deep learning with differential privacy.
\newblock In \emph{Proceedings of the 2016 ACM SIGSAC Conference on Computer
  and Communications Security}, pp.\  308--318, 2016.

\bibitem[Agarwal et~al.(2018)Agarwal, Suresh, Yu, Kumar, and
  McMahan]{agarwal2018cpsgd}
Agarwal, N., Suresh, A.~T., Yu, F. X.~X., Kumar, S., and McMahan, B.
\newblock {cpSGD}: Communication-efficient and differentially-private
  distributed {SGD}.
\newblock In \emph{Advances in Neural Information Processing Systems}, pp.\
  7564--7575, 2018.

\bibitem[Agarwal et~al.(2021)Agarwal, Kairouz, and Liu]{agarwal2021skellam}
Agarwal, N., Kairouz, P., and Liu, Z.
\newblock The skellam mechanism for differentially private federated learning.
\newblock \emph{Advances in Neural Information Processing Systems},
  34:\penalty0 5052--5064, 2021.

\bibitem[Alistarh et~al.(2017)Alistarh, Grubic, Li, Tomioka, and
  Vojnovic]{alistarh2017qsgd}
Alistarh, D., Grubic, D., Li, J., Tomioka, R., and Vojnovic, M.
\newblock {QSGD}: Communication-efficient {SGD} via gradient quantization and
  encoding.
\newblock In \emph{Advances in Neural Information Processing Systems}, pp.\
  1709--1720, 2017.

\bibitem[Allouah et~al.(2023{\natexlab{a}})Allouah, Farhadkhani, Guerraoui,
  Gupta, Pinot, and Stephan]{allouah2023fixing}
Allouah, Y., Farhadkhani, S., Guerraoui, R., Gupta, N., Pinot, R., and Stephan,
  J.
\newblock Fixing by mixing: A recipe for optimal byzantine ml under
  heterogeneity.
\newblock In \emph{International Conference on Artificial Intelligence and
  Statistics}, pp.\  1232--1300. PMLR, 2023{\natexlab{a}}.

\bibitem[Allouah et~al.(2023{\natexlab{b}})Allouah, Guerraoui, Gupta, Pinot,
  and Stephan]{allouah2023privacy}
Allouah, Y., Guerraoui, R., Gupta, N., Pinot, R., and Stephan, J.
\newblock On the privacy-robustness-utility trilemma in distributed learning.
\newblock In \emph{International Conference on Machine Learning},
  2023{\natexlab{b}}.

\bibitem[Baruch et~al.(2019)Baruch, Baruch, and Goldberg]{baruch2019little}
Baruch, G., Baruch, M., and Goldberg, Y.
\newblock A little is enough: Circumventing defenses for distributed learning.
\newblock In \emph{Proceedings of NeurIPS}, volume~32, pp.\  8635--8645, 2019.

\bibitem[Batir(2008)]{batir2008sharp}
Batir, N.
\newblock Sharp inequalities for factorial n.
\newblock \emph{Proyecciones (Antofagasta)}, 27\penalty0 (1):\penalty0 97--102,
  2008.

\bibitem[Bernstein et~al.(2018)Bernstein, Wang, Azizzadenesheli, and
  Anandkumar]{bernstein2018signsgd1}
Bernstein, J., Wang, Y.-X., Azizzadenesheli, K., and Anandkumar, A.
\newblock sign{SGD}: Compressed optimisation for non-convex problems.
\newblock In \emph{International Conference on Machine Learning}, pp.\
  560--569, 2018.

\bibitem[Bernstein et~al.(2019)Bernstein, Zhao, Azizzadenesheli, and
  Anandkumar]{bernstein2018signsgd2}
Bernstein, J., Zhao, J., Azizzadenesheli, K., and Anandkumar, A.
\newblock sign{SGD} with majority vote is communication efficient and byzantine
  fault tolerant.
\newblock In \emph{International Conference on Learning Representations}, 2019.

\bibitem[Bertsekas \& Tsitsiklis(2015)Bertsekas and
  Tsitsiklis]{bertsekas2015parallel}
Bertsekas, D. and Tsitsiklis, J.
\newblock \emph{Parallel and distributed computation: numerical methods}.
\newblock Athena Scientific, 2015.

\bibitem[Blanchard et~al.(2017)Blanchard, Guerraoui, Stainer,
  et~al.]{blanchard2017machine}
Blanchard, P., Guerraoui, R., Stainer, J., et~al.
\newblock Machine learning with adversaries: Byzantine tolerant gradient
  descent.
\newblock In \emph{Advances in Neural Information Processing Systems}, pp.\
  119--129, 2017.

\bibitem[Canonne et~al.(2020)Canonne, Kamath, and Steinke]{canonne2020discrete}
Canonne, C.~L., Kamath, G., and Steinke, T.
\newblock The discrete gaussian for differential privacy.
\newblock \emph{Advances in Neural Information Processing Systems},
  33:\penalty0 15676--15688, 2020.

\bibitem[Chaudhuri et~al.(2022)Chaudhuri, Guo, and
  Rabbat]{chaudhuri2022privacy}
Chaudhuri, K., Guo, C., and Rabbat, M.
\newblock Privacy-aware compression for federated data analysis.
\newblock In \emph{Uncertainty in Artificial Intelligence}, pp.\  296--306.
  PMLR, 2022.

\bibitem[Chen et~al.(2020{\natexlab{a}})Chen, Kairouz, and
  Ozgur]{chen2020breaking}
Chen, W.-N., Kairouz, P., and Ozgur, A.
\newblock Breaking the communication-privacy-accuracy trilemma.
\newblock \emph{Advances in Neural Information Processing Systems},
  33:\penalty0 3312--3324, 2020{\natexlab{a}}.

\bibitem[Chen et~al.(2022)Chen, Ozgur, and Kairouz]{chen2022poisson}
Chen, W.-N., Ozgur, A., and Kairouz, P.
\newblock The poisson binomial mechanism for unbiased federated learning with
  secure aggregation.
\newblock In \emph{International Conference on Machine Learning}, pp.\
  3490--3506. PMLR, 2022.

\bibitem[Chen et~al.(2023)Chen, Song, Ozgur, and Kairouz]{chen2023privacy}
Chen, W.-N., Song, D., Ozgur, A., and Kairouz, P.
\newblock Privacy amplification via compression: Achieving the optimal
  privacy-accuracy-communication trade-off in distributed mean estimation.
\newblock \emph{arXiv preprint arXiv:2304.01541}, 2023.

\bibitem[Chen et~al.(2020{\natexlab{b}})Chen, Chen, Sun, Wu, and
  Hong]{chen2019distributed}
Chen, X., Chen, T., Sun, H., Wu, S.~Z., and Hong, M.
\newblock Distributed training with heterogeneous data: Bridging median- and
  mean-based algorithms.
\newblock In \emph{Advances in Neural Information Processing Systems},
  volume~33, pp.\  21616--21626, 2020{\natexlab{b}}.

\bibitem[Dean et~al.(2012)Dean, Corrado, Monga, Chen, Devin, Mao, Senior,
  Tucker, Yang, Le, et~al.]{dean2012large}
Dean, J., Corrado, G., Monga, R., Chen, K., Devin, M., Mao, M., Senior, A.,
  Tucker, P., Yang, K., Le, Q.~V., et~al.
\newblock Large scale distributed deep networks.
\newblock In \emph{Advances in neural information processing systems}, pp.\
  1223--1231, 2012.

\bibitem[Deng et~al.(2009)Deng, Dong, Socher, Li, Li, and
  Fei-Fei]{deng2009imagenet}
Deng, J., Dong, W., Socher, R., Li, L.-J., Li, K., and Fei-Fei, L.
\newblock Imagenet: A large-scale hierarchical image database.
\newblock In \emph{2009 IEEE conference on computer vision and pattern
  recognition}, pp.\  248--255. Ieee, 2009.

\bibitem[Dong et~al.(2021)Dong, Roth, and Su]{dong2021gaussian}
Dong, J., Roth, A., and Su, W.
\newblock Gaussian differential privacy.
\newblock \emph{Journal of the Royal Statistical Society}, 2021.

\bibitem[Dwork et~al.(2006)Dwork, Kenthapadi, McSherry, Mironov, and
  Naor]{dwork2006our}
Dwork, C., Kenthapadi, K., McSherry, F., Mironov, I., and Naor, M.
\newblock Our data, ourselves: Privacy via distributed noise generation.
\newblock In \emph{Annual international conference on the theory and
  applications of cryptographic techniques}, pp.\  486--503. Springer, 2006.

\bibitem[Fang et~al.(2022)Fang, Li, Fan, and Li]{fang2022improved}
Fang, H., Li, X., Fan, C., and Li, P.
\newblock Improved convergence of differential private sgd with gradient
  clipping.
\newblock In \emph{The Eleventh International Conference on Learning
  Representations}, 2022.

\bibitem[Farhadkhani et~al.(2022)Farhadkhani, Guerraoui, Gupta, Pinot, and
  Stephan]{farhadkhani2022byzantine}
Farhadkhani, S., Guerraoui, R., Gupta, N., Pinot, R., and Stephan, J.
\newblock Byzantine machine learning made easy by resilient averaging of
  momentums.
\newblock In \emph{International Conference on Machine Learning}, pp.\
  6246--6283. PMLR, 2022.

\bibitem[Guerraoui et~al.(2021)Guerraoui, Gupta, Pinot, Rouault, and
  Stephan]{guerraoui2021differential}
Guerraoui, R., Gupta, N., Pinot, R., Rouault, S., and Stephan, J.
\newblock Differential privacy and byzantine resilience in {SGD}: Do they add
  up?
\newblock \emph{arXiv preprint arXiv:2102.08166}, 2021.

\bibitem[Guo et~al.(2023)Guo, Chaudhuri, Stock, and Rabbat]{guo2023privacy}
Guo, C., Chaudhuri, K., Stock, P., and Rabbat, M.
\newblock Privacy-aware compression for federated learning through numerical
  mechanism design.
\newblock In \emph{International Conference on Machine Learning}, pp.\
  11888--11904. PMLR, 2023.

\bibitem[Haddadpour et~al.(2020)Haddadpour, Kamani, Mokhtari, and
  Mahdavi]{haddadpour2020federated}
Haddadpour, F., Kamani, M.~M., Mokhtari, A., and Mahdavi, M.
\newblock Federated learning with compression: Unified analysis and sharp
  guarantees.
\newblock \emph{arXiv preprint arXiv:2007.01154}, 2020.

\bibitem[Jiang \& Agrawal(2018)Jiang and Agrawal]{jiang2018linear}
Jiang, P. and Agrawal, G.
\newblock A linear speedup analysis of distributed deep learning with sparse
  and quantized communication.
\newblock In \emph{Advances in Neural Information Processing Systems}, pp.\
  2525--2536, 2018.

\bibitem[Jin et~al.(2020)Jin, Huang, He, Dai, and Wu]{jin2020stochastic}
Jin, R., Huang, Y., He, X., Dai, H., and Wu, T.
\newblock Stochastic-{Sign} {SGD} for federated learning with theoretical
  guarantees.
\newblock \emph{arXiv preprint arXiv:2002.10940}, 2020.

\bibitem[Jin et~al.(2023)Jin, Su, Zhong, Zhang, Quek, and Dai]{jin2023breaking}
Jin, R., Su, Z., Zhong, C., Zhang, Z., Quek, T., and Dai, H.
\newblock Breaking the communication-privacy-accuracy tradeoff with
  $f$-differential privacy.
\newblock \emph{Advances in Neural Information Processing Systems}, 2023.

\bibitem[Jin et~al.(2024)Jin, Liu, Huang, He, Wu, and Dai]{jin2024sign}
Jin, R., Liu, Y., Huang, Y., He, X., Wu, T., and Dai, H.
\newblock Sign-based gradient descent with heterogeneous data: Convergence and
  byzantine resilience.
\newblock \emph{IEEE Transactions on Neural Networks and Learning Systems},
  2024.

\bibitem[Jorgensen et~al.(2018)Jorgensen, Chen, Milam, and
  Pavone]{jorgensen2018team}
Jorgensen, S., Chen, R.~H., Milam, M.~B., and Pavone, M.
\newblock The team surviving orienteers problem: routing teams of robots in
  uncertain environments with survival constraints.
\newblock \emph{Autonomous Robots}, 42:\penalty0 927--952, 2018.

\bibitem[Kairouz et~al.(2021{\natexlab{a}})Kairouz, Liu, and
  Steinke]{kairouz2021distributed}
Kairouz, P., Liu, Z., and Steinke, T.
\newblock The distributed discrete gaussian mechanism for federated learning
  with secure aggregation.
\newblock In \emph{International Conference on Machine Learning}, pp.\
  5201--5212. PMLR, 2021{\natexlab{a}}.

\bibitem[Kairouz et~al.(2021{\natexlab{b}})Kairouz, McMahan, Avent, Bellet,
  Bennis, Bhagoji, Bonawitz, Charles, Cormode, Cummings,
  et~al.]{kairouz2019advances}
Kairouz, P., McMahan, H.~B., Avent, B., Bellet, A., Bennis, M., Bhagoji, A.~N.,
  Bonawitz, K., Charles, Z., Cormode, G., Cummings, R., et~al.
\newblock Advances and open problems in federated learning.
\newblock \emph{Foundations and Trends in Machine Learning}, 14\penalty0 (1),
  2021{\natexlab{b}}.

\bibitem[Karimireddy et~al.(2019)Karimireddy, Rebjock, Stich, and
  Jaggi]{karimireddy2019error}
Karimireddy, S.~P., Rebjock, Q., Stich, S., and Jaggi, M.
\newblock Error feedback fixes sign{SGD} and other gradient compression
  schemes.
\newblock In \emph{International Conference on Machine Learning}, pp.\
  3252--3261, 2019.

\bibitem[Karimireddy et~al.(2021)Karimireddy, He, and
  Jaggi]{karimireddy2021learning}
Karimireddy, S.~P., He, L., and Jaggi, M.
\newblock Learning from history for byzantine robust optimization.
\newblock In \emph{International Conference on Machine Learning}, pp.\
  5311--5319. PMLR, 2021.

\bibitem[Karimireddy et~al.(2022)Karimireddy, He, and
  Jaggi]{karimireddy2022byzantine}
Karimireddy, S.~P., He, L., and Jaggi, M.
\newblock Byzantine-robust learning on heterogeneous datasets via bucketing.
\newblock In \emph{International Conference on Learning Representations}, 2022.

\bibitem[Li et~al.(2019)Li, Xu, Chen, Giannakis, and Ling]{li2019rsa}
Li, L., Xu, W., Chen, T., Giannakis, G.~B., and Ling, Q.
\newblock {RSA}: Byzantine-robust stochastic aggregation methods for
  distributed learning from heterogeneous datasets.
\newblock In \emph{Proceedings of the AAAI Conference on Artificial
  Intelligence}, volume~33, pp.\  1544--1551, 2019.

\bibitem[Mhamdi et~al.(2018)Mhamdi, Guerraoui, and Rouault]{mhamdi2018hidden}
Mhamdi, E. M.~E., Guerraoui, R., and Rouault, S.
\newblock The hidden vulnerability of distributed learning in byzantium.
\newblock In \emph{International Conference on Machine Learning}, pp.\
  3521--3530, 2018.

\bibitem[Naseri et~al.(2022)Naseri, Hayes, and De~Cristofaro]{naseri2020local}
Naseri, M., Hayes, J., and De~Cristofaro, E.
\newblock Local and central differential privacy for robustness and privacy in
  federated learning.
\newblock In \emph{Network and Distributed System Security Symposium (NDSS)},
  2022.

\bibitem[Nguyen et~al.(2022)Nguyen, Rieger, De~Viti, Chen, Brandenburg, Yalame,
  M{\"o}llering, Fereidooni, Marchal, Miettinen, et~al.]{nguyen2022flame}
Nguyen, T.~D., Rieger, P., De~Viti, R., Chen, H., Brandenburg, B.~B., Yalame,
  H., M{\"o}llering, H., Fereidooni, H., Marchal, S., Miettinen, M., et~al.
\newblock {FLAME}: Taming backdoors in federated learning.
\newblock In \emph{31st USENIX Security Symposium (USENIX Security 22)}, pp.\
  1415--1432, 2022.

\bibitem[Rieke et~al.(2020)Rieke, Hancox, Li, Milletari, Roth, Albarqouni,
  Bakas, Galtier, Landman, Maier-Hein, et~al.]{rieke2020future}
Rieke, N., Hancox, J., Li, W., Milletari, F., Roth, H.~R., Albarqouni, S.,
  Bakas, S., Galtier, M.~N., Landman, B.~A., Maier-Hein, K., et~al.
\newblock The future of digital health with federated learning.
\newblock \emph{NPJ digital medicine}, 3\penalty0 (1):\penalty0 1--7, 2020.

\bibitem[Safaryan \& Richt{\'a}rik(2021)Safaryan and
  Richt{\'a}rik]{safaryan2021stochastic}
Safaryan, M. and Richt{\'a}rik, P.
\newblock Stochastic sign descent methods: New algorithms and better theory.
\newblock In \emph{International Conference on Machine Learning}, pp.\
  9224--9234. PMLR, 2021.

\bibitem[Samuels(1965)]{samuels1965number}
Samuels, S.~M.
\newblock On the number of successes in independent trials.
\newblock \emph{The Annals of Mathematical Statistics}, pp.\  1272--1278, 1965.

\bibitem[Stich et~al.(2018)Stich, Cordonnier, and Jaggi]{stich2018sparsified}
Stich, S.~U., Cordonnier, J.~B., and Jaggi, M.
\newblock Sparsified {SGD} with memory.
\newblock In \emph{Advances in Neural Information Processing Systems}, pp.\
  4447--4458, 2018.

\bibitem[Sun et~al.(2019)Sun, Kairouz, Suresh, and McMahan]{sun2019can}
Sun, Z., Kairouz, P., Suresh, A.~T., and McMahan, H.~B.
\newblock Can you really backdoor federated learning?
\newblock \emph{arXiv preprint arXiv:1911.07963}, 2019.

\bibitem[Xiang \& Su(2023)Xiang and Su]{xiang2023distributed}
Xiang, M. and Su, L.
\newblock Distributed non-convex optimization with one-bit compressors on
  heterogeneous data: Efficient and resilient algorithms.
\newblock \emph{arXiv preprint arXiv:2210.00665v2}, 2023.

\bibitem[Xie et~al.(2018)Xie, Koyejo, and Gupta]{xie2018generalized}
Xie, C., Koyejo, O., and Gupta, I.
\newblock Generalized byzantine-tolerant {SGD}.
\newblock \emph{arXiv preprint arXiv:1802.10116}, 2018.

\bibitem[Xie et~al.(2019)Xie, Koyejo, and Gupta]{xie2019slsgd}
Xie, C., Koyejo, O., and Gupta, I.
\newblock {SLSGD}: Secure and efficient distributed on-device machine learning.
\newblock In \emph{Joint European Conference on Machine Learning and Knowledge
  Discovery in Databases}, pp.\  213--228. Springer, 2019.

\bibitem[Xie et~al.(2020)Xie, Koyejo, and Gupta]{xie2020fall}
Xie, C., Koyejo, O., and Gupta, I.
\newblock Fall of empires: Breaking byzantine-tolerant sgd by inner product
  manipulation.
\newblock In \emph{Uncertainty in Artificial Intelligence}, pp.\  261--270.
  PMLR, 2020.

\bibitem[Yang et~al.(2021)Yang, Fang, and Liu]{yang2021achieving}
Yang, H., Fang, M., and Liu, J.
\newblock Achieving linear speedup with partial worker participation in
  non-{IID} federated learning.
\newblock In \emph{International Conference on Learning Representations}, 2021.

\bibitem[Yin et~al.(2018)Yin, Chen, Kannan, and Bartlett]{yin2018byzantine}
Yin, D., Chen, Y., Kannan, R., and Bartlett, P.
\newblock Byzantine-robust distributed learning: Towards optimal statistical
  rates.
\newblock In \emph{International Conference on Machine Learning}, pp.\
  5650--5659, 2018.

\bibitem[Yue et~al.(2023)Yue, Jin, Wong, Baron, and Dai]{yue2023gradient}
Yue, K., Jin, R., Wong, C.-W., Baron, D., and Dai, H.
\newblock Gradient obfuscation gives a false sense of security in federated
  learning.
\newblock In \emph{32nd USENIX Security Symposium (USENIX Security 23)}, pp.\
  6381--6398, 2023.

\bibitem[Zhang et~al.(2022)Zhang, Chen, Hong, Wu, and
  Yi]{zhang2022understanding}
Zhang, X., Chen, X., Hong, M., Wu, Z.~S., and Yi, J.
\newblock Understanding clipping for federated learning: Convergence and
  client-level differential privacy.
\newblock In \emph{International Conference on Machine Learning, ICML 2022},
  2022.

\bibitem[Zhang \& Hu(2023)Zhang and Hu]{zhang2023byzantine}
Zhang, Z. and Hu, R.
\newblock Byzantine-robust federated learning with variance reduction and
  differential privacy.
\newblock \emph{arXiv preprint arXiv:2309.03437}, 2023.

\bibitem[Zhu et~al.(2023)Zhu, Wang, Pang, Wang, Jiao, Song, and
  Jordan]{zhu2023byzantine}
Zhu, B., Wang, L., Pang, Q., Wang, S., Jiao, J., Song, D., and Jordan, M.~I.
\newblock Byzantine-robust federated learning with optimal statistical rates.
\newblock In \emph{International Conference on Artificial Intelligence and
  Statistics}, pp.\  3151--3178. PMLR, 2023.

\bibitem[Zhu \& Ling(2022)Zhu and Ling]{zhu2022bridging}
Zhu, H. and Ling, Q.
\newblock Bridging differential privacy and byzantine-robustness via model
  aggregation.
\newblock \emph{International Joint Conferences on Artificial Intelligence},
  2022.

\bibitem[Zhu \& Blaschko(2023)Zhu and Blaschko]{zhu2023improving}
Zhu, J. and Blaschko, M.~B.
\newblock Improving differentially private sgd via randomly sparsified
  gradients.
\newblock \emph{Transactions on Machine Learning Research}, 2023.

\bibitem[Zhu et~al.(2019)Zhu, Liu, and Han]{zhu2019deep}
Zhu, L., Liu, Z., and Han, S.
\newblock Deep leakage from gradients.
\newblock In \emph{Advances in Neural Information Processing Systems}, 2019.

\end{thebibliography}
\bibliographystyle{icml2024}

\newpage
\appendix
\onecolumn
\setcounter{Corollary}{0}
\setcounter{Theorem}{1}
\setcounter{Remark}{7}
\setlength{\abovedisplayskip}{1pt}
\setlength{\belowdisplayskip}{1pt}

\section{Additional Experimental Results}\label{suppadditionalresults}

\begin{table*}[h]
\vspace{-0.2in}
\caption{Test Accuracy on MNIST with $A/B = 0.1$ (200 communication rounds)}
\label{table_mnist0_1}
\begin{center}
\begin{sc}
\begin{tabular}{ccccccc}
\toprule
$\mu$ & 0.1 & 0.2 & 0.3 & 0.5 & 1 \\
\midrule
\makecell{Gaussian Noise \&\\  Random Sparsification} & $53.22\pm 1.91\%$ & $69.21\pm 1.23\%$ & $75.08\pm 0.89\%$ & $80.14\pm 1.29\%$ & $87.41\pm 0.31\%$ \\
\textit{TernaryMean} & $56.34\pm 1.99\%$ & $69.73\pm 1.17\%$ & $76.74\pm 1.03\%$ & $82.84\pm 0.84\%$ & $86.99\pm 0.78\%$ \\
\textit{TernaryVote} & $55.57\pm 1.50\%$ & $69.91\pm 1.69\%$ & $76.19\pm 0.83\%$ & $81.38\pm 0.42\%$ & $86.43\pm 0.44\%$ \\
\bottomrule
\end{tabular}
\end{sc}
\end{center}
\vspace{-0.1in}
\end{table*}

\begin{table*}[h]
\vspace{-0.2in}
\caption{Test Accuracy on Fashion-MNIST with $A/B = 0.1$ (200 communication rounds)}
\label{table_fashionmnist0_1}
\begin{center}
\begin{sc}
\begin{tabular}{ccccccc}
\toprule
$\mu$ & 0.1 & 0.2 & 0.3 & 0.5 & 1 \\
\midrule
\makecell{Gaussian Noise \&\\  Random Sparsification} & $50.79\pm 3.48\%$ & $67.43\pm 0.99\%$ & $70.32\pm 1.12\%$ & $73.13\pm 0.47\%$ & $76.24\pm 0.14\%$ \\
\textit{TernaryMean} & $57.53\pm 2.17\%$ & $69.26\pm 0.75\%$ & $72.77\pm 0.97\%$ & $73.38\pm 0.55\%$ & $76.07\pm 0.83\%$ \\
\textit{TernaryVote} & $57.80\pm 3.23\%$ & $68.91\pm 1.16\%$ & $71.50\pm 0.76\%$ & $74.58\pm 0.51\%$ & $77.00\pm 0.27\%$ \\
\bottomrule
\end{tabular}
\end{sc}
\end{center}
\vspace{-0.1in}
\end{table*}

Table. \ref{table_mnist0_1} and Table. \ref{table_fashionmnist0_1} compare \textit{TernaryMean} and \textit{TernaryVote} with the baseline algorithm on MNIST and Fashion-MNIST, respectively, given $A/B= 0.1$. It can be observed that \textit{TernaryMean} and \textit{TernaryVote} outperform the baseline in the high-privacy scenario (i.e., small $\mu$). As $\mu$ increases (i.e., the privacy requirement becomes less stringent), \textit{TernaryMean} and \textit{TernaryVote} do not necessarily outperform the baseline. For instance, when $\mu = 1$, the baseline algorithm achieves a higher test accuracy than \textit{TernaryMean}. We note that the ternary compressor is a combination of the one-bit compressor \cite{jin2020stochastic,xiang2023distributed} and random sparsification. When privacy is less of a concern, the bias introduced by one-bit compression may be larger than that of the Gaussian noise (for the coordinates that are not zeroed out), which leads to performance degradation. In this case, however, the proposed \textit{TernaryMean} and \textit{TernaryVote} still enjoy savings in communication overhead.

\section{Proofs of Theoretical Results}
\subsection{Proof of Theorem \ref{SuppTernary_Batch_scalar}}\label{ProofPrivacyAnalysisforTernaryScalar}
\begin{Theorem}\label{SuppTernary_Batch_scalar}
Assuming that $B > A + c$, the ternary compressor is $f(\alpha)$-DP for the scalar $x_{i}$ with
\begin{equation}
\begin{split}
&f(\alpha) =
\begin{cases}
\hfill 1 - \frac{Ab-(b-2)c}{(A-c)b}\alpha, \hfill \text{for $\alpha \in [0,\frac{A-c}{2B}]$},\\
\hfill 1-\frac{c}{Bb}- \alpha, \hfill \text{for $\alpha \in [\frac{A-c}{2B}, 1-\frac{Ab-(b-2)c}{2Bb}]$},\\
\hfill \frac{(A-c)b}{Ab-(b-2)c}(1 - \alpha), \hfill \text{for $\alpha \in [1-\frac{Ab-(b-2)c}{2Bb},1]$}.\\
\end{cases}
\end{split}
\end{equation}
\end{Theorem}

The proof of Theorem \ref{SuppTernary_Batch_scalar} utilizes the following lemma from \cite{jin2023breaking}.

\begin{Lemma}[\cite{jin2023breaking}]\label{generaltheorem}
For two neighboring datasets $S$ and $S'$, suppose that the range of the randomized mechanism $\mathcal{R}(\mathcal{M}(S)) \cup \mathcal{R}(\mathcal{M}(S')) = \mathbb{Z}_{\mathcal{M}}^{U} = [\mathcal{Z}^{U}_{L},\dots,\mathcal{Z}^{U}_{R}] \subset \mathbb{Z}$ and $\mathcal{R}(\mathcal{M}(S)) \cap \mathcal{R}(\mathcal{M}(S')) = \mathbb{Z}_{\mathcal{M}}^{I} = [\mathcal{Z}_{L}^{I},\dots,\mathcal{Z}_{R}^{I}] \subset \mathbb{Z}$. Let $X = \mathcal{M}(S)$ and $Y = \mathcal{M}(S')$. Then,

Case \textbf{(1)} If $\mathcal{M}(S) \in [\mathcal{Z}_{L}^{I},\mathcal{Z}_{L}^{I}+1,\dots,\mathcal{Z}^{U}_{R}]$, $\mathcal{M}(S') \in [\mathcal{Z}^{U}_{L},\mathcal{Z}^{U}_{L}+1,\dots,\mathcal{Z}_{R}^{I}]$, and $\frac{P(Y = k)}{P(X = k)}$ is a decreasing function of $k$ for $k \in \mathbb{Z}_{\mathcal{M}}^{I}$, the tradeoff function in Definition \ref{tradeofffunction} is given by
\begin{equation}\label{general1equation}
\begin{split}
&\beta_{\phi}^{+}(\alpha) = \begin{cases}
P(Y \geq k) + \frac{P(Y=k)P(X < k)}{P(X=k)} - \frac{P(Y=k)}{P(X=k)}\alpha,  \hfill \text{if $\alpha \in (P(X < k), P(X \leq k)]$, $k \in [\mathcal{Z}_{L}^{I},\mathcal{Z}_{R}^{I}]$}. \\
0, \hfill \text{if $\alpha \in (P(X < \mathcal{Z}_{R}^{I}+1),1]$.}
\end{cases}
\end{split}
\end{equation}

Case \textbf{(2)} If $\mathcal{M}(S) \in [\mathcal{Z}^{U}_{L},\mathcal{Z}^{U}_{L}+1,\cdots,\mathcal{Z}_{R}^{I}]$, $\mathcal{M}(S') \in [\mathcal{Z}_{L}^{I},\mathcal{Z}_{L}^{I}+1,\cdots,\mathcal{Z}^{U}_{R}]$, and $\frac{P(Y = k)}{P(X = k)}$ is an increasing function of $k$ for $k \in \mathbb{Z}_{\mathcal{M}}^{I}$, the tradeoff function in Definition \ref{tradeofffunction} is given by
\begin{equation}\label{general1equation2}
\begin{split}
&\beta_{\phi}^{-}(\alpha) = \begin{cases}
P(Y \leq k) + \frac{P(Y=k)P(X > k)}{P(X=k)} - \frac{P(Y=k)}{P(X=k)}\alpha,  \hfill \text{if $\alpha \in (P(X > k), P(X \geq k)]$, $k \in [\mathcal{Z}_{L}^{I},\mathcal{Z}_{R}^{I}]$}. \\
0, \hfill \text{if $\alpha \in (P(X > \mathcal{Z}_{L}^{I}-1),1]$.}
\end{cases}
\end{split}
\end{equation}
\end{Lemma}
Given Lemma \ref{generaltheorem}, we are ready to prove Theorem \ref{SuppTernary_Batch_scalar}.
\begin{proof}
Let $Y = ternary(\frac{1}{b}(y +x'_{i}),A,B)$ and $X = ternary(\frac{1}{b}(y +x_{i}),A,B)$, we have
\begin{equation}
\begin{split}
&\frac{P(Y=-1)}{P(X=-1)} = \frac{A-\frac{1}{b}(y +x'_{i})}{A-\frac{1}{b}(y +x_{i})},   \\
&\frac{P(Y=0)}{P(X=0)} = 1,  \\
&\frac{P(Y=1)}{P(X=1)} = \frac{A+\frac{1}{b}(y +x'_{i})}{A+\frac{1}{b}(y +x_{i})}.
\end{split}
\end{equation}
When $x_{i} > x'_{i}$, we have $\frac{1}{b}(y+x_{i}) > \frac{1}{b}(y+x'_{i})$. It can be observed that $\frac{P(Y=k)}{P(X=k)}$ is a decreasing function of $k$. According to Lemma \ref{generaltheorem}, we have
\begin{equation}
\begin{split}
&\beta_{\phi}^{+}(\alpha) =
\begin{cases}
\hfill 1 - \frac{A-\frac{1}{b}(y+x'_{i})}{A-\frac{1}{b}(y+x_{i})}\alpha, \hfill ~~\text{for $\alpha \in [0,\frac{A-\frac{1}{b}(y+x_{i})}{2B}]$},\\
\hfill 1+\frac{x'_{i}-x_{i}}{2Bb}- \alpha, \hfill ~~\text{for $\alpha \in [\frac{A-\frac{1}{b}(y+x_{i})}{2B}, 1-\frac{A+\frac{1}{b}(y+x_{i}) }{2B}]$},\\
\hfill \frac{A+\frac{1}{b}(y+x'_{i})}{A+\frac{1}{b}(y+x_{i})} - \frac{A+\frac{1}{b}(y+x'_{i})}{A+\frac{1}{b}(y+x_{i})}\alpha, \hfill ~~\text{for $\alpha \in [1-\frac{A+\frac{1}{b}(y+x_{i}) }{2B},1]$}.\\
\end{cases}
\end{split}
\end{equation}

When $x_{i} < x'_{i}$, we have $\frac{1}{b}(y+x_{i}) < \frac{1}{b}(y+x'_{i})$. It can be observed that $\frac{P(Y=k)}{P(X=k)}$ is an increasing function of $k$. According to Lemma \ref{generaltheorem}, we have
\begin{equation}
\begin{split}
&\beta_{\phi}^{-}(\alpha) =
\begin{cases}
\hfill 1 - \frac{A+\frac{1}{b}(y+x'_{i})}{A+\frac{1}{b}(y+x_{i})}\alpha, \hfill ~~\text{for $\alpha \in [0,\frac{A+\frac{1}{b}(y+x_{i})}{2B}]$},\\
\hfill 1-\frac{x'_{i}-x_{i}}{2Bb}- \alpha, \hfill ~~\text{for $\alpha \in [\frac{A+\frac{1}{b}(y+x_{i})}{2B}, 1-\frac{A-\frac{1}{b}(y+x_{i}) }{2B}]$},\\
\hfill \frac{A-\frac{1}{b}(y+x'_{i})}{A-\frac{1}{b}(y+x_{i})} - \frac{A-\frac{1}{b}(y+x'_{i})}{A-\frac{1}{b}(y+x_{i})}\alpha, \hfill ~~\text{for $\alpha \in [1-\frac{A-\frac{1}{b}(y+x_{i}) }{2B},1]$}.\\
\end{cases}
\end{split}
\end{equation}

For any given $y$, the infimum of $\beta_{\phi}^{+}(\alpha)$ is attained when $x_{i} = c$ and $x'_{i}=-c$, while the infimum of $\beta_{\phi}^{-}(\alpha)$ is attained when $x_{i} = -c$ and $x'_{i}=c$. As a result, we have
\begin{equation}
\begin{split}
&\beta_{\phi,\text{inf}}^{+}(\alpha) =
\begin{cases}
\hfill 1 - \frac{A-\frac{1}{b}(y-c)}{A-\frac{1}{b}(y+c)}\alpha, \hfill ~~\text{for $\alpha \in [0,\frac{A-\frac{1}{b}(y+c)}{2B}]$},\\
\hfill 1-\frac{c}{Bb}- \alpha, \hfill ~~\text{for $\alpha \in [\frac{A-\frac{1}{b}(y+c)}{2B}, 1-\frac{A+\frac{1}{b}(y+c) }{2B}]$},\\
\hfill \frac{A+\frac{1}{b}(y-c)}{A+\frac{1}{b}(y+c)} - \frac{A+\frac{1}{b}(y-c)}{A+\frac{1}{b}(y+c)}\alpha, \hfill ~~\text{for $\alpha \in [1-\frac{A+\frac{1}{b}(y+c) }{2B},1]$},\\
\end{cases}
\end{split}
\end{equation}

and

\begin{equation}
\begin{split}
&\beta_{\phi,\text{inf}}^{-}(\alpha) =
\begin{cases}
\hfill 1 - \frac{A+\frac{1}{b}(y+c)}{A+\frac{1}{b}(y-c)}\alpha, \hfill ~~\text{for $\alpha \in [0,\frac{A+\frac{1}{b}(y-c)}{2B}]$},\\
\hfill 1-\frac{c}{Bb}- \alpha, \hfill ~~\text{for $\alpha \in [\frac{A+\frac{1}{b}(y-c)}{2B}, 1-\frac{A-\frac{1}{b}(y-c) }{2B}]$},\\
\hfill \frac{A-\frac{1}{b}(y+c)}{A-\frac{1}{b}(y-c)} - \frac{A-\frac{1}{b}(y+c)}{A-\frac{1}{b}(y-c)}\alpha, \hfill ~~\text{for $\alpha \in [1-\frac{A-\frac{1}{b}(y-c)}{2B},1]$}.\\
\end{cases}
\end{split}
\end{equation}

Assume that $B > A+c$, we have $[\frac{A-\frac{1}{b}(y+c)}{2B}, 1-\frac{A+\frac{1}{b}(y+c) }{2B}] \cap [\frac{A+\frac{1}{b}(y-c)}{2B}, 1-\frac{A-\frac{1}{b}(y-c) }{2B}] \neq \emptyset$. In this case, when $y > 0$,

\begin{equation}
\begin{split}
&\min\{\beta_{\phi,\text{inf}}^{+}(\alpha),\beta_{\phi,\text{inf}}^{-}(\alpha)|y>0\} =
\begin{cases}
\hfill 1 - \frac{A-\frac{1}{b}(y-c)}{A-\frac{1}{b}(y+c)}\alpha, \hfill ~~\text{for $\alpha \in [0,\frac{A-\frac{1}{b}(y+c)}{2B}]$},\\
\hfill 1-\frac{c}{Bb}- \alpha, \hfill ~~\text{for $\alpha \in [\frac{A-\frac{1}{b}(y+c)}{2B}, 1-\frac{A-\frac{1}{b}(y-c) }{2B}]$},\\
\hfill \frac{A-\frac{1}{b}(y+c)}{A-\frac{1}{b}(y-c)} - \frac{A-\frac{1}{b}(y+c)}{A-\frac{1}{b}(y-c)}\alpha, \hfill ~~\text{for $\alpha \in [1-\frac{A-\frac{1}{b}(y-c) }{2B},1]$}.\\
\end{cases}
\end{split}
\end{equation}

When $y < 0$,

\begin{equation}
\begin{split}
&\min\{\beta_{\phi,\text{inf}}^{+}(\alpha),\beta_{\phi,\text{inf}}^{-}(\alpha)|y<0\} =
\begin{cases}
\hfill 1 - \frac{A+\frac{1}{b}(y+c)}{A+\frac{1}{b}(y-c)}\alpha, \hfill ~~\text{for $\alpha \in [0,\frac{A+\frac{1}{b}(y-c)}{2B}]$},\\
\hfill 1-\frac{c}{Bb}- \alpha, \hfill ~~\text{for $\alpha \in [\frac{A+\frac{1}{b}(y-c)}{2B}, 1-\frac{A+\frac{1}{b}(y+c) }{2B}]$},\\
\hfill \frac{A+\frac{1}{b}(y-c)}{A+\frac{1}{b}(y+c)} - \frac{A+\frac{1}{b}(y-c)}{A+\frac{1}{b}(y+c)}\alpha, \hfill ~~\text{for $\alpha \in [1-\frac{A+\frac{1}{b}(y+c) }{2B},1]$}.\\
\end{cases}
\end{split}
\end{equation}

It can be verified that $\min\{\beta_{\phi,\text{inf}}^{+}(\alpha),\beta_{\phi,\text{inf}}^{-}(\alpha)|y>0\}$ and $\min\{\beta_{\phi,\text{inf}}^{+}(\alpha),\beta_{\phi,\text{inf}}^{-}(\alpha)|y<0\}$ are minimized when $y = (b-1)c$ and $y=-(b-1)c$, respectively. As a result,

\begin{equation}
\begin{split}
&f(\alpha) =
\begin{cases}
\hfill 1 - \frac{A-\frac{(b-2)c}{b}}{A-c}\alpha, \hfill ~~\text{for $\alpha \in [0,\frac{A-c}{2B}]$},\\
\hfill 1-\frac{c}{Bb}- \alpha, \hfill ~~\text{for $\alpha \in [\frac{A-c}{2B}, 1-\frac{A-\frac{(b-2)c}{b}}{2B}]$},\\
\hfill \frac{A-c}{A-\frac{(b-2)c}{b}} - \frac{A-c}{A-\frac{(b-2)c}{b}}\alpha, \hfill ~~\text{for $\alpha \in [1-\frac{A-\frac{(b-2)c}{b}}{2B},1]$},\\
\end{cases}
\end{split}
\end{equation}
which completes the proof.
\end{proof}

\subsection{Proof of Theorem \ref{SuppTernary_Batch_vector}}\label{PrivacyAnalysisforTernaryVector}
\begin{Theorem}\label{SuppTernary_Batch_vector}
Assuming that $B > A + c$, the ternary compressor is $f(\alpha)$-DP for the vector $\boldsymbol{x}_{i}$ with
\begin{equation}
\begin{split}
G_{\mu}(\alpha+\gamma)-\gamma \leq f(\alpha) \leq G_{\mu}(\alpha-\gamma)+\gamma,
\end{split}
\end{equation}
in which
\begin{equation}
\mu = \frac{2\sqrt{d}c}{\sqrt{(A-c)Bb^2+Bbc-c^2}},
\end{equation}
\begin{equation}
\begin{split}
\gamma = \frac{0.56\left[\frac{A-c}{2B}\left|1+\frac{c}{Bb}\right|^3+\frac{Ab-(b-2)c}{2Bb}\left|1-\frac{c}{Bb}\right|^3\right]}{(\frac{(A-c)b+c}{Bb}-\frac{c^2}{B^2b^2})^{3/2}d^{1/2}}+\frac{0.56\left[\left(1-\frac{(A-c)b+c}{Bb}\right)\left|\frac{c}{Bb}\right|^{3}\right]}{(\frac{(A-c)b+c}{Bb}-\frac{c^2}{B^2b^2})^{3/2}d^{1/2}}.
\end{split}
\end{equation}
\end{Theorem}

Before proving Theorem \ref{SuppTernary_Batch_vector}, we first define the following functions as in \cite{dong2021gaussian},
\begin{equation}
    \text{kl}(f) = -\int_{0}^{1}\log|f'(x)|dx,
\end{equation}
\begin{equation}
    \kappa_{2}(f) = \int_{0}^{1}\log^{2}|f'(x)|dx,
\end{equation}
\begin{equation}
    \kappa_{3}(f) = \int_{0}^{1}|\log|f'(x)||^3dx,
\end{equation}
\begin{equation}
    \bar{\kappa}_{3}(f) = \int_{0}^{1}|\log|f'(x)|+\text{kl}(f)|^3dx.
\end{equation}

The central limit theorem for $f$-DP is formally introduced as follows.

\begin{Lemma}[\cite{dong2021gaussian}]\label{cltfdp}
Let $f_{1},...,f_{n}$ be symmetric trade-off functions such that $\kappa_{3}(f_{i}) < \infty$ for all $1 \leq i \leq d$. Denote
\begin{equation}\nonumber
\mu = \frac{2||\text{kl}||_{1}}{\sqrt{||\kappa_{2}||_{1}-||\text{kl}||_{2}^{2}}}, \text{and~~} \gamma = \frac{0.56||\bar{\kappa}_{3}||_{1}}{(||\kappa_{2}||_{1}-||\text{kl}||_{2}^{2})^{3/2}},
\end{equation}
and assume $\gamma < \frac{1}{2}$. Then, for all $\alpha \in [\gamma, 1-\gamma]$, we have
\begin{equation}
    G_{\mu}(\alpha+\gamma)-\gamma \leq f_{1}\otimes f_{2}\otimes\cdots \otimes f_{d}(\alpha) \leq G_{\mu}(\alpha-\gamma)+\gamma.
\end{equation}
\end{Lemma}

Given Lemma \ref{cltfdp}, we are ready to prove Theorem \ref{SuppTernary_Batch_vector}.

\begin{proof}
Given $f_{i}(\alpha)$ in (\ref{fternarymechanism}), we have
\begin{equation}
\begin{split}
\text{kl}(f) &= -\left[\frac{A-c}{2B}\log\left(\frac{A-c+\frac{2c}{b}}{A-c}\right) + \frac{A-\frac{b-2}{b}c}{2B}\log\left(\frac{A-c}{A-c+\frac{2c}{b}}\right)\right]\\
&=\left[\frac{A-\frac{b-2}{b}c}{2B} - \frac{A-c}{2B}\right]\log\left(\frac{A-c+\frac{2c}{b}}{A-c}\right)\\
&=\frac{c}{Bb}\log\left(\frac{A-c+\frac{2c}{b}}{A-c}\right),
\end{split}
\end{equation}
\begin{equation}
\begin{split}
\kappa_{2}(f) &= \left[\frac{A-c}{2B}\log^2\left(\frac{A-c+\frac{2c}{b}}{A-c}\right) + \frac{A-\frac{b-2}{b}c}{2B}\log^2\left(\frac{A-c}{A-c+\frac{2c}{b}}\right)\right] =\frac{A-c+\frac{c}{b}}{B}\log^2\left(\frac{A-c+\frac{2c}{b}}{A-c}\right),
\end{split}
\end{equation}
\begin{equation}
\begin{split}
\kappa_{3}(f) &= \left[\frac{A-c}{2B}\left|\log\left(\frac{A-c+\frac{2c}{b}}{A-c}\right)\right|^3 + \frac{A-\frac{b-2}{b}c}{2B}\left|\log\left(\frac{A-c}{A-c+\frac{2c}{b}}\right)\right|^3\right] =\frac{A-c+\frac{c}{b}}{B}\left|\log\left(\frac{A-c+\frac{2c}{b}}{A-c}\right)\right|^3,
\end{split}
\end{equation}
\begin{equation}
\begin{split}
\bar{\kappa}_{3}(f) = \left[\frac{A-c}{2B}\left|1+\frac{c}{Bb}\right|^3+\frac{A-\frac{b-2}{b}c}{2B}\left|1-\frac{c}{Bb}\right|^3+\left(1-\frac{A-c+\frac{c}{b}}{B}\right)\left|\frac{c}{Bb}\right|^{3}\right]\left|\log\left(\frac{A-c+\frac{2c}{b}}{A-c}\right)\right|^3 .
\end{split}
\end{equation}

The corresponding $\mu$ and $\gamma$ are given as follows
\begin{equation}
\mu = \frac{2d\frac{c}{Bb}}{\sqrt{\frac{A-c+\frac{c}{b}}{B}d-\frac{c^2}{B^2b^2}d}} = \frac{2\sqrt{d}c}{\sqrt{(A-c+\frac{c}{b})Bb^2-c^2}},
\end{equation}

\begin{equation}
\gamma = \frac{0.56\left[\frac{A-c}{2B}\left|1+\frac{c}{Bb}\right|^3+\frac{A-\frac{b-2}{b}c}{2B}\left|1-\frac{c}{Bb}\right|^3+\left(1-\frac{A-c+\frac{c}{b}}{B}\right)\left|\frac{c}{Bb}\right|^{3}\right]}{(\frac{A-c+\frac{c}{b}}{B}-\frac{c^2}{B^2b^2})^{3/2}d^{1/2}},
\end{equation}
which completes the proof.
\end{proof}

\subsection{Proof of Theorem \ref{Suppconvergerate_schemeI}}
\begin{Theorem}[\textbf{Convergence of \textit{TernaryMean}}]\label{Suppconvergerate_schemeI}
Suppose Assumptions \ref{A1}-\ref{A4} are satisfied, then by running Algorithm \ref{DPSGDAlgorithm} with \textit{TernaryMean} for $T$ iterations, we have
\begin{equation}\label{SuppConvergenceEquation}
\begin{split}
\frac{1}{T}\sum_{t=1}^{T}||\nabla F(\boldsymbol{w}^{(t)})||_{2}^{2} \leq \frac{F(\boldsymbol{w}^{(0)}) - F^{*}}{T\left(\frac{\eta}{B} - \frac{L\eta^2}{2B^{2}}\right)} + \frac{L\eta^2}{2B^{2}\left(\frac{\eta}{B} - \frac{L\eta^2}{2B^{2}}\right)}\left[\frac{ABd}{M} + \frac{||\bar{\boldsymbol{\sigma}}||_{2}^{2}}{M}\right].
\end{split}
\end{equation}
\end{Theorem}

The proof of Theorem \ref{Suppconvergerate_schemeI} follows the well-known strategy of relating the norm of the gradient to the expected improvement of the global objective in a single iteration. Then accumulating the improvement over the iterations yields the convergence rate of the algorithm.

\begin{proof}
According to Assumption \ref{A2}, we have
\begin{equation}\label{suppprot1}
\begin{split}
&F(\boldsymbol{w}^{(t+1)}) - F(\boldsymbol{w}^{(t)}) \\
&\leq \langle\nabla F(\boldsymbol{w}^{(t)}), \boldsymbol{w}^{(t+1)}-\boldsymbol{w}^{(t)}\rangle + \frac{L}{2}||\boldsymbol{w}^{(t+1)}-\boldsymbol{w}^{(t)}||_{2}^{2} \\
& =-\eta \left\langle\nabla F(\boldsymbol{w}^{(t)}), \frac{1}{M}\sum_{m=1}^{M}ternary(\boldsymbol{g}_{m}^{(t)},A,B)\right\rangle +\frac{L\eta^{2}}{2}\bigg|\bigg|\frac{1}{M}\sum_{m=1}^{M}ternary(\boldsymbol{g}_{m}^{(t)},A,B)\bigg|\bigg|^2. \\
\end{split}
\end{equation}

Noticing that $\mathbb{E}[Bternary(\boldsymbol{g}_{m}^{(t)},A,B)] = \boldsymbol{g}_{m}^{(t)}$, we have
\begin{equation}
\begin{split}
&\mathbb{E}\left[\bigg|\bigg|\frac{1}{M}\sum_{m=1}^{M}Bternary(\boldsymbol{g}_{m}^{(t)},A,B)\bigg|\bigg|^2\right] \\
&= \mathbb{E}\left[\bigg|\bigg|\frac{1}{M}\sum_{m=1}^{M}Bternary(\boldsymbol{g}_{m}^{(t)},A,B) - \frac{1}{M}\sum_{m=1}^{M}\boldsymbol{g}_{m}^{(t)}\bigg|\bigg|^2\right] + \mathbb{E}\left[\bigg|\bigg|\frac{1}{M}\sum_{m=1}^{M}\boldsymbol{g}_{m}^{(t)}\bigg|\bigg|^2\right]\\
&= \frac{1}{M^{2}}\sum_{m=1}^{M}\mathbb{E}\left[ABd - ||\boldsymbol{g}_{m}^{(t)}||^{2}\right] + \mathbb{E}\left[\bigg|\bigg|\frac{1}{M}\sum_{m=1}^{M}\boldsymbol{g}_{m}^{(t)}\bigg|\bigg|^2\right] \\
&\leq \frac{ABd}{M} +  \mathbb{E}\left[\bigg|\bigg|\frac{1}{M}\sum_{m=1}^{M}\boldsymbol{g}_{m}^{(t)} - \nabla F(\boldsymbol{w}^{(t)})\bigg|\bigg|^2\right] + ||\nabla F(\boldsymbol{w}^{(t)})||^2 \\
&\leq \frac{ABd}{M} + \frac{||\bar{\boldsymbol{\sigma}}||_{2}^{2}}{M} + ||\nabla F(\boldsymbol{w}^{(t)})||^2.
\end{split}
\end{equation}

Therefore, taking expectations on both sides of (\ref{suppprot1}) yields

\begin{equation}
\begin{split}
\mathbb{E}[F(\boldsymbol{w}^{(t+1)}) - F(\boldsymbol{w}^{(t)})] &\leq \mathbb{E}\left[-\eta\left\langle\nabla F(\boldsymbol{w}^{(t)}), \frac{1}{M}\sum_{m=1}^{M}ternary(\boldsymbol{g}_{m}^{(t)},A,B)\right\rangle\right] + \frac{L\eta^2}{2B^{2}}\left[\frac{ABd}{M} + \frac{||\bar{\boldsymbol{\sigma}}||_{2}^{2}}{M} + ||\nabla F(\boldsymbol{w}^{(t)})||^2\right]\\
&=- \frac{\eta||\nabla F(\boldsymbol{w}^{(t)})||_{2}^{2}}{B} + \frac{L\eta^2}{2B^{2}}\left[\frac{ABd}{M} + \frac{||\bar{\boldsymbol{\sigma}}||_{2}^{2}}{M} + ||\nabla F(\boldsymbol{w}^{(t)})||^2\right]\\
&=-\left(\frac{\eta}{B} - \frac{L\eta^2}{2B^{2}}\right)||\nabla F(\boldsymbol{w}^{(t)})||_{2}^{2} + \frac{L\eta^2}{2B^{2}}\left[\frac{ABd}{M} + \frac{||\bar{\boldsymbol{\sigma}}||_{2}^{2}}{M}\right] \\
\end{split}
\end{equation}

Adjusting the above inequality and averaging both sides over $t=1,2,\cdots,T$, we can obtain
\begin{equation}
\color{black}
\begin{split}
&\frac{1}{T}\sum_{t=1}^{T}\left(\frac{\eta}{B} - \frac{L\eta^2}{2B^{2}}\right)|||\nabla F(\boldsymbol{w}^{(t)})||_{2}^{2} \leq \frac{\mathbb{E}[F(\boldsymbol{w}^{(0)}) - F(\boldsymbol{w}^{(t+1)})]}{T}  + \frac{L\eta^2}{2B^{2}}\left[\frac{ABd}{M} + \frac{||\bar{\boldsymbol{\sigma}}||_{2}^{2}}{M}\right].
\end{split}
\end{equation}

Dividing both sides by $\left(\frac{\eta}{B} - \frac{L\eta^2}{2B^{2}}\right)$ gives
\begin{equation}
\begin{split}
\frac{1}{T}\sum_{t=1}^{T}||\nabla F(\boldsymbol{w}^{(t)})||_{2}^{2} &\leq
\frac{\mathbb{E}[F(\boldsymbol{w}^{(0)}) - F(\boldsymbol{w}^{(t+1)})]}{T\left(\frac{\eta}{B} - \frac{L\eta^2}{2B^{2}}\right)} + \frac{L\eta^2}{2B^{2}\left(\frac{\eta}{B} - \frac{L\eta^2}{2B^{2}}\right)}\left[\frac{ABd}{M} + \frac{||\bar{\boldsymbol{\sigma}}||_{2}^{2}}{M}\right] \\
&\leq \frac{F(\boldsymbol{w}^{(0)}) - F^{*}}{T\left(\frac{\eta}{B} - \frac{L\eta^2}{2B^{2}}\right)} + \frac{L\eta^2}{2B^{2}\left(\frac{\eta}{B} - \frac{L\eta^2}{2B^{2}}\right)}\left[\frac{ABd}{M} + \frac{||\bar{\boldsymbol{\sigma}}||_{2}^{2}}{M}\right].
\end{split}
\end{equation}
which completes the proof.
\end{proof}

\subsection{Proof of Theorem \ref{SPconvergerateSchemeII}}
\begin{Theorem}[\textbf{Convergence of \textit{TernaryVote}}]\label{SPconvergerateSchemeII}
Suppose Assumptions \ref{A1}-\ref{A4} are satisfied, and the learning rate is set as $\eta=\frac{1}{\sqrt{TLd}}$. Then by running Algorithm \ref{DPSGDAlgorithm} with \textit{TernaryVote} for $T$ iterations, we have
\begin{equation}\label{AppConvergenceEquation}
\begin{split}
\frac{1}{T}\sum_{t=1}^{T}||\nabla F(\boldsymbol{w}^{(t)})||_{1} &\leq \frac{(F(\boldsymbol{w}^{(0)}) - F^{*})\sqrt{Ld}}{\sqrt{T}} + \frac{\sqrt{Ld}}{2\sqrt{T}}+ \frac{4||\bar{\boldsymbol{\sigma}}||_{1}}{\sqrt{M}}+ \frac{2Bd}{\sqrt{M+1}}\bigg(1-\frac{1}{M+1}\bigg)^{\frac{M}{2}} \\
&\leq \mathcal{O}(1/\sqrt{T}) + \mathcal{O}(B/\sqrt{M}).
\end{split}
\end{equation}
\end{Theorem}

Before proving Theorem \ref{SPconvergerateSchemeII}, we first present the following lemma from \cite{jin2024sign} and extend it to the ternary stochastic compressor in Lemma \ref{LemmaProbofWrongGenericTernary}.

\begin{Lemma}[\textbf{Probability of Wrong Aggregation for Generic Sign-based Compressor \cite{jin2024sign}}]\label{LemmaProbofWrongGenericSign}
Let $u_{1},u_{2},\cdots,u_{M}$ be $M$ known and fixed real numbers and consider binary random variables $\hat{u}_{m}$, $1\leq m \leq M$. Suppose $\Bar{p} = \frac{1}{M}\sum_{m=1}^{M}P\left(sign\left(\frac{1}{M}\sum_{m=1}^{M}u_{m}\right) \neq \hat{u}_{m}\right) < \frac{1}{2}$, then
\begin{equation}\label{ProbabilityOfError}
\begin{split}
P\bigg(sign\bigg(\frac{1}{M}\sum_{m=1}^{M}\hat{u}_{m}\bigg)&\neq sign\bigg(\frac{1}{M}\sum_{m=1}^{M}u_{m}\bigg)\bigg) \leq \big[4\Bar{p}(1-\Bar{p})\big]^{\frac{M}{2}}.
\end{split}
\end{equation}
\end{Lemma}

\begin{Lemma}[\textbf{Probability of Wrong Aggregation for $ternary$}]\label{LemmaProbofWrongGenericTernary}
Let $u_{1},u_{2},\cdots,u_{M}$ be $M$ known and fixed real numbers and consider binary random variables $\hat{u}_{m}$, $1\leq m \leq M$, which is given by
\begin{equation}
\hat{u}_{m} = ternary(u_{m},A,B) =
\begin{cases}
\hfill 1, \hfill \text{with probability $\frac{A+u_{m}}{2B}$},\\
\hfill 0, \hfill \text{with probability $1-\frac{A}{B}$},\\
\hfill -1, \hfill \text{with probability $\frac{A-u_{m}}{2B}$},\\
\end{cases}
\end{equation}
Suppose $B \geq 2A$, then
\begin{equation}
\begin{split}
P\bigg(sign\bigg(\frac{1}{M}\sum_{m=1}^{M}\hat{u}_{m}\bigg)&\neq sign\bigg(\frac{1}{M}\sum_{m=1}^{M}u_{m}\bigg)\bigg) \leq \bigg(1-\frac{|\sum_{m=1}^{M}u_{m}|^{2}}{B^{2}}\bigg)^{\frac{M}{2}}.
\end{split}
\end{equation}
\end{Lemma}

\begin{proof}
For each $u_{m}$, we construct the following two random variables
\begin{equation}
\hat{u}_{m,1} =
\begin{cases}
\hfill 1, \hfill \text{with probability $\frac{1}{2}+\frac{u_{m}}{2B}+\sqrt{\frac{1}{4}+\frac{|u_{m}|^{2}}{4B^2}-\frac{A}{2B}}$},\\
\hfill -1, \hfill \text{with probability $\frac{1}{2}-\frac{u_{m}}{2B}-\sqrt{\frac{1}{4}+\frac{|u_{m}|^{2}}{4B^2}-\frac{A}{2B}}$},\\
\end{cases}
\end{equation}

\begin{equation}
\hat{u}_{m,2} =
\begin{cases}
\hfill 1, \hfill \text{with probability $\frac{1}{2}+\frac{u_{m}}{2B}-\sqrt{\frac{1}{4}+\frac{|u_{m}|^{2}}{4B^2}-\frac{A}{2B}}$},\\
\hfill -1, \hfill \text{with probability $\frac{1}{2}-\frac{u_{m}}{2B}+\sqrt{\frac{1}{4}+\frac{|u_{m}|^{2}}{4B^2}-\frac{A}{2B}}$},\\
\end{cases}
\end{equation}

It can be observed that $\frac{\hat{u}_{m,1} + \hat{u}_{m,2}}{2}$ follows the same distribution as $\hat{u}_{m}$, which means that
\begin{equation}
\begin{split}
P\bigg(sign\bigg(\frac{1}{M}\sum_{m=1}^{M}\hat{u}_{m}\bigg) &\neq sign\bigg(\frac{1}{M}\sum_{m=1}^{M}u_{m}\bigg)\bigg) = P\bigg(sign\bigg(\frac{1}{2M}\sum_{m=1}^{M}[\hat{u}_{m,1} + \hat{u}_{m,2}]\bigg) \neq sign\bigg(\frac{1}{M}\sum_{m=1}^{M}u_{m}\bigg)\bigg).
\end{split}
\end{equation}

In this case, $\Bar{p} = \frac{1}{2M}\sum_{m=1}^{M}\bigg[P\left(sign\left(\frac{1}{M}\sum_{m=1}^{M}u_{m}\right) \neq \hat{u}_{m,1}\right)+P\left(sign\left(\frac{1}{M}\sum_{m=1}^{M}u_{m}\right) \neq \hat{u}_{m,2}\right)\bigg] = \frac{1}{2} - \frac{|\frac{1}{M}\sum_{m=1}^{M}u_{m}|}{2B}$. Invoking Lemma \ref{LemmaProbofWrongGenericSign} completes the proof.
\end{proof}

Given Lemma \ref{LemmaProbofWrongGenericTernary} at hand, we are ready to prove Theorem \ref{SPconvergerateSchemeII}.

\begin{proof}
According to Assumption \ref{A2}, we have
\begin{equation}
\begin{split}
&F(\boldsymbol{w}^{(t+1)}) - F(\boldsymbol{w}^{(t)}) \\
&\leq \langle\nabla F(\boldsymbol{w}^{(t)}), \boldsymbol{w}^{(t+1)}-\boldsymbol{w}^{(t)}\rangle + \frac{L}{2}||\boldsymbol{w}^{(t+1)}_{i}-\boldsymbol{w}^{(t)}_{i}||_{2}^2 \\
& =-\eta \left\langle\nabla F(\boldsymbol{w}^{(t)}), sign\bigg(\frac{1}{M}\sum_{m=1}^{M}ternary(\boldsymbol{g}_{m}^{(t)},A,B)\bigg)\right\rangle +\frac{L}{2}\bigg|\bigg|\eta sign\bigg(\frac{1}{M}\sum_{m=1}^{M}ternary(\boldsymbol{g}_{m}^{(t)},A,B)\bigg)\bigg|\bigg|^2 \\
& \leq -\eta \left\langle\nabla F(\boldsymbol{w}^{(t)}), sign\bigg(\frac{1}{M}\sum_{m=1}^{M}ternary(\boldsymbol{g}_{m}^{(t)},A,B)\bigg)\right\rangle + \frac{Ld\eta^2}{2} \\
& = -\eta ||\nabla F(\boldsymbol{w}^{(t)})||_{1} + \frac{Ld\eta^2}{2} + 2\eta\sum_{i=1}^{d}|\nabla F(\boldsymbol{w}^{(t)})_{i}|\times\mathds{1}_{sign(\frac{1}{M}\sum_{m=1}^{M}ternary(\boldsymbol{g}_{m,i}^{(t)},A,B))\neq sign(\nabla F(\boldsymbol{w}^{(t)})_{i})},
\end{split}
\end{equation}
where $\nabla F(\boldsymbol{w}^{(t)})_{i}$ is the $i$-th entry of the vector $\nabla F(\boldsymbol{w}^{(t)})$ and $\eta$ is the learning rate. Taking expectations on both sides yields

\begin{equation}\label{convergencee1}
\begin{split}
&\mathbb{E}[F(\boldsymbol{w}^{(t+1)}) - F(\boldsymbol{w}^{(t)})] \\
&\leq -\eta ||\nabla F(\boldsymbol{w}^{(t)})||_{1} + \frac{Ld\eta^2}{2} +2\eta\sum_{i=1}^{d}\mathbb{E}\bigg[|\nabla F(\boldsymbol{w}^{(t)})_{i}|P\bigg(sign\bigg(\frac{1}{M}\sum_{m=1}^{M}ternary(\boldsymbol{g}_{m,i}^{(t)},A,B)\bigg)\neq sign(\nabla F(\boldsymbol{w}^{(t)})_{i})\bigg)\bigg]\\
&\leq -\eta ||\nabla F(\boldsymbol{w}^{(t)})||_{1} + \frac{Ld\eta^2}{2} +2\eta\sum_{i=1}^{d}\mathbb{E}\bigg[|\nabla F(\boldsymbol{w}^{(t)})_{i}|\bigg[P\bigg(sign\bigg(\frac{1}{M}\sum_{m=1}^{M}ternary(\boldsymbol{g}_{m,i}^{(t)},A,B)\bigg)\neq sign\bigg(\frac{1}{M}\sum_{m=1}^{M}\boldsymbol{g}^{(t)}_{m,i}\bigg)\bigg)\\
&+P\bigg(sign\bigg(\frac{1}{M}\sum_{m=1}^{M}\boldsymbol{g}_{m,i}^{(t)}\bigg)\neq sign(\nabla F(\boldsymbol{w}^{(t)})_{i})\bigg)\bigg]\bigg]\\
&\leq -\eta ||\nabla F(\boldsymbol{w}^{(t)})||_{1} + \frac{Ld\eta^2}{2} +2\eta\sum_{i=1}^{d}\mathbb{E}\bigg[|\nabla F(\boldsymbol{w}^{(t)})_{i}|\bigg(1-\frac{|\frac{1}{M}\sum_{m=1}^{M}\boldsymbol{g}_{m,i}^{(t)}|^2}{B^2}\bigg)^{\frac{M}{2}}\bigg]\\
&+2\eta\sum_{i=1}^{d}\mathbb{E}\bigg[|\nabla F(\boldsymbol{w}^{(t)})_{i}|P\bigg(sign\bigg(\frac{1}{M}\sum_{m=1}^{M}\boldsymbol{g}_{m,i}^{(t)}\bigg)\neq sign(\nabla F(\boldsymbol{w}^{(t)})_{i})\bigg)\bigg]\\
&\leq -\eta ||\nabla F(\boldsymbol{w}^{(t)})||_{1} + \frac{Ld\eta^2}{2} + 2\eta\sum_{i=1}^{d}\mathbb{E}\bigg[\bigg|\nabla F(\boldsymbol{w}^{(t)})_{i}-\frac{1}{M}\sum_{m=1}^{M}\boldsymbol{g}_{m,i}^{(t)}\bigg|\bigg(1-\frac{|\frac{1}{M}\sum_{m=1}^{M}\boldsymbol{g}_{m,i}^{(t)}|^2}{B^2}\bigg)^{\frac{M}{2}}\bigg] \\
&+2\eta\sum_{i=1}^{d}\mathbb{E}\bigg[\bigg|\frac{1}{M}\sum_{m=1}^{M}\boldsymbol{g}_{m,i}^{(t)}\bigg|\bigg(1-\frac{|\frac{1}{M}\sum_{m=1}^{M}\boldsymbol{g}_{m,i}^{(t)}|^2}{B^2}\bigg)^{\frac{M}{2}}\bigg]\\
&+2\eta\sum_{i=1}^{d}\mathbb{E}\bigg[|\nabla F(\boldsymbol{w}^{(t)})_{i}|P\bigg(sign\bigg(\frac{1}{M}\sum_{m=1}^{M}\boldsymbol{g}_{m,i}^{(t)}\bigg)\neq sign(\nabla F(\boldsymbol{w}^{(t)})_{i})\bigg)\bigg].
\end{split}
\end{equation}

In addition,
\begin{equation}\label{connect_to_variance}
\begin{split}
\sum_{i=1}^{d}\mathbb{E}\bigg[\bigg|\nabla F(\boldsymbol{w}^{(t)})_{i}-\frac{1}{M}\sum_{m=1}^{M}\boldsymbol{g}_{m,i}^{(t)}\bigg|\bigg] &= \sum_{i=1}^{d}\sqrt{\bigg[\mathbb{E}\bigg[\bigg|\nabla F(\boldsymbol{w}^{(t)})_{i}-\frac{1}{M}\sum_{m=1}^{M}\boldsymbol{g}_{m,i}^{(t)}\bigg|\bigg]\bigg]^2} \\
&\leq \sum_{i=1}^{d}\sqrt{\mathbb{E}\bigg[\bigg|\nabla F(\boldsymbol{w}^{(t)})_{i}-\frac{1}{M}\sum_{m=1}^{M}\boldsymbol{g}_{m,i}^{(t)}\bigg|^2\bigg]}\\
&=\sum_{i=1}^{d}\sqrt{\mathbb{E}\bigg[\bigg|\frac{1}{M}\sum_{m=1}^{M}\nabla f_{m}(\boldsymbol{w}^{(t)})_{i}-\frac{1}{M}\sum_{m=1}^{M}\boldsymbol{g}_{m,i}^{(t)}\bigg|^2\bigg]}\\
&=\sum_{i=1}^{d}\sqrt{\frac{1}{M^2}\sum_{m=1}^{M}\mathbb{E}[|\nabla f_{m}(\boldsymbol{w}^{(t)})_{i}-\boldsymbol{g}_{m,i}^{(t)}|^2]} \leq \sum_{i=1}^{d}\sqrt{\frac{\sigma_{i}^{2}}{M}} = \frac{||\bar{\boldsymbol{\sigma}}||_1}{\sqrt{M}}.
\end{split}
\end{equation}
\begin{equation}
\begin{split}
&\sum_{i=1}^{d}|\nabla F(\boldsymbol{w}^{(t)})_{i}|P\bigg(sign\bigg(\frac{1}{M}\sum_{m=1}^{M}\nabla f_{m}(\boldsymbol{w}^{(t)})_{i}\bigg)\neq sign\bigg(\frac{1}{M}\sum_{m=1}^{M}\boldsymbol{g}_{m,i}^{(t)}\bigg)\bigg) \\
&\leq \sum_{i=1}^{d}|\nabla F(\boldsymbol{w}^{(t)})_{i}|P\bigg(\bigg|\frac{1}{M}\sum_{m=1}^{M}\nabla f_{m}(\boldsymbol{w}^{(t)})_{i}-\frac{1}{M}\sum_{m=1}^{M}\boldsymbol{g}_{m,i}^{(t)}\bigg|\geq \bigg|\frac{1}{M}\sum_{m=1}^{M}\nabla f_{m}(\boldsymbol{w}^{(t)})_{i}\bigg|\bigg)\\
&\leq \sum_{i=1}^{d}|\nabla F(\boldsymbol{w}^{(t)})_{i}|\frac{\mathbb{E}[|\frac{1}{M}\sum_{m=1}^{M}\nabla f_{m}(\boldsymbol{w}^{(t)})_{i}-\frac{1}{M}\sum_{m=1}^{M}\boldsymbol{g}_{m,i}^{(t)}|]}{|\frac{1}{M}\sum_{m=1}^{M}\nabla f_{m}(\boldsymbol{w}^{(t)})_{i}|} \\
&\leq \sum_{i=1}^{d}|\nabla F(\boldsymbol{w}^{(t)})_{i}|\frac{\sqrt{\mathbb{E}[(\frac{1}{M}\sum_{m=1}^{M}\nabla f_{m}(\boldsymbol{w}^{(t)})_{i}-\frac{1}{M}\sum_{m=1}^{M}\boldsymbol{g}_{m,i}^{(t)})^2]}}{|\frac{1}{M}\sum_{m=1}^{M}\nabla f_{m}(\boldsymbol{w}^{(t)})_{i}|} \leq \frac{||\bar{\boldsymbol{\sigma}}||_{1}}{\sqrt{M}}.
\end{split}
\end{equation}

Moreover, for a function $h(x) = x(1-\frac{x^{2}}{B^{2}})^{\frac{M}{2}}$, it can be derived that $h'(x) = (1-\frac{x^{2}}{B^{2}})^{\frac{M}{2}-1}[1-\frac{(M+1)x^{2}}{B^{2}}]$. Since $B \geq 2A > x$, we can conclude that $h(x)$ attains the maximum when $x = \frac{B}{\sqrt{M+1}}$. As a result, we have
\begin{equation}\label{upperterm}
\mathbb{E}\bigg[\bigg|\frac{1}{M}\sum_{m=1}^{M}\boldsymbol{g}_{m,i}^{(t)}\bigg|\bigg(1-\frac{|\frac{1}{M}\sum_{m=1}^{M}\boldsymbol{g}_{m,i}^{(t)}|^2}{B^2}\bigg)^{\frac{M}{2}}\bigg] \leq \frac{B}{\sqrt{M+1}}\bigg(1-\frac{1}{M+1}\bigg)^{\frac{M}{2}}.
\end{equation}

Plugging (\ref{connect_to_variance}) and (\ref{upperterm}) into (\ref{convergencee1}) yields

\begin{equation}
\begin{split}
&\mathbb{E}[F(\boldsymbol{w}^{(t+1)}) - F(\boldsymbol{w}^{(t)})] \leq -\eta ||\nabla F(\boldsymbol{w}^{(t)})||_{1} + \frac{Ld\eta^2}{2} + 4\eta\frac{||\bar{\boldsymbol{\sigma}}||_{1}}{\sqrt{M}}+ \frac{2\eta Bd}{\sqrt{M+1}}\bigg(1-\frac{1}{M+1}\bigg)^{\frac{M}{2}}.
\end{split}
\end{equation}

Adjusting the above inequality and averaging both sides over $t=1,2,\cdots,T$, we can obtain
\begin{equation}
\color{black}
\begin{split}
&\frac{1}{T}\sum_{t=1}^{T}\eta||\nabla F(\boldsymbol{w}^{(t)})||_{1} \leq \frac{\mathbb{E}[F(\boldsymbol{w}^{(0)}) - F(\boldsymbol{w}^{(t+1)})]}{T} + \frac{Ld\eta^2}{2}+ 4\eta\frac{||\bar{\boldsymbol{\sigma}}||_{1}}{\sqrt{M}}+ \frac{2\eta Bd}{\sqrt{M+1}}\bigg(1-\frac{1}{M+1}\bigg)^{\frac{M}{2}}.
\end{split}
\end{equation}

Letting $\eta=\frac{1}{\sqrt{LTd}}$ and dividing both sides by $\eta$ gives
\begin{equation}
\begin{split}
\frac{1}{T}\sum_{t=1}^{T}||\nabla F(\boldsymbol{w}^{(t)})||_{1} &\leq
\frac{\mathbb{E}[F(\boldsymbol{w}^{(0)}) - F(\boldsymbol{w}^{(t+1)})]\sqrt{Ld}}{\sqrt{T}} + \frac{\sqrt{Ld}}{2\sqrt{T}}+ \frac{4||\bar{\boldsymbol{\sigma}}||_{1}}{\sqrt{M}}+ \frac{2Bd}{\sqrt{M+1}}\bigg(1-\frac{1}{M+1}\bigg)^{\frac{M}{2}}\\
&\leq \frac{(F(\boldsymbol{w}^{(0)}) - F^{*})\sqrt{Ld}}{\sqrt{T}} + \frac{\sqrt{Ld}}{2\sqrt{T}}+ \frac{4||\bar{\boldsymbol{\sigma}}||_{1}}{\sqrt{M}}+ \frac{2Bd}{\sqrt{M+1}}\bigg(1-\frac{1}{M+1}\bigg)^{\frac{M}{2}}.
\end{split}
\end{equation}
which completes the proof.
\end{proof}

\subsection{Proof of Theorem \ref{SPconvergerateSchemeII2}}
\begin{Theorem}[\textbf{Convergence of \textit{TernaryVote}}]\label{SPconvergerateSchemeII2}
Suppose Assumptions \ref{A1}-\ref{A4} are satisfied, $|\nabla F(\boldsymbol{w}^{(t)})_{i}| < Q, \forall i, t$, $B \geq 2A = \mathcal{O}(\sqrt{T})$, and $\lim_{T\rightarrow \infty}M/\sqrt{T} = 0$. Then by running Algorithm \ref{DPSGDAlgorithm} with \textit{TernaryVote} and the learning rate $\eta=\frac{1}{\sqrt{TLd}}$ for $T$ iterations, we have
\begin{equation}
\begin{split}
&\frac{1}{T}\sum_{t=1}^{T}||\nabla F(\boldsymbol{w}^{(t)})||_{2}^{2} \\
&\leq \frac{1}{\mathcal{I}(A,B,M)}\bigg[\frac{(F(\boldsymbol{w}^{(0)}) - F^{*})\sqrt{Ld}}{\sqrt{T}} + \frac{\sqrt{Ld}}{2\sqrt{T}} + \sum_{n=2}^{M}\bigg(1-\frac{A}{B}\bigg)^{M-n}\bigg[{M \choose n}\mathcal{O}\bigg(\frac{A^{n-2}}{B^{n}}\bigg)\bigg]Qd\bigg]\\
&\leq \mathcal{O}\bigg(\frac{B}{\sqrt{T}}\bigg) + \mathcal{O}\bigg(\frac{1}{B}\bigg),
\end{split}
\end{equation}
in which
\begin{equation}\nonumber
\mathcal{I}(A,B,M) = \sum_{n=1}^{M}(1-\frac{A}{B})^{M-n}\bigg[\frac{{n-1 \choose \lfloor\frac{n-1}{2}\rfloor}MA^{n-1}{M-1 \choose n-1}}{2^{n-1}B^{n}}\bigg].
\end{equation}
\end{Theorem}

\begin{Remark}
Since $\frac{n^{n+1}e^{-n}\sqrt{2\pi}}{\sqrt{n}} \leq n! < \frac{n^{n+1}e^{-n}\sqrt{2\pi}}{\sqrt{n-1}}$ \cite{batir2008sharp}, we can readily show that ${n-1 \choose \lfloor\frac{n-1}{2}\rfloor}/2^{n-1} \geq 2(n-1)/[\sqrt{2\pi}(n-1)^{3/2}] = \mathcal{O}(1/\sqrt{n})$. Therefore, utilizing the fact that $M/\sqrt{n} \geq \sqrt{M}$, we have $\mathcal{I}(A,B,M) \geq \sum_{n=1}^{M}\big(1-\frac{A}{B}\big)^{M-n}{M-1 \choose n-1}\frac{A^{n-1}}{B^{n-1}}\mathcal{O}(\frac{\sqrt{M}}{B})=\mathcal{O}(\frac{\sqrt{M}}{B})$, which measures the impact of $M$.
\end{Remark}

Before proving Theorem \ref{SPconvergerateSchemeII2}, we first prove the following lemma.

\begin{Lemma}\label{bsufficientlylarge}
Let $u_{1},u_{2},\cdots,u_{M}$ be $M$ known and fixed real numbers with $|u_{m}| < c, \forall m$, and consider $\hat{u}_{m} = ternary(u_{m},A,B)$, $1\leq m \leq M$. Suppose that $B \geq 2A = \mathcal{O}(\sqrt{T})$ and $\lim_{T\rightarrow \infty}\frac{M}{\sqrt{T}} = 0$, then
\begin{equation}
\begin{split}
&P\bigg(sign\bigg(\frac{1}{M}\sum_{m=1}^{M}\hat{u}_m\bigg)=1\bigg) - P\bigg(sign\bigg(\frac{1}{M}\sum_{m=1}^{M}\hat{u}_m\bigg)=-1\bigg) \\
& = \sum_{n=1}^{M}\bigg(1-\frac{A}{B}\bigg)^{M-n}\bigg[\frac{{n-1 \choose \lfloor\frac{n-1}{2}\rfloor}A^{n-1}}{2^{n-1}B^{n}}{M-1 \choose n-1}\sum_{m=1}^{M}u_{m}\bigg] + \sum_{n=2}^{M}\bigg(1-\frac{A}{B}\bigg)^{M-n}{M \choose n}\mathcal{O}\bigg(\frac{A^{n-2}}{B^{n}}\bigg).
\end{split}
\end{equation}
\end{Lemma}
\begin{proof}
According to the definition of the $ternary$ compressor, we have
\begin{equation}
\hat{u}_{m} = ternary(x,A,B) =
\begin{cases}
\hfill 1, \hfill \text{with probability $\frac{A+u_{m}}{2B}$},\\
\hfill 0, \hfill \text{with probability $1-\frac{A}{B}$},\\
\hfill -1, \hfill \text{with probability $\frac{A-u_{m}}{2B}$},\\
\end{cases}
\end{equation}

Suppose that $n \geq 2$ of the $\hat{u}_{m}$'s are non-zero and denote the set by $\mathcal{F}_{\neq 0}^{n}$. In this case, $sign\big(\frac{1}{M}\sum_{m=1}^{M}\hat{u}_m\big) = sign\big(\frac{1}{n}\sum_{i \in \mathcal{F}_{\neq 0}^{n}}\hat{u}_i\big)$.
Let $\hat{Z} = \sum_{i \in \mathcal{F}_{\neq 0}^{n}}\mathds{1}_{\hat{u}_{i}=1}$, then
\begin{equation}
\begin{split}
P\bigg(sign\bigg(\frac{1}{n}\sum_{i \in \mathcal{F}_{\neq 0}^{n}}\hat{u}_i\bigg)=1\bigg) &= P\bigg(\hat{Z} > \frac{n}{2}\bigg) + \frac{1}{2}P\bigg(\hat{Z} = \frac{n}{2}\bigg), \\
\end{split}
\end{equation}
in which we break the tie randomly. Particularly, there are two possible cases.

\textbf{Case 1: $n$ is odd}. In this case, $P\big(\hat{Z} = \frac{n}{2}\big) = 0$, and $P\big(\hat{Z} > \frac{n}{2}\big) = \sum_{H=\frac{n+1}{2}}^{n}P(\hat{Z} = H)$, where
\begin{equation}\label{Chapter4-SPhatZ}
\begin{split}
P(\hat{Z}=H) &= \frac{\sum_{\mathcal{A} \in \mathcal{F}_H}\prod_{k \in \mathcal{A}}(A+u_{k})\prod_{j \in \mathcal{F}_{\neq 0}^{n}\setminus\mathcal{A}}(A-u_{j})}{(2B)^{n}} = \frac{a_{n,H}A^{n} + a_{n-1,H}A^{n-1} + \cdots + a_{0,H}A^{0}}{(2B)^{n}},
\end{split}
\end{equation}
and
\begin{equation}\label{Chapter4-SPhatZ2}
\begin{split}
&\sum_{H = \frac{n+1}{2}}^{n}P(\hat{Z}=H) = \frac{\sum_{H = \frac{n+1}{2}}^{n}a_{n,H}A^{n}}{(2B)^{n}} + \frac{\sum_{H = \frac{n+1}{2}}^{n}a_{n-1,H}A^{n-1}}{(2B)^{n}} + \cdots + \frac{\sum_{H = \frac{n+1}{2}}^{n}a_{0,H}A^{0}}{(2B)^{n}},
\end{split}
\end{equation}
in which $F_H$ is the set of all subsets of $H$ integers that can be selected from $\mathcal{F}_{\neq 0}^{n}$; $a_{i,H}, \forall 0\leq i \leq n$ is some constant. It can be easily verified that $a_{n,H} = {n \choose H}$.

In particular, $\forall i$, we have
\begin{equation}
\begin{split}
&\sum_{\mathcal{A} \in F_H}\prod_{k \in \mathcal{A}}(A+u_{k})\prod_{j \in \mathcal{F}_{\neq 0}^{n}\setminus\mathcal{A}}(A-u_{j}) \\
&= (A+u_{i})\sum_{\mathcal{A} \in F_H, i\in \mathcal{A}}\prod_{k \in \mathcal{A}\setminus\{i\}}(A+u_{k})\prod_{j \in \mathcal{F}_{\neq 0}^{n}\setminus\mathcal{A}}(A-u_{j}) + (A-u_{i})\sum_{\mathcal{A} \in F_H, i\notin \mathcal{A}}\prod_{k \in \mathcal{A}}(A+u_{k})\prod_{j \in \mathcal{F}_{\neq 0}^{n}\setminus\mathcal{A},\{i\}}(A-u_{j}).
\end{split}
\end{equation}
As a result, when $1 \leq H \leq n-1$, the $u_{i}$ related term in $a_{n-1,H}$ is given by
\begin{equation}
\bigg[{n-1 \choose H-1} - {n-1 \choose H}\bigg]u_{i}.
\end{equation}
When $H = n$, the $u_{i}$ related term in $a_{n-1,H}$ is given by
\begin{equation}
\bigg[{n-1 \choose H-1}\bigg]u_{i}.
\end{equation}
When $H = 0$, the $u_{i}$ related term in $a_{n-1,H}$ is given by
\begin{equation}
\bigg[{n-1 \choose H}\bigg]u_{i}.
\end{equation}
By summing over $i$, we have
\begin{equation}
\begin{split}
a_{n-1,H} = \bigg[{n-1 \choose H-1} - {n-1 \choose H}\bigg]\sum_{i \in \mathcal{F}^{n}_{\neq 0}}u_{i}, ~~~~~\text{if}~~~ 1 \leq H \leq n-1,
\end{split}
\end{equation}
\begin{equation}
a_{n-1,H} = \bigg[{n-1 \choose H-1}\bigg]\sum_{i \in \mathcal{F}^{n}_{\neq 0}}u_{i}, ~~~~~\text{if}~~~H = n,
\end{equation}
and
\begin{equation}
a_{n-1,H} = -\bigg[{n-1 \choose H}\bigg]\sum_{i \in \mathcal{F}^{n}_{\neq 0}}u_{i}, ~~~~~\text{if}~~~H = 0.
\end{equation}

By summing over $H$, we have
\begin{equation}
\sum_{H = \frac{n+1}{2}}^{n}a_{n,H} = \sum_{H = \frac{n+1}{2}}^{n}{n \choose H} = 2^{n-1},
\end{equation}
\begin{equation}
\sum_{H = 0}^{\frac{n-1}{2}}a_{n,H} = \sum_{H = 0}^{\frac{n-1}{2}}{n \choose H} = 2^{n-1},
\end{equation}
\begin{equation}
\sum_{H = \frac{n+1}{2}}^{n}a_{n-1,H} = {n-1 \choose \frac{n-1}{2}}\sum_{i \in \mathcal{F}^{n}_{\neq 0}}u_{i},
\end{equation}
and
\begin{equation}
\sum_{H = 0}^{\frac{n-1}{2}}a_{n-1,H} = -{n-1 \choose \frac{n-1}{2}}\sum_{i \in \mathcal{F}^{n}_{\neq 0}}u_{i}.
\end{equation}
Following the same procedure, it can be shown that
\begin{equation}
\begin{split}
a_{n-v,H} = \left[\sum_{i=0}^{v}{v \choose v-i}{n-v \choose H-v+i}(-1)^{i}\right]\sum_{F_{v}}\prod_{j \in F_{v}}u_{j},
\end{split}
\end{equation}
in which $F_v$ is the set of all subsets of $v$ integers that can be selected from $\mathcal{F}_{\neq 0}^{n}$ and ${M-v \choose H-v+i} = 0$ if $H-v+i > M-v$. For $v > 0$, summing over $H$ yields
\begin{equation}
\begin{split}
&\sum_{H = \frac{n+1}{2}}^{n}a_{n-v,H} = \left[\sum_{i=0}^{v-1}\left[\sum_{j=0}^{i}{v \choose j}(-1)^{j}\right]{n-v \choose \frac{n+1}{2}-v+i} \right]\sum_{F_{v}}\prod_{j \in F_{v}}u_{j}.
\end{split}
\end{equation}

In particular,
\begin{equation}
\begin{split}
\sum_{H = \frac{n+1}{2}}^{n}a_{n-2,H} &=\left[{n-2 \choose \frac{n+1}{2}-2}-{n-2 \choose \frac{n+1}{2}-1}\right]\sum_{F_{2}}\prod_{j \in F_{2}}u_{j}=0,
\end{split}
\end{equation}
and
\begin{equation}
\begin{split}
\left|\sum_{H = \frac{n+1}{2}}^{n}a_{n-v,H}\right| &= \left|\left[\sum_{i=0}^{v-1}\left[\sum_{j=0}^{i}{v \choose j}(-1)^{j}\right]{n-v \choose \frac{n+1}{2}-v+i} \right]\sum_{F_{v}}\prod_{k \in F_{v}}u_{j}\right| \\
&\leq \left|\left[\sum_{i=0}^{v-1}\left[\sum_{j=0}^{i}{v \choose j}(-1)^{j}\right]{n-v \choose \frac{n+1}{2}-v+i} \right]{n \choose v}c^{v}\right|.
\end{split}
\end{equation}

Since $\lim_{T\rightarrow \infty}\frac{n}{\sqrt{T}} = 0$, $\sum_{H = \frac{n+1}{2}}^{n}P(\hat{Z}=H)$ is dominated by the first two terms in (\ref{Chapter4-SPhatZ2}) when $T$ is large enough. As a result,
\begin{equation}
\begin{split}
P\bigg(sign\bigg(\frac{1}{n}\sum_{i \in \mathcal{F}_{\neq 0}^{n}}\hat{u}_i\bigg)=1\bigg) &= P\bigg(\hat{Z} > \frac{n}{2}\bigg)  = \sum_{H = \frac{n+1}{2}}^{n}P(\hat{Z} = H) \\
&= \frac{2^{n-1}A^{n} + {n-1 \choose \frac{n-1}{2}}\sum_{i \in \mathcal{F}^{n}_{\neq 0}}u_{i}A^{n-1}}{(2B)^{n}} + \mathcal{O}\bigg(\frac{A^{n-2}}{B^{n}}\bigg)\\
& = \frac{A^{n}}{2B^{n}} + \frac{{n-1 \choose \frac{n-1}{2}}A^{n-1}}{2^{n}B^{n}}\sum_{i \in \mathcal{F}^{n}_{\neq 0}}u_{i} + \mathcal{O}\bigg(\frac{A^{n-2}}{B^{n}}\bigg).
\end{split}
\end{equation}

Similarly,
\begin{equation}
\begin{split}
P\bigg(sign\bigg(\frac{1}{n}\sum_{i \in \mathcal{F}_{\neq 0}^{n}}\hat{u}_i\bigg)= -1\bigg) &= P\bigg(\hat{Z} < \frac{n}{2}\bigg)  = \sum_{H = 0}^{\frac{n-1}{2}}P(\hat{Z} = H) \\
&= \frac{2^{n-1}A^{n} - {n-1 \choose \frac{n-1}{2}}\sum_{i \in \mathcal{F}^{n}_{\neq 0}}u_{i}A^{n-1}}{(2B)^{n}} + \mathcal{O}\bigg(\frac{A^{n-2}}{B^{n}}\bigg)\\
& = \frac{A^{n}}{2B^{n}} - \frac{{n-1 \choose \frac{n-1}{2}}A^{n-1}}{2^{n}B^{n}}\sum_{i \in \mathcal{F}^{n}_{\neq 0}}u_{i} + \mathcal{O}\bigg(\frac{A^{n-2}}{B^{n}}\bigg).
\end{split}
\end{equation}

Therefore,
\begin{equation}
\begin{split}
P\bigg(sign\bigg(\frac{1}{n}\sum_{i \in \mathcal{F}_{\neq 0}^{n}}\hat{u}_i\bigg)=1\bigg) - P\bigg(sign\bigg(\frac{1}{n}\sum_{i \in \mathcal{F}_{\neq 0}^{n}}\hat{u}_i\bigg)= -1\bigg) = \frac{{n-1 \choose \frac{n-1}{2}}A^{n-1}}{2^{n-1}B^{n}}\sum_{i \in \mathcal{F}^{n}_{\neq 0}}u_{i} + \mathcal{O}\bigg(\frac{A^{n-2}}{B^{n}}\bigg).
\end{split}
\end{equation}

\textbf{Case 2: $n$ is even.} In this case, $P\big(sign\big(\frac{1}{n}\sum_{i \in \mathcal{F}_{\neq 0}^{n}}\hat{u}_i\big)=1\big) - P\big(sign\big(\frac{1}{n}\sum_{i \in \mathcal{F}_{\neq 0}^{n}}\hat{u}_i\big)= -1\big) = P(\hat{Z} > \frac{n}{2}) - P(\hat{Z} < \frac{n}{2})$. Similarly,
\begin{equation}\label{Chapter4-SPhatZ2_2}
\begin{split}
P(\hat{Z}=H) &= \frac{\sum_{\mathcal{A} \in \mathcal{F}_H}\prod_{k \in \mathcal{A}}(A+u_{k})\prod_{j \in \mathcal{F}_{\neq 0}^{n}\setminus\mathcal{A}}(A-u_{j})}{(2B)^{n}} = \frac{a_{n,H}A^{n} + a_{n-1,H}A^{n-1} + \cdots + a_{0,H}A^{0}}{(2B)^{n}},
\end{split}
\end{equation}
and
\begin{equation}
\sum_{H = 0}^{\frac{n}{2}-1}a_{n,H} = \sum_{H = 0}^{\frac{n}{2}-1}{n \choose H} = \sum_{H = \frac{n}{2}+1}^{n}{n \choose H} = \sum_{H = \frac{n}{2}+1}^{n}a_{n,H},
\end{equation}
\begin{equation}
\sum_{H = \frac{n}{2}+1}^{n}a_{n-1,H} = {n-1 \choose \frac{n}{2}}\sum_{i \in \mathcal{F}^{n}_{\neq 0}}u_{i},
\end{equation}
\begin{equation}
\sum_{H = 0}^{\frac{n}{2}-1}a_{n-1,H} = -{n-1 \choose \frac{n}{2} - 1}\sum_{i \in \mathcal{F}^{n}_{\neq 0}}u_{i} = -{n-1 \choose \frac{n}{2}}\sum_{i \in \mathcal{F}^{n}_{\neq 0}}u_{i}.
\end{equation}
Following the same procedure as that when $n$ is odd, it can be shown that
\begin{equation}
\begin{split}
P\bigg(sign\bigg(\frac{1}{n}\sum_{i \in \mathcal{F}_{\neq 0}^{n}}\hat{u}_i\bigg)=1\bigg) - P\bigg(sign\bigg(\frac{1}{n}\sum_{i \in \mathcal{F}_{\neq 0}^{n}}\hat{u}_i\bigg)= -1\bigg) = \frac{{n-1 \choose \frac{n}{2}-1}A^{n-1}}{2^{n-1}B^{n}}\sum_{i \in \mathcal{F}^{n}_{\neq 0}}u_{i} + \mathcal{O}\bigg(\frac{A^{n-2}}{B^{n}}\bigg).
\end{split}
\end{equation}

Overall, we have
\begin{equation}
\begin{split}
P\bigg(sign\bigg(\frac{1}{n}\sum_{i \in \mathcal{F}_{\neq 0}^{n}}\hat{u}_i\bigg)=1\bigg) - P\bigg(sign\bigg(\frac{1}{n}\sum_{i \in \mathcal{F}_{\neq 0}^{n}}\hat{u}_i\bigg)= -1\bigg) = \frac{{n-1 \choose \lfloor\frac{n-1}{2}\rfloor}A^{n-1}}{2^{n-1}B^{n}}\sum_{i \in \mathcal{F}^{n}_{\neq 0}}u_{i} + \mathcal{O}\bigg(\frac{A^{n-2}}{B^{n}}\bigg).
\end{split}
\end{equation}

Then we consider the scenario $n < 2$. It is obvious that $P\big(sign\big(\frac{1}{n}\sum_{i \in \mathcal{F}_{\neq 0}^{n}}\hat{u}_i\big)=1\big) - P\big(sign\big(\frac{1}{n}\sum_{i \in \mathcal{F}_{\neq 0}^{n}}\hat{u}_i\big)= -1\big) = 0$ when $n = 0$. In addition, when $n=1$, we have
\begin{equation}
\begin{split}
P\bigg(sign\bigg(\frac{1}{n}\sum_{i \in \mathcal{F}_{\neq 0}^{1}}\hat{u}_i\bigg)=1\bigg) - P\bigg(sign\bigg(\frac{1}{n}\sum_{i \in \mathcal{F}_{\neq 0}^{1}}\hat{u}_i\bigg)= -1\bigg) = \sum_{i \in \mathcal{F}_{\neq 0}^{1}}\frac{u_{i}}{B}.
\end{split}
\end{equation}

Therefore,
\begin{equation}
\begin{split}
&P\bigg(sign\bigg(\frac{1}{M}\sum_{m=1}^{M}\hat{u}_m\bigg)=1\bigg) - P\bigg(sign\bigg(\frac{1}{M}\sum_{m=1}^{M}\hat{u}_m\bigg)=-1\bigg) \\
&= \sum_{n=2}^{M}\bigg(1-\frac{A}{B}\bigg)^{M-n}\sum_{\mathcal{F}_{\neq 0}^{n}}\bigg[\frac{{n-1 \choose \lfloor\frac{n-1}{2}\rfloor}A^{n-1}}{2^{n-1}B^{n}}\sum_{i \in \mathcal{F}^{n}_{\neq 0}}u_{i} + \mathcal{O}\bigg(\frac{A^{n-2}}{B^{n}}\bigg)\bigg] + \bigg(1-\frac{A}{B}\bigg)^{M-1}\sum_{\mathcal{F}_{\neq 0}^{1}}\bigg[\frac{1}{B}\sum_{i \in \mathcal{F}^{1}_{\neq 0}}u_{i}\bigg]\\
& = \sum_{n=2}^{M}\bigg(1-\frac{A}{B}\bigg)^{M-n}\bigg[\frac{{n-1 \choose \lfloor\frac{n-1}{2}\rfloor}A^{n-1}}{2^{n-1}B^{n}}{M-1 \choose n-1}\sum_{m=1}^{M}u_{m} + {M \choose n}\mathcal{O}\bigg(\frac{A^{n-2}}{B^{n}}\bigg)\bigg] + \bigg(1-\frac{A}{B}\bigg)^{M-1}\bigg[\frac{1}{B}\sum_{m=1}^{M}u_{m}\bigg]\\
& = \sum_{n=1}^{M}\bigg(1-\frac{A}{B}\bigg)^{M-n}\bigg[\frac{{n-1 \choose \lfloor\frac{n-1}{2}\rfloor}A^{n-1}}{2^{n-1}B^{n}}{M-1 \choose n-1}\sum_{m=1}^{M}u_{m}\bigg] + \sum_{n=2}^{M}\bigg(1-\frac{A}{B}\bigg)^{M-n}{M \choose n}\mathcal{O}\bigg(\frac{A^{n-2}}{B^{n}}\bigg).
\end{split}
\end{equation}
which completes the proof.
\end{proof}

Now, we are ready to prove Theorem \ref{SPconvergerateSchemeII2}.

\begin{proof}
According to Assumption \ref{A2}, we have
\begin{equation}
\begin{split}
&F(\boldsymbol{w}^{(t+1)}) - F(\boldsymbol{w}^{(t)}) \\
&\leq \langle\nabla F(\boldsymbol{w}^{(t)}), \boldsymbol{w}^{(t+1)}-\boldsymbol{w}^{(t)}\rangle + \frac{L}{2}||\boldsymbol{w}^{(t+1)}_{i}-\boldsymbol{w}^{(t)}_{i}||_{2}^2 \\
& =-\eta \left\langle\nabla F(\boldsymbol{w}^{(t)}), sign\bigg(\frac{1}{M}\sum_{m=1}^{M}ternary(\boldsymbol{g}_{m}^{(t)},A,B)\bigg)\right\rangle +\frac{L}{2}\bigg|\bigg|\eta sign\bigg(\frac{1}{M}\sum_{m=1}^{M}ternary(\boldsymbol{g}_{m}^{(t)},A,B)\bigg)\bigg|\bigg|^2 \\
& \leq -\eta \left\langle\nabla F(\boldsymbol{w}^{(t)}), sign\bigg(\frac{1}{M}\sum_{m=1}^{M}ternary(\boldsymbol{g}_{m}^{(t)},A,B)\bigg)\right\rangle + \frac{Ld\eta^2}{2},
\end{split}
\end{equation}
where $\eta$ is the learning rate. Taking expectations on both sides yields

\begin{equation}
\begin{split}
&\mathbb{E}[F(\boldsymbol{w}^{(t+1)}) - F(\boldsymbol{w}^{(t)})] \\
&\leq \frac{Ld\eta^2}{2} -\eta \sum_{i=1}^{d}\nabla F(\boldsymbol{w}^{(t)})_{i}\times\\
&\mathbb{E}\bigg[P\bigg(sign\bigg(\frac{1}{M}\sum_{m=1}^{M}ternary(\boldsymbol{g}_{m,i}^{(t)},A,B)\bigg)=1\bigg) - P\bigg(sign\bigg(\frac{1}{M}\sum_{m=1}^{M}ternary(\boldsymbol{g}_{m,i}^{(t)},A,B)\bigg) = -1\bigg)\bigg] \\
&= \frac{Ld\eta^2}{2} -\eta \sum_{i=1}^{d}\nabla F(\boldsymbol{w}^{(t)})_{i}\sum_{n=1}^{M}\bigg(1-\frac{A}{B}\bigg)^{M-n}\mathbb{E}\bigg[\frac{{n-1 \choose \lfloor\frac{n-1}{2}\rfloor}A^{n-1}}{2^{n-1}B^{n}}{M-1 \choose n-1}\sum_{m=1}^{M}\boldsymbol{g}_{m,i}^{(t)}\bigg] \\
&-\eta \sum_{i=1}^{d}\nabla F(\boldsymbol{w}^{(t)})_{i}\sum_{n=2}^{M}\bigg(1-\frac{A}{B}\bigg)^{M-n}{M \choose n}\mathcal{O}\bigg(\frac{A^{n-2}}{B^{n}}\bigg) \\
&=\frac{Ld\eta^2}{2} - \eta\sum_{n=1}^{M}\bigg(1-\frac{A}{B}\bigg)^{M-n}\bigg[\frac{{n-1 \choose \lfloor\frac{n-1}{2}\rfloor}MA^{n-1}}{2^{n-1}B^{n}}{M-1 \choose n-1}\bigg]||\nabla F(\boldsymbol{w}^{(t)})||_{2}^{2}\\
&+\eta \sum_{i=1}^{d}|\nabla F(\boldsymbol{w}^{(t)})_{i}|\sum_{n=2}^{M}\bigg(1-\frac{A}{B}\bigg)^{M-n}\bigg[{M \choose n}\mathcal{O}\bigg(\frac{A^{n-2}}{B^{n}}\bigg)\bigg].
\end{split}
\end{equation}

Adjusting the above inequality and averaging both sides over $t=1,2,\cdots, T$ yields
\begin{equation}
\color{black}
\begin{split}
&\frac{1}{T}\sum_{t=1}^{T}\eta\sum_{n=1}^{M}\bigg(1-\frac{A}{B}\bigg)^{M-n}\bigg[\frac{{n-1 \choose \lfloor\frac{n-1}{2}\rfloor}MA^{n-1}}{2^{n-1}B^{n}}{M-1 \choose n-1}\bigg]||\nabla F(\boldsymbol{w}^{(t)})||_{2}^{2} \\
&\leq \frac{\mathbb{E}[F(\boldsymbol{w}^{(0)}) - F(\boldsymbol{w}^{(t+1)})]}{T} + \frac{Ld\eta^2}{2}+ \frac{\eta}{T}\sum_{t=1}^{T}\sum_{i=1}^{d}\nabla F(\boldsymbol{w}^{(t)})_{i}\sum_{n=2}^{M}\bigg(1-\frac{A}{B}\bigg)^{M-n}\bigg[{M \choose n}\mathcal{O}\bigg(\frac{A^{n-2}}{B^{n}}\bigg)\bigg] \\
&\leq \frac{\mathbb{E}[F(\boldsymbol{w}^{(0)}) - F(\boldsymbol{w}^{(t+1)})]}{T} + \frac{Ld\eta^2}{2}+ \eta\sum_{n=2}^{M}\bigg(1-\frac{A}{B}\bigg)^{M-n}\bigg[{M \choose n}\mathcal{O}\bigg(\frac{A^{n-2}}{B^{n}}\bigg)\bigg]Qd. \\
\end{split}
\end{equation}

Let $\eta=\frac{1}{\sqrt{LTd}}$ and $\mathcal{I}(A,B,M) = \sum_{n=1}^{M}(1-\frac{A}{B})^{M-n}\big[\frac{{n-1 \choose \lfloor\frac{n-1}{2}\rfloor}MA^{n-1}{M-1 \choose n-1}}{2^{n-1}B^{n}}\big]$. Dividing both sides by $\mathcal{I}(A,B,M)$ gives
\begin{equation}
\begin{split}
\frac{1}{T}\sum_{t=1}^{T}||\nabla F(\boldsymbol{w}^{(t)})||_{2}^{2} &\leq
\frac{\mathbb{E}[F(\boldsymbol{w}^{(0)}) - F(\boldsymbol{w}^{(t+1)})]\sqrt{Ld}}{\sqrt{T}\mathcal{I}(A,B,M)} + \frac{\sqrt{Ld}}{2\sqrt{T}\mathcal{I}(A,B,M)}+ \frac{\sum_{n=2}^{M}\big(1-\frac{A}{B}\big)^{M-n}\bigg[{M \choose n}\mathcal{O}\bigg(\frac{A^{n-2}}{B^{n}}\bigg)\bigg]Qd}{\mathcal{I}(A,B,M)}\\
&\leq \frac{(F(\boldsymbol{w}^{(0)}) - F^{*})\sqrt{Ld}}{\sqrt{T}\mathcal{I}(A,B,M)} + \frac{\sqrt{Ld}}{2\sqrt{T}\mathcal{I}(A,B,M)}+ \frac{\sum_{n=2}^{M}\big(1-\frac{A}{B}\big)^{M-n}\bigg[{M \choose n}\mathcal{O}\bigg(\frac{A^{n-2}}{B^{n}}\bigg)\bigg]Qd}{\mathcal{I}(A,B,M)}\\
&\leq \mathcal{O}\bigg(\frac{B}{\sqrt{T}}\bigg) + \mathcal{O}\bigg(\frac{1}{B}\bigg).
\end{split}
\end{equation}
which completes the proof.
\end{proof}

\subsection{Proof of Theorem \ref{SPByzantine1}}
\begin{Theorem}\label{SPByzantine1}
Suppose Assumptions \ref{A1}-\ref{A4} are satisfied, $|\nabla F(\boldsymbol{w}^{(t)})_{i}| < Q, \forall i, t$, and the learning rate is set as $\eta=\frac{1}{\sqrt{TLd}}$, then by running Algorithm \ref{DPSGDAlgorithm} with \textit{TernaryVote} and $\mathcal{N}_{t} = \mathcal{M} \cup \mathcal{K}$ for $T$ iterations, we have
\begin{equation}\label{APPByzantineConvergenceEquation}
\begin{split}
\frac{1}{T}\sum_{t=1}^{T}||\nabla F(\boldsymbol{w}^{(t)})||_{1} &\leq \frac{(F(\boldsymbol{w}^{(0)}) - F^{*})\sqrt{Ld}}{\sqrt{T}} + \frac{\sqrt{Ld}}{2\sqrt{T}}+ \frac{4\sqrt{M}||\bar{\boldsymbol{\sigma}}||_{1}}{M+K} \\
&+ \frac{4K(Q+A)d}{M+K} + \frac{2Bd}{\sqrt{M+K+1}}\big(1-\frac{1}{M+K+1}\big)^{\frac{M+K}{2}}\\
&\leq \mathcal{O}\left(\frac{1}{\sqrt{T}}\right) + \mathcal{O}\left(\frac{B}{\sqrt{M+K}}\right) + \mathcal{O}\left(\frac{\sqrt{M}+K}{M+K}\right).
\end{split}
\end{equation}\end{Theorem}

\begin{proof}
Let $\frac{1}{M+K}\left[\sum_{m\in\mathcal{M}}(\boldsymbol{g}_{m}^{(t)})+\sum_{k\in\mathcal{K}}(\boldsymbol{g}_{k}^{(t)})\right] = \bar{\boldsymbol{g}}_{\mathcal{M}+\mathcal{K}}^{(t)}$ and

$\frac{1}{M+K}\left[\sum_{m\in \mathcal{M}}ternary(\boldsymbol{g}_{m}^{(t)},A,B)+\sum_{k\in \mathcal{K}}byzantine(\boldsymbol{g}_{k}^{(t)},A,B)\right] = \hat{\boldsymbol{g}}_{\mathcal{M}+\mathcal{K}}^{(t)}$ .

According to Assumption \ref{A2}, we have
\begin{equation}
\begin{split}
F(\boldsymbol{w}^{(t+1)}) - F(\boldsymbol{w}^{(t)}) &\leq \langle\nabla F(\boldsymbol{w}^{(t)}), \boldsymbol{w}^{(t+1)}-\boldsymbol{w}^{(t)}\rangle + \frac{L}{2}||\boldsymbol{w}^{(t+1)}-\boldsymbol{w}^{(t)}||_{2}^2 \\
& =-\eta \left\langle\nabla F(\boldsymbol{w}^{(t)}), sign\left(\hat{\boldsymbol{g}}_{\mathcal{M}+\mathcal{K}}^{(t)}\right)\right\rangle +\frac{L}{2}\bigg|\bigg|\eta sign\left(\hat{\boldsymbol{g}}_{\mathcal{M}+\mathcal{K},i}^{(t)}\right)\bigg|\bigg|^2 \\
& \leq -\eta \left\langle\nabla F(\boldsymbol{w}^{(t)}), sign\left(\hat{\boldsymbol{g}}_{\mathcal{M}+\mathcal{K}}^{(t)}\right)\right\rangle + \frac{Ld\eta^2}{2} \\
& = -\eta ||\nabla F(\boldsymbol{w}^{(t)})||_{1} + \frac{Ld\eta^2}{2} + 2\eta\sum_{i=1}^{d}|\nabla F(\boldsymbol{w}^{(t)})_{i}|\times\mathds{1}_{sign(\hat{\boldsymbol{g}}_{\mathcal{M}+\mathcal{K},i}^{(t)})\neq sign(\nabla F(\boldsymbol{w}^{(t)})_{i})},
\end{split}
\end{equation}
where $\nabla F(\boldsymbol{w}^{(t)})_{i}$ is the $i$-th entry of $\nabla F(\boldsymbol{w}^{(t)})$ and $\eta$ is the learning rate. Taking expectations on both sides yields
\begin{equation}\label{Byzantineconvergencee1}
\begin{split}
&\mathbb{E}[F(\boldsymbol{w}^{(t+1)}) - F(\boldsymbol{w}^{(t)})] \\
&\leq -\eta ||\nabla F(\boldsymbol{w}^{(t)})||_{1} + \frac{Ld\eta^2}{2} +2\eta\sum_{i=1}^{d}\mathbb{E}\bigg[|\nabla F(\boldsymbol{w}^{(t)})_{i}|P\bigg(sign\left(\hat{\boldsymbol{g}}_{\mathcal{M}+\mathcal{K},i}^{(t)}\right)\neq sign(\nabla F(\boldsymbol{w}^{(t)})_{i})\bigg)\bigg]\\
&\leq -\eta ||\nabla F(\boldsymbol{w}^{(t)})||_{1} + \frac{Ld\eta^2}{2} +2\eta\sum_{i=1}^{d}\mathbb{E}\bigg[|\nabla F(\boldsymbol{w}^{(t)})_{i}|\bigg[P\bigg(sign\left(\hat{\boldsymbol{g}}_{\mathcal{M}+\mathcal{K},i}^{(t)}\right)\neq sign\left(\bar{\boldsymbol{g}}_{\mathcal{M}+\mathcal{K},i}^{(t)}\right)\bigg)\\
&+P\bigg(sign\left(\bar{\boldsymbol{g}}_{\mathcal{M}+\mathcal{K},i}^{(t)}\right)\neq sign(\nabla F(\boldsymbol{w}^{(t)})_{i})\bigg)\bigg]\bigg]\\
&\leq -\eta ||\nabla F(\boldsymbol{w}^{(t)})||_{1} + \frac{Ld\eta^2}{2} +2\eta\sum_{i=1}^{d}\mathbb{E}\bigg[|\nabla F(\boldsymbol{w}^{(t)})_{i}|\bigg(1-\frac{|\bar{\boldsymbol{g}}_{\mathcal{M}+\mathcal{K},i}^{(t)}|^2}{B^2}\bigg)^{\frac{M+K}{2}}\bigg]\\
&+2\eta\sum_{i=1}^{d}\mathbb{E}\bigg[|\nabla F(\boldsymbol{w}^{(t)})_{i}|P\bigg(sign\left(\bar{\boldsymbol{g}}_{\mathcal{M}+\mathcal{K},i}^{(t)}\right)\neq sign(\nabla F(\boldsymbol{w}^{(t)})_{i})\bigg)\bigg]\\
&\leq -\eta ||\nabla F(\boldsymbol{w}^{(t)})||_{1} + \frac{Ld\eta^2}{2} + 2\eta\sum_{i=1}^{d}\mathbb{E}\bigg[\bigg|\nabla F(\boldsymbol{w}^{(t)})_{i}-\bar{\boldsymbol{g}}_{\mathcal{M}+\mathcal{K},i}^{(t)}\bigg|\bigg(1-\frac{|\bar{\boldsymbol{g}}_{\mathcal{M}+\mathcal{K},i}^{(t)}|^2}{B^2}\bigg)^{\frac{M+K}{2}}\bigg] \\
&+2\eta\sum_{i=1}^{d}\mathbb{E}\bigg[\bigg|\bar{\boldsymbol{g}}_{\mathcal{M}+\mathcal{K},i}^{(t)}\bigg|\bigg(1-\frac{|\bar{\boldsymbol{g}}_{\mathcal{M}+\mathcal{K},i}^{(t)}|^2}{B^2}\bigg)^{\frac{M+K}{2}}\bigg]\\
&+2\eta\sum_{i=1}^{d}\mathbb{E}\bigg[|\nabla F(\boldsymbol{w}^{(t)})_{i}|P\bigg(sign\left(\bar{\boldsymbol{g}}_{\mathcal{M}+\mathcal{K},i}^{(t)}\right)\neq sign\left(\nabla F(\boldsymbol{w}^{(t)})\right)_{i}\bigg)\bigg].
\end{split}
\end{equation}

In addition,
\begin{equation}\label{Byzantineconnect_to_variance}
\begin{split}
&\sum_{i=1}^{d}\mathbb{E}\bigg[\bigg|\nabla F(\boldsymbol{w}^{(t)})_{i}-\bar{\boldsymbol{g}}_{\mathcal{M}+\mathcal{K},i}^{(t)}\bigg|\bigg] \\
&\leq \sum_{i=1}^{d}\mathbb{E}\bigg[\bigg|\frac{M}{M+K}\nabla F(\boldsymbol{w}^{(t)})_{i}-\frac{1}{M+K}\sum_{m\in\mathcal{M}}\boldsymbol{g}_{m,i}^{(t)}\bigg|\bigg] + \sum_{i=1}^{d}\mathbb{E}\bigg[\bigg|\frac{K}{M+K}\nabla F(\boldsymbol{w}^{(t)})_{i}-\frac{1}{M+K}\sum_{k\in\mathcal{K}}\boldsymbol{g}_{k,i}^{(t)}\bigg|\bigg]\\
&\leq \sum_{i=1}^{d}\sqrt{\bigg[\mathbb{E}\bigg[\bigg|\frac{M}{M+K}\nabla F(\boldsymbol{w}^{(t)})_{i}-\frac{1}{M+K}\sum_{m\in\mathcal{M}}\boldsymbol{g}_{m,i}^{(t)}\bigg|\bigg]\bigg]^2} + \frac{K(Q+A)d}{M+K}\\
&\leq \sum_{i=1}^{d}\sqrt{\mathbb{E}\bigg[\bigg|\frac{M}{M+K}\nabla F(\boldsymbol{w}^{(t)})_{i}-\frac{1}{M+K}\sum_{m\in\mathcal{M}}\boldsymbol{g}_{m,i}^{(t)}\bigg|^2\bigg]}+ \frac{K(Q+A)d}{M+K}\\
&=\sum_{i=1}^{d}\sqrt{\mathbb{E}\bigg[\bigg|\frac{1}{M+K}\sum_{m\in\mathcal{M}}\nabla f_{m}(\boldsymbol{w}^{(t)})_{i}-\frac{1}{M+K}\sum_{m\in\mathcal{M}}\boldsymbol{g}_{m,i}^{(t)}\bigg|^2\bigg]}+ \frac{K(Q+A)d}{M+K}\\
&=\sum_{i=1}^{d}\sqrt{\frac{1}{(M+K)^2}\sum_{m\in\mathcal{M}}\mathbb{E}[|\nabla f_{m}(\boldsymbol{w}^{(t)})_{i}-\boldsymbol{g}_{m,i}^{(t)}|^2]} + \frac{K(Q+A)d}{M+K}\\
&\leq \sum_{i=1}^{d}\frac{\sqrt{M\sigma_{i}^{2}}}{M+K}  + \frac{K(Q+A)d}{M+K} = \frac{\sqrt{M}||\sigma||_1}{M+K} + \frac{K(Q+A)d}{M+K}.
\end{split}
\end{equation}
\begin{equation}\label{SPpi2}
\begin{split}
&\sum_{i=1}^{d}|\nabla F(\boldsymbol{w}^{(t)})_{i}|P\bigg(sign\left(\nabla F(\boldsymbol{w}^{(t)})\right)_{i}\neq sign\left(\bar{\boldsymbol{g}}_{\mathcal{M}+\mathcal{K},i}^{(t)}\right)\bigg) \\
&\leq \sum_{i=1}^{d}|\nabla F(\boldsymbol{w}^{(t)})_{i}|P\bigg(\bigg|\nabla F(\boldsymbol{w}^{(t)})_{i}-\bar{\boldsymbol{g}}_{\mathcal{M}+\mathcal{K},i}^{(t)}\bigg|\geq \bigg|\nabla F(\boldsymbol{w}^{(t)})_{i}\bigg|\bigg)\\
&\leq \sum_{i=1}^{d}|\nabla F(\boldsymbol{w}^{(t)})_{i}|\frac{\mathbb{E}[|\nabla F(\boldsymbol{w}^{(t)})_{i}-\bar{\boldsymbol{g}}_{\mathcal{M}+\mathcal{K},i}^{(t)}|]}{|\nabla F(\boldsymbol{w}^{(t)})_{i}|} \\
&\leq \mathbb{E}[|\nabla F(\boldsymbol{w}^{(t)})_{i}-\bar{\boldsymbol{g}}_{\mathcal{M}+\mathcal{K},i}^{(t)}|] \\
&\leq \frac{\sqrt{M}||\boldsymbol{\bar{\sigma}}||_1}{M+K} + \frac{K(Q+A)d}{M+K}.
\end{split}
\end{equation}

Similar to (\ref{upperterm}), we can readily show that
\begin{equation}\label{Byzantineupperterm}
\mathbb{E}\bigg[\bigg|\bar{\boldsymbol{g}}_{\mathcal{M}+\mathcal{K},i}^{(t)}\bigg|\bigg(1-\frac{|\bar{\boldsymbol{g}}_{\mathcal{M}+\mathcal{K},i}^{(t)}|^2}{B^2}\bigg)^{\frac{M+K}{2}}\bigg] \leq \frac{B}{M+K+1}\bigg(1-\frac{1}{M+K+1}\bigg)^{\frac{M+K}{2}}.
\end{equation}

Plugging (\ref{Byzantineconnect_to_variance}), (\ref{SPpi2}), and (\ref{Byzantineupperterm}) into (\ref{Byzantineconvergencee1}) yields

\begin{equation}
\begin{split}
\mathbb{E}[F(\boldsymbol{w}^{(t+1)}) - F(\boldsymbol{w}^{(t)})] &\leq -\eta ||\nabla F(\boldsymbol{w}^{(t)})||_{1} + \frac{Ld\eta^2}{2} + 4\eta\left[\frac{\sqrt{M}||\bar{\boldsymbol{\sigma}}||_{1}}{M+K} + \frac{K(Q+A)d}{M+K}\right] \\
&+ \frac{2\eta Bd}{\sqrt{M+K+1}}\bigg(1-\frac{1}{M+K+1}\bigg)^{\frac{M+K}{2}}.
\end{split}
\end{equation}

Adjusting the above inequality and averaging both sides over $t=1,2,\cdots,T$, we can obtain
\begin{equation}
\color{black}
\begin{split}
\frac{1}{T}\sum_{t=1}^{T}\eta||\nabla F(\boldsymbol{w}^{(t)})||_{1} &\leq \frac{\mathbb{E}[F(\boldsymbol{w}^{(0)}) - F(\boldsymbol{w}^{(t+1)})]}{T} + \frac{Ld\eta^2}{2} + 4\eta\left[\frac{\sqrt{M}||\bar{\boldsymbol{\sigma}}||_{1}}{M+K} + \frac{K(Q+A)d}{M+K}\right] \\
&+ \frac{2\eta Bd}{\sqrt{M+K+1}}\bigg(1-\frac{1}{M+K+1}\bigg)^{\frac{M+K}{2}}.
\end{split}
\end{equation}
Letting $\eta=\frac{1}{\sqrt{TLd}}$ and dividing both sides by $\eta$ gives
\begin{equation}
\begin{split}
&\frac{1}{T}\sum_{t=1}^{T}||\nabla F(\boldsymbol{w}^{(t)})||_{1} \\
&\leq
\frac{\mathbb{E}[F(\boldsymbol{w}^{(0)}) - F(\boldsymbol{w}^{(t+1)})]\sqrt{Ld}}{\sqrt{T}} + \frac{\sqrt{Ld}}{2\sqrt{T}}+ 4\left[\frac{\sqrt{M}||\bar{\boldsymbol{\sigma}}||_{1}}{M+K} + \frac{K(Q+A)d}{M+K}\right] + \frac{2Bd}{\sqrt{M+K+1}}\bigg(1-\frac{1}{M+K+1}\bigg)^{\frac{M+K}{2}}\\
&\leq \frac{(F(\boldsymbol{w}^{(0)}) - F^{*})\sqrt{Ld}}{\sqrt{T}} + \frac{\sqrt{Ld}}{2\sqrt{T}}+ 4\left[\frac{\sqrt{M}||\bar{\boldsymbol{\sigma}}||_{1}}{M+K} + \frac{K(Q+A)d}{M+K}\right] + \frac{2Bd}{\sqrt{M+K+1}}\bigg(1-\frac{1}{M+K+1}\bigg)^{\frac{M+K}{2}}.
\end{split}
\end{equation}
which completes the proof.
\end{proof}

\subsection{Proof of Theorem \ref{SPByzantine2}}
\begin{Theorem}\label{SPByzantine2}
Suppose Assumptions \ref{A1}-\ref{A4} are satisfied, $|\nabla F(\boldsymbol{w}^{(t)})_{i}| < Q, \forall i, t$, $B \geq 2A = \mathcal{O}(\sqrt{T})$, and $\lim_{T\rightarrow \infty}\frac{M+K}{\sqrt{T}} = 0$. Then by running Algorithm \ref{DPSGDAlgorithm} with \textit{TernaryVote}, $\mathcal{N}_{t} = \mathcal{M}\cup\mathcal{K}$, and the learning rate $\eta=\frac{1}{\sqrt{TLd}}$ for $T$ iterations, we have
\begin{equation}
\begin{split}
&\frac{1}{T}\sum_{t=1}^{T}\sum_{i=1}^{d}\left(M\left|\nabla F(\boldsymbol{w}^{(t)})_{i}\right| - \left|\sum_{k\in\mathcal{K}}\boldsymbol{g}_{k,i}^{(t)}\right|\right)|\nabla F(\boldsymbol{w}^{(t)})_{i}| \leq \mathcal{O}\bigg(\frac{B}{\sqrt{T}}\bigg) + \mathcal{O}\bigg(\frac{1}{B}\bigg).
\end{split}
\end{equation}
\end{Theorem}

\begin{proof}
Let $\frac{1}{M+K}\left[\sum_{m\in \mathcal{M}}ternary(\boldsymbol{g}_{m}^{(t)},A,B)+\sum_{k\in \mathcal{K}}byzantine(\boldsymbol{g}_{k}^{(t)},A,B)\right] = \hat{\boldsymbol{g}}_{\mathcal{M}+\mathcal{K}}^{(t)}$ and $\frac{1}{M+K}\left[\sum_{m\in\mathcal{M}}\boldsymbol{g}_{m}^{(t)}+\sum_{k\in\mathcal{K}}\boldsymbol{g}_{k}^{(t)}\right] = \bar{\boldsymbol{g}}_{\mathcal{M}+\mathcal{K}}^{(t)}$. According to Assumption \ref{A2}, we have
\begin{equation}
\begin{split}
F(\boldsymbol{w}^{(t+1)}) - F(\boldsymbol{w}^{(t)}) &\leq \langle\nabla F(\boldsymbol{w}^{(t)}), \boldsymbol{w}^{(t+1)}-\boldsymbol{w}^{(t)}\rangle + \frac{L}{2}||\boldsymbol{w}^{(t+1)}-\boldsymbol{w}^{(t)}||_{2}^2 \\
& =-\eta \left\langle\nabla F(\boldsymbol{w}^{(t)}), sign\left(\hat{\boldsymbol{g}}_{\mathcal{M}+\mathcal{K}}^{(t)}\right)\right\rangle +\frac{L}{2}\bigg|\bigg|\eta sign\left(\hat{\boldsymbol{g}}_{\mathcal{M}+\mathcal{K}}^{(t)}\right)\bigg|\bigg|^2 \\
& \leq -\eta \left\langle\nabla F(\boldsymbol{w}^{(t)}), sign\left(\hat{\boldsymbol{g}}_{\mathcal{M}+\mathcal{K}}^{(t)}\right)\right\rangle + \frac{Ld\eta^2}{2}, \\
\end{split}
\end{equation}
where $\eta$ is the learning rate. Taking expectations on both sides yields
\begin{equation}
\begin{split}
&\mathbb{E}[F(\boldsymbol{w}^{(t+1)}) - F(\boldsymbol{w}^{(t)})] \\
&\leq \frac{Ld\eta^2}{2} -\eta \sum_{i=1}^{d}\nabla F(\boldsymbol{w}^{(t)})_{i}\mathbb{E}\bigg[P\bigg(sign\left(\hat{\boldsymbol{g}}_{\mathcal{M}+\mathcal{K},i}^{(t)}\right)=1\bigg) - P\bigg(sign\left(\hat{\boldsymbol{g}}_{\mathcal{M}+\mathcal{K},i}^{(t)}\right) = -1\bigg)\bigg] \\
&= \frac{Ld\eta^2}{2} -\eta \sum_{i=1}^{d}\nabla F(\boldsymbol{w}^{(t)})_{i}\sum_{n=1}^{M+K}\bigg(1-\frac{A}{B}\bigg)^{M+K-n}\mathbb{E}\bigg[\frac{{n-1 \choose \lfloor\frac{n-1}{2}\rfloor}A^{n-1}}{2^{n-1}B^{n}}{M+K-1 \choose n-1}\left[\sum_{m\in\mathcal{M}}\boldsymbol{g}_{m,i}^{(t)} + \sum_{k\in\mathcal{K}}\boldsymbol{g}_{k,i}^{(t)}\right] \\
&+ {M+K \choose n}\mathcal{O}\bigg(\frac{A^{n-2}}{B^{n}}\bigg)\bigg]\\
&= \frac{Ld\eta^2}{2} -\eta \sum_{i=1}^{d}\nabla F(\boldsymbol{w}^{(t)})_{i}\sum_{n=1}^{M+K}\bigg(1-\frac{A}{B}\bigg)^{M+K-n}\bigg[\frac{{n-1 \choose \lfloor\frac{n-1}{2}\rfloor}A^{n-1}}{2^{n-1}B^{n}}{M+K-1 \choose n-1}\left[M\nabla F(\boldsymbol{w}^{(t)})_{i} + \sum_{k\in\mathcal{K}}\boldsymbol{g}_{k,i}^{(t)}\right] \\
&+ {M+K \choose n}\mathcal{O}\bigg(\frac{A^{n-2}}{B^{n}}\bigg)\bigg]\\
&=\frac{Ld\eta^2}{2} - \eta\mathcal{I}(A,B,M,K)\left[\sum_{i=1}^{d}M|\nabla F(\boldsymbol{w}^{(t)})_{i}|^{2} + \nabla F(\boldsymbol{w}^{(t)})_{i}\sum_{k\in\mathcal{K}}\boldsymbol{g}_{k,i}^{(t)}\right]\\
&+\eta \sum_{i=1}^{d}\nabla F(\boldsymbol{w}^{(t)})_{i}\sum_{n=1}^{M+K}\bigg(1-\frac{A}{B}\bigg)^{M+K-n}\bigg[{M+K \choose n}\mathcal{O}\bigg(\frac{A^{n-2}}{B^{n}}\bigg)\bigg],
\end{split}
\end{equation}
in which $\mathcal{I}(A,B,M,K) = \sum_{n=1}^{M+K}\big(1-\frac{A}{B}\big)^{M+K-n}\big[\frac{{n-1 \choose \lfloor\frac{n-1}{2}\rfloor}A^{n-1}}{2^{n-1}B^{n}}{M+K-1 \choose n-1}\big]$. Adjusting the above inequality and averaging both sides over $t=1,2,\cdots, T$ yields
\begin{equation}
\color{black}
\begin{split}
&\frac{1}{T}\sum_{t=1}^{T}\eta\mathcal{I}(A,B,M,K)\left[\sum_{i=1}^{d}M|\nabla F(\boldsymbol{w}^{(t)})_{i}|^{2} + \nabla F(\boldsymbol{w}^{(t)})_{i}\sum_{k\in\mathcal{K}}\boldsymbol{g}_{k,i}^{(t)}\right] \\
&\leq \frac{\mathbb{E}[F(\boldsymbol{w}^{(0)}) - F(\boldsymbol{w}^{(t+1)})]}{T} + \frac{Ld\eta^2}{2}+ \frac{\eta}{T}\sum_{t=1}^{T}\sum_{i=1}^{d}\nabla F(\boldsymbol{w}^{(t)})_{i}\sum_{n=1}^{M+K}\bigg(1-\frac{A}{B}\bigg)^{M+K-n}\bigg[{M+K \choose n}\mathcal{O}\bigg(\frac{A^{n-2}}{B^{n}}\bigg)\bigg] \\
&\leq \frac{\mathbb{E}[F(\boldsymbol{w}^{(0)}) - F(\boldsymbol{w}^{(t+1)})]}{T} + \frac{Ld\eta^2}{2}+ \eta\sum_{n=1}^{M+K}\bigg(1-\frac{A}{B}\bigg)^{M+K-n}\bigg[{M+K \choose n}\mathcal{O}\bigg(\frac{A^{n-2}}{B^{n}}\bigg)\bigg]Qd. \\
\end{split}
\end{equation}

Let $\eta=\frac{1}{\sqrt{TLd}}$ and dividing both sides by $\eta\mathcal{I}(A,B,M,K)$ gives
\begin{equation}
\begin{split}
&\frac{1}{T}\sum_{t=1}^{T}\left[\sum_{i=1}^{d}M|\nabla F(\boldsymbol{w}^{(t)})_{i}|^{2} + \nabla F(\boldsymbol{w}^{(t)})_{i}\sum_{k\in\mathcal{K}}\boldsymbol{g}_{k,i}^{(t)}\right] \\
&\leq
\frac{\mathbb{E}[F(\boldsymbol{w}^{(0)}) - F(\boldsymbol{w}^{(t+1)})]\sqrt{Ld}}{\sqrt{T}\mathcal{I}(A,B,M,K)} + \frac{\sqrt{Ld}}{2\sqrt{T}\mathcal{I}(A,B,M,K)}+ \frac{\sum_{n=1}^{M+K}\big(1-\frac{A}{B}\big)^{M+K-n}\bigg[{M+K \choose n}\mathcal{O}\bigg(\frac{A^{n-2}}{B^{n}}\bigg)\bigg]Qd}{\mathcal{I}(A,B,M,K)}\\
&\leq \frac{(F(\boldsymbol{w}^{(0)}) - F^{*})\sqrt{Ld}}{\sqrt{T}\mathcal{I}(A,B,M,K)} + \frac{\sqrt{Ld}}{2\sqrt{T}\mathcal{I}(A,B,M,K)}+ \frac{\sum_{n=1}^{M+K}\big(1-\frac{A}{B}\big)^{M+K-n}\bigg[{M+K \choose n}\mathcal{O}\bigg(\frac{A^{n-2}}{B^{n}}\bigg)\bigg]Qd}{\mathcal{I}(A,B,M,K)}\\
&\leq \mathcal{O}\bigg(\frac{B}{\sqrt{T}}\bigg) + \mathcal{O}\bigg(\frac{1}{B}\bigg).
\end{split}
\end{equation}

Since
\begin{equation}
\begin{split}
M|\nabla F(\boldsymbol{w}^{(t)})_{i}|^{2} + \nabla F(\boldsymbol{w}^{(t)})_{i}\sum_{k\in\mathcal{K}}\boldsymbol{g}_{k,i}^{(t)} \geq |\nabla F(\boldsymbol{w}^{(t)})_{i}|\left(M\left|\nabla F(\boldsymbol{w}^{(t)})_{i}\right| - \left|\sum_{k\in\mathcal{K}}\boldsymbol{g}_{k,i}^{(t)}\right|\right),
\end{split}
\end{equation}

we have
\begin{equation}
\begin{split}
&\frac{1}{T}\sum_{t=1}^{T}\sum_{i=1}^{d}\left(M\left|\nabla F(\boldsymbol{w}^{(t)})_{i}\right| - \left|\sum_{k\in\mathcal{K}}\boldsymbol{g}_{k,i}^{(t)}\right|\right)|\nabla F(\boldsymbol{w}^{(t)})_{i}| \leq \mathcal{O}\bigg(\frac{B}{\sqrt{T}}\bigg) + \mathcal{O}\bigg(\frac{1}{B}\bigg),
\end{split}
\end{equation}
which completes the proof.
\end{proof}

\section{Convergence of Algorithm \ref{DPSGDAlgorithm} with Full-Batch Gradient}\label{convergencefullbatch}
\begin{Theorem}\label{SPconvergerateGDSchemeII}
Suppose Assumptions \ref{A1}-\ref{A4} are satisfied, and the learning rate is set as $\eta=\frac{1}{\sqrt{TLd}}$. Then by running Algorithm \ref{DPSGDAlgorithm} with \textit{TernaryVote} for $T$ iterations, we have
\begin{equation}
\begin{split}
\frac{1}{T}\sum_{t=1}^{T}c_{0}||\nabla F(\boldsymbol{w}^{(t)})||_{1} &\leq
\frac{\mathbb{E}[F(\boldsymbol{w}^{(0)}) - F(\boldsymbol{w}^{(t+1)})]\sqrt{Ld}}{\sqrt{T}} + \frac{\sqrt{Ld}}{2\sqrt{T}} + \frac{2Bd}{M}\\
&\leq \frac{(F(\boldsymbol{w}^{(0)}) - F^{*})\sqrt{Ld}}{\sqrt{T}} + \frac{\sqrt{Ld}}{2\sqrt{T}} + \frac{2Bd}{M},
\end{split}
\end{equation}
where $0 < c_{0} < 1$ is some positive constant.
\end{Theorem}

Before proving Theorem \ref{SPconvergerateGDSchemeII}, we first introduce the following lemmas.

\begin{Lemma}\label{poissonbinomiallemma1}
Let $X_{i}$ denote a Bernoulli random variable with a successful probability of $p_{i}$ and $S_{X_{1:M}} = \sum_{i=1}^{M}X_{i}$. Without loss of generality, suppose $0 < p_{1} \leq p_{2} \leq \cdots \leq p_{M} < 1$. Then $P(S_{X_{1:M}} \geq k) < \frac{1}{2}$ for $k \geq 1 + \sum_{i=1}^{M}p_{i}$.
\end{Lemma}

The following proof of Lemma \ref{poissonbinomiallemma1} is inspired by the method in \cite{jorgensen2018team}.

\begin{proof}
Decomposing $S_{X_{1:M}}$ as $S_{X_{2:M-1}}+X_{1}+X_{M}$, we have
\begin{equation}
\begin{split}
P(S_{X_{1:M}} \geq k) &= p_{1}p_{M}P(S_{X_{2:M-1}} \geq k-2) + [p_{1}(1-p_{M})+p_{M}(1-p_{1})]P(S_{X_{2:M-1}} \geq k-1) \\
&+ (1-p_{1})(1-p_{M})P(S_{X_{2:M-1}} \geq k) \\
&= p_{1}p_{M}[P(S_{X_{2:M-1}} \geq k-2) - 2P(S_{X_{2:M-1}} \geq k-1) + P(S_{X_{2:M-1}} \geq k)] \\
&+ (p_{1}+p_{M})[P(S_{X_{2:M-1}} \geq k-1) - P(S_{X_{2:M-1}} \geq k)] + P(S_{X_{2:M-1}} \geq k).
\end{split}
\end{equation}

Further, define another set of Bernoulli random variables $Y_{i}$ with a successful probability of $q_{i}$. If $q_{1} = q_{M} = \frac{p_{1}+p_{M}}{2}$ and $q_{i} = p_{i}$ $\forall i \in \{2,3,...,M-1\}$,
similarly, if we decompose $S_{Y_{1:M}}$ as $S_{Y_{2:M-1}}+Y_{1}+Y_{M}$, we have
\begin{equation}
\begin{split}
P(S_{Y_{1:M}} \geq k) &= q_{1}q_{M}P(S_{Y_{2:M-1}} \geq k-2) + [q_{1}(1-q_{M})+q_{M}(1-q_{1})]P(S_{Y_{2:M-1}} \geq k-1) \\
&+ (1-q_{1})(1-q_{M})P(S_{Y_{2:M-1}} \geq k) \\
&= q_{1}q_{M}[P(S_{Y_{2:M-1}} \geq k-2) - 2P(S_{Y_{2:M-1}} \geq k-1) + P(S_{Y_{2:M-1}} \geq k)] \\
&+ (q_{1}+q_{M})[P(S_{Y_{2:M-1}} \geq k-1) - P(S_{Y_{2:M-1}} \geq k)] + P(S_{Y_{2:M-1}} \geq k).
\end{split}
\end{equation}

Since $q_{1}+q_{M} = p_{1}+p_{M}$ and $P(S_{X_{2:M-1}} \geq k) = P(S_{Y_{2:M-1}} \geq k)$, we have
\begin{equation}
\begin{split}
P(S_{Y_{1:M}} \geq k) - P(S_{X_{1:M}} \geq k) &= (q_{1}q_{M}-p_{1}p_{M})[P(S_{X_{2:M-1}} \geq k-2) - 2P(S_{X_{2:M-1}} \geq k-1) + P(S_{X_{2:M-1}} \geq k)] \\
& = \left(\frac{p_{1}-p_{M}}{2}\right)^{2}[P(S_{X_{2:M-1}} \geq k-2) - 2P(S_{X_{2:M-1}} \geq k-1) + P(S_{X_{2:M-1}} \geq k)]\\
&=\left(\frac{p_{1}-p_{M}}{2}\right)^{2}[P(S_{X_{2:M-1}} = k-2) - P(S_{X_{2:M-1}} = k-1)].
\end{split}
\end{equation}

Therefore, $P(S_{X_{2:M-1}} = k-2) > P(S_{X_{2:M-1}} = k-1)$ is a sufficient condition for $P(S_{X_{1:M}} \geq k) < P(S_{Y_{1:M}} \geq k)$. According to Theorem 1 in \cite{samuels1965number}, if $p_{M} + \sum_{i=1}^{M}p_{i} < k+1$, then $P(S_{X_{1:M}} = k) > P(S_{X_{1:M}} = k+1)$. Similarly, if $p_{M-1} + \sum_{i=2}^{M-1}p_{i} < k-1$, then $P(S_{X_{2:M-1}} = k-2) > P(S_{X_{2:M-1}} = k-1)$. Therefore, given that $k \geq 1 + \sum_{i=1}^{M}p_{i}$, we have  $P(S_{X_{1:M}} \geq k) < P(S_{Y_{1:M}} \geq k)$.

Applying the techniques above repeatedly, we will finally obtain $P(S_{X_{1:M}} \geq k) < P(S_{\hat{X}_{1:M}} \geq k)$ for $k \geq 1 + \sum_{i=1}^{M}p_{i}$, where $\{\hat{X}_{i}\}_{i=1}^{M}$'s are Bernoulli random variables with a successful probability of $\Bar{p} = \frac{1}{M}\sum_{i=1}^{M}p_{i}$. Therefore, $P(S_{X_{1:M}} \geq k) < P(\hat{S}_{M} \geq k)$, where $\hat{S}_{M} \sim \text{BIN}(M,\Bar{p})$.

Now, define another Poisson Binomial random variable $S_{M+2} = 0 + \hat{S}_{M} + 1$. Note that the constants $0$ and $1$ correspond to Bernoulli trials with success probabilities of $0$ and $1$, respectively. Then, we have
\begin{equation}
\begin{split}
P(\hat{S}_{M} \geq k) = P(S_{M+2} \geq k+1) < P(\hat{S}_{M+2} \geq k+1),
\end{split}
\end{equation}
where $\hat{S}_{M+2} \sim \text{BIN}(M+2,p_{M+2})$ with $(M+2)p_{M+2} = 1 + \sum_{i=1}^{M}p_{i}$. By applying the same argument repeatedly, we can obtain a sequence of Binomial random variables $\hat{S}_{M+2j} \sim \text{BIN}(M+2j,p_{M+2j})$, where $(M+2j)p_{M+2j} = j + \sum_{i=1}^{M}p_{i}$. Particularly, $P(\hat{S}_{M+2j} \geq k+j)$ increases as $j$ increases.

Notice that the success probability $p_{M+2j} = \frac{j+\sum_{i=1}^{M}p_{i}}{M+2j}$ approaches $\frac{1}{2}$ as $j$ increases, while the variance grows and approaches $\infty$. Invoking the central limit theorem implies that the probability distribution of $\hat{S}_{M+2j}$ approaches normal distribution with mean $j + \sum_{i=1}^{M}p_{i}$ and variance $\sigma_{M+2j}^{2}=(j + \sum_{i=1}^{M}p_{i})(1-\frac{j+\sum_{i=1}^{M}p_{i}}{M+2j})$. Therefore, $P(\hat{S}_{M+2j} \geq k+j) = P(\frac{\hat{S}_{M+2j}-(j + \sum_{i=1}^{M}p_{i})}{\sigma_{M+2j}} \geq \frac{k- \sum_{i=1}^{M}p_{i}}{\sigma_{M+2j}})$ approaches $\frac{1}{2}$ as $j$ increases. As a result, we have $P(S_{X_{1:M}} \geq k) < \frac{1}{2}$ for $k \geq 1 + \sum_{i=1}^{M}p_{i}$, which completes the proof.


\end{proof}

\begin{Lemma}\label{LemmaProbofWrongGenericTernaryGD}
Let $u_{1},u_{2},\cdots,u_{M}$ be $M$ known and fixed real numbers and consider binary random variables $\hat{u}_{m}$, $1\leq m \leq M$, which is given by
\begin{equation}
\hat{u}_{m} = ternary(u_{m},A,B) =
\begin{cases}
\hfill 1, \hfill \text{with probability $\frac{A+u_{m}}{2B}$},\\
\hfill 0, \hfill \text{with probability $1-\frac{A}{B}$},\\
\hfill -1, \hfill \text{with probability $\frac{A-u_{m}}{2B}$},\\
\end{cases}
\end{equation}
Suppose $B \geq 2A$ and $|\bar{u}| = |\frac{1}{M}\sum_{i=1}^{M}u_{M}| \geq \frac{B}{M}$, then there exists some positive constant $c_{0}$ such that
\begin{equation}
\begin{split}
P\bigg(sign\bigg(\frac{1}{M}\sum_{m=1}^{M}\hat{u}_{m}\bigg)&\neq sign\bigg(\frac{1}{M}\sum_{m=1}^{M}u_{m}\bigg)\bigg) \leq \frac{1 - c_{0}}{2}.
\end{split}
\end{equation}
\end{Lemma}

\begin{proof}
For each $u_{m}$, we construct the following two random variables
\begin{equation}
\hat{u}_{m,1} =
\begin{cases}
\hfill 1, \hfill \text{with probability $\frac{1}{2}+\frac{u_{m}}{2B}+\sqrt{\frac{1}{4}+\frac{|u_{m}|^{2}}{4B^2}-\frac{A}{2B}}$},\\
\hfill -1, \hfill \text{with probability $\frac{1}{2}-\frac{u_{m}}{2B}-\sqrt{\frac{1}{4}+\frac{|u_{m}|^{2}}{4B^2}-\frac{A}{2B}}$},\\
\end{cases}
\end{equation}

\begin{equation}
\hat{u}_{m,2} =
\begin{cases}
\hfill 1, \hfill \text{with probability $\frac{1}{2}+\frac{u_{m}}{2B}-\sqrt{\frac{1}{4}+\frac{|u_{m}|^{2}}{4B^2}-\frac{A}{2B}}$},\\
\hfill -1, \hfill \text{with probability $\frac{1}{2}-\frac{u_{m}}{2B}+\sqrt{\frac{1}{4}+\frac{|u_{m}|^{2}}{4B^2}-\frac{A}{2B}}$},\\
\end{cases}
\end{equation}

It can be observed that $\frac{\hat{u}_{m,1} + \hat{u}_{m,2}}{2}$ follows the same distribution as $\hat{u}_{m}$, which means that
\begin{equation}
\begin{split}
P\bigg(sign\bigg(\frac{1}{M}\sum_{m=1}^{M}\hat{u}_{m}\bigg) &\neq sign\bigg(\frac{1}{M}\sum_{m=1}^{M}u_{m}\bigg)\bigg) = P\bigg(sign\bigg(\frac{1}{2M}\sum_{m=1}^{M}[\hat{u}_{m,1} + \hat{u}_{m,2}]\bigg) \neq sign\bigg(\frac{1}{M}\sum_{m=1}^{M}u_{m}\bigg)\bigg).
\end{split}
\end{equation}

Denote $p_{m,j} = P(\hat{u}_{m,j} \neq sign\big(\frac{1}{M}\sum_{m=1}^{M}u_{m}\big)), \forall i \in \{1,2\}$, it can be shown that $\bar{p} = \frac{1}{2M}\sum_{m=1}^{M}[p_{m,1}+p_{m,2}] = \frac{1}{2}-\frac{|\bar{u}|}{2B}$. Then, let $X_{m,j}$ denote a Bernoulli random variable with a success probability of $p_{m,j}$, we have
\begin{equation}
\begin{split}
&P\bigg(sign\bigg(\frac{1}{2M}\sum_{m=1}^{M}[\hat{u}_{m,1} + \hat{u}_{m,2}]\bigg) \neq sign\bigg(\frac{1}{M}\sum_{m=1}^{M}u_{m}\bigg)\bigg) \leq  P\bigg(\sum_{m=1}^{M}X_{m,1}+X_{m,2} \geq M\bigg).
\end{split}
\end{equation}

Lemma \ref{poissonbinomiallemma1} implies that $P\big(\sum_{m=1}^{M}X_{m,1}+X_{m,2} \geq M\big) < \frac{1}{2}$ as long as $M \geq 1 + 2M\bar{p} = 1 + M - \frac{M|\bar{u}|}{B}$, which is equivalent to $|\bar{u}| \geq \frac{B}{M}$. This essentially means that there exists some constant $c_{0}$ such that $P\big(\sum_{m=1}^{M}X_{m,1}+X_{m,2} \geq M\big) \leq \frac{1-c_{0}}{2}$, which completes the proof of Lemma \ref{LemmaProbofWrongGenericTernaryGD}.

\end{proof}

Given Lemma \ref{LemmaProbofWrongGenericTernaryGD} at hand, we are ready to prove Theorem \ref{SPconvergerateGDSchemeII}.

\begin{proof}
According to Assumption \ref{A2}, we have
\begin{equation}
\begin{split}
&F(\boldsymbol{w}^{(t+1)}) - F(\boldsymbol{w}^{(t)}) \\
&\leq \langle\nabla F(\boldsymbol{w}^{(t)}), \boldsymbol{w}^{(t+1)}-\boldsymbol{w}^{(t)}\rangle + \frac{L}{2}||\boldsymbol{w}^{(t+1)}_{i}-\boldsymbol{w}^{(t)}_{i}||_{2}^2 \\
& =-\eta \left\langle\nabla F(\boldsymbol{w}^{(t)}), sign\bigg(\frac{1}{M}\sum_{m=1}^{M}ternary(\boldsymbol{g}_{m}^{(t)},A,B)\bigg)\right\rangle +\frac{L}{2}\bigg|\bigg|\eta sign\bigg(\frac{1}{M}\sum_{m=1}^{M}ternary(\boldsymbol{g}_{m}^{(t)},A,B)\bigg)\bigg|\bigg|^2 \\
& \leq -\eta \left\langle\nabla F(\boldsymbol{w}^{(t)}), sign\bigg(\frac{1}{M}\sum_{m=1}^{M}ternary(\boldsymbol{g}_{m}^{(t)},A,B)\bigg)\right\rangle + \frac{Ld\eta^2}{2} \\
& = -\eta ||\nabla F(\boldsymbol{w}^{(t)})||_{1} + \frac{Ld\eta^2}{2} + 2\eta\sum_{i=1}^{d}|\nabla F(\boldsymbol{w}^{(t)})_{i}|\times\mathds{1}_{sign(\frac{1}{M}\sum_{m=1}^{M}ternary(\boldsymbol{g}_{m,i}^{(t)},A,B))\neq sign(\nabla F(\boldsymbol{w}^{(t)})_{i})},
\end{split}
\end{equation}
where $\nabla F(\boldsymbol{w}^{(t)})_{i}$ is the $i$-th entry of the vector $\nabla F(\boldsymbol{w}^{(t)})$ and $\eta$ is the learning rate. Taking expectations on both sides yields

\begin{equation}
\begin{split}
&\mathbb{E}[F(\boldsymbol{w}^{(t+1)}) - F(\boldsymbol{w}^{(t)})] \\
&\leq -\eta ||\nabla F(\boldsymbol{w}^{(t)})||_{1} + \frac{Ld\eta^2}{2} +2\eta\sum_{i=1}^{d}\mathbb{E}\bigg[|\nabla F(\boldsymbol{w}^{(t)})_{i}|P\bigg(sign\bigg(\frac{1}{M}\sum_{m=1}^{M}ternary(\boldsymbol{g}_{m,i}^{(t)},A,B)\bigg)\neq sign(\nabla F(\boldsymbol{w}^{(t)})_{i})\bigg)\bigg]\\
&\leq -\eta ||\nabla F(\boldsymbol{w}^{(t)})||_{1} + \frac{Ld\eta^2}{2} \\
&+2\eta\sum_{i=1}^{d}\mathbb{E}\bigg[|\nabla F(\boldsymbol{w}^{(t)})_{i}|P\bigg(sign\bigg(\frac{1}{M}\sum_{m=1}^{M}ternary(\boldsymbol{g}_{m,i}^{(t)},A,B)\bigg)\neq sign\bigg(\frac{1}{M}\sum_{m=1}^{M}\boldsymbol{g}^{(t)}_{m,i}\bigg)\bigg)\mathds{1}_{|\frac{1}{M}\sum_{m=1}^{M}\boldsymbol{g}^{(t)}_{m,i}| \geq \frac{B}{M}}\bigg]\\
&+2\eta\sum_{i=1}^{d}\mathbb{E}\bigg[|\nabla F(\boldsymbol{w}^{(t)})_{i}|P\bigg(sign\bigg(\frac{1}{M}\sum_{m=1}^{M}ternary(\boldsymbol{g}_{m,i}^{(t)},A,B)\bigg)\neq sign\bigg(\frac{1}{M}\sum_{m=1}^{M}\boldsymbol{g}^{(t)}_{m,i}\bigg)\bigg)\mathds{1}_{|\frac{1}{M}\sum_{m=1}^{M}\boldsymbol{g}^{(t)}_{m,i}| < \frac{B}{M}}\bigg]\\
&\leq -\eta ||\nabla F(\boldsymbol{w}^{(t)})||_{1} + \frac{Ld\eta^2}{2} +2\eta\sum_{i=1}^{d}\mathbb{E}\bigg[\frac{1-c_{0}}{2}|\nabla F(\boldsymbol{w}^{(t)})_{i}|\bigg]+2\eta\sum_{i=1}^{d}\frac{B}{M}\\
&\leq -c_{0}\eta||\nabla F(\boldsymbol{w}^{(t)})||_{1} + \frac{Ld\eta^2}{2} + \frac{2\eta Bd}{M}.
\end{split}
\end{equation}

Adjusting the above inequality and averaging both sides over $t=1,2,\cdots,T$, we can obtain
\begin{equation}
\color{black}
\begin{split}
&\frac{1}{T}\sum_{t=1}^{T}c_{0}\eta||\nabla F(\boldsymbol{w}^{(t)})||_{1} \leq \frac{\mathbb{E}[F(\boldsymbol{w}^{(0)}) - F(\boldsymbol{w}^{(t+1)})]}{T} + \frac{Ld\eta^2}{2} + \frac{2\eta Bd}{M}.
\end{split}
\end{equation}

Letting $\eta=\frac{1}{\sqrt{LTd}}$ and dividing both sides by $\eta$ gives
\begin{equation}
\begin{split}
\frac{1}{T}\sum_{t=1}^{T}c_{0}||\nabla F(\boldsymbol{w}^{(t)})||_{1} &\leq
\frac{\mathbb{E}[F(\boldsymbol{w}^{(0)}) - F(\boldsymbol{w}^{(t+1)})]\sqrt{Ld}}{\sqrt{T}} + \frac{\sqrt{Ld}}{2\sqrt{T}} + \frac{2Bd}{M}\\
&\leq \frac{(F(\boldsymbol{w}^{(0)}) - F^{*})\sqrt{Ld}}{\sqrt{T}} + \frac{\sqrt{Ld}}{2\sqrt{T}} + \frac{2Bd}{M},
\end{split}
\end{equation}
which completes the proof.
\end{proof}

\section{Details of the Implementation}\label{DetailsImplementation}
Our experiments are mainly implemented using Python 3.8 with packages Numpy 1.19.2 and Pytorch 1.10.1.
\subsection{Dataset and Pre-processing}
We perform experiments on the standard MNIST, Fashion-MNIST, and CIFAR-10 datasets. The MNIST dataset is for handwritten digit recognition consisting of 60,000 training samples and 10,000 testing samples. Each sample is a 28$\times$28 size gray-level image. The Fashion-MNIST dataset shares the same image size, data format and the structure of training and testing splits as the MNIST dataset. We normalize the data by dividing it by the max RGB value (i.e., 255.0). The CIFAR-10 dataset contains 50,000 training samples and 10,000 testing samples. Each sample is a 32$\times$32 color image. The data are normalized with a zero-centered mean.
\subsection{Neural Network Setting}
For MNIST and Fashion-MNIST, we implement a three-layer fully connected neural network with softmax of classes with cross-entropy loss. The two hidden layers have 512 and 256 hidden ReLU units, respectively. For CIFAR-10, we implement a simple convolutional neural network with 4 convolution layers. It has two contiguous blocks of two convolution layers with 64 and 128 channels, respectively, followed by a max-pooling, and then it has one dense layer with 256 hidden units.
\subsection{Learning Rate Tuning}
For all the algorithms, we tune the initial learning rates from the set $\{0.00001,0.00002,0.00005,0.0001,0.0002,0.0005,\\0.001,0.002,0.005,0.01,0.02,0.05,0.1,0.2,0.5,1.0,2.0,5.0,10.0,20.0,50.0,100.0\}$. For MNIST and Fashion-MNIST, we use a fixed learning rate, while for CIFAR-10, we decrease the learning rate by a factor of 10 after 250 communication rounds.

\end{document}